\documentclass[10pt,twocolumn,letterpaper]{article}

\usepackage[pagenumbers]{preprint}

\usepackage{amsfonts}
\usepackage{amsthm}
\usepackage{bm}
\usepackage{cancel}
\usepackage{overpic}
\usepackage{calc}
\usepackage{algorithm}
\usepackage{algpseudocode}
\usepackage{tikz}
\usepackage[export]{adjustbox}
\usepackage{dblfloatfix}
\usepackage{mathtools}
\usepackage{cuted}
\usepackage{booktabs}
\usepackage{xcolor}
\usepackage{colortbl}
\usepackage{multirow}
\usepackage{adjustbox}
\usepackage{float}
\usepackage{thmtools}
\usepackage{thm-restate}
\definecolor{vfirst}{RGB}{162, 221, 196}   %
\definecolor{vsecond}{RGB}{222, 244, 212}  %
\providecommand{\first}[1]{\cellcolor{vfirst}\textbf{#1}}
\providecommand{\second}[1]{\cellcolor{vsecond}#1}

\newlength{\figheight}

\newcommand{\SSplats}{\emph{StochasticSplats}\xspace}
\newcommand{\GPS}{\emph{Gaussian Point Splatting}\xspace}
 
\newcommand{\Lvar}{\mathcal{L}_{\text{var}}}     %
\newcommand{\lvar}{\lambda_{\text{var}}}         %
 
\NewDocumentCommand{\Cspat}{o m}{\hat{C}^{\IfValueT{#1}{#1,}\text{spatial}}_{#2}}
 
\newcommand{\wgt}[2]{w_{#1\rightarrow #2}}       %
\newcommand{\occ}[3]{o^{#1}_{#2\rightarrow #3}}  %
 
\DeclareMathOperator{\lerp}{lerp}                %

\definecolor{linkblue}{rgb}{0.21,0.49,0.74}
\usepackage[pagebackref,breaklinks,colorlinks,allcolors=linkblue]{hyperref}

\title{SteadySplats: Resampling of Low-Variance Gaussians for High-Fidelity Stochastic Rendering}
\author{Felix Windisch*\\
Graz University of Technology\\
{\tt\small felix.windisch@tugraz.at}
\and
Thomas K\"ohler*\\
Graz University of Technology\\
{\tt\small t.koehler@tugraz.at}
\and
Lukas Radl\\
Graz University of Technology\\
{\tt\small lukas.radl@tugraz.at}
\and
Chris Wyman\\
NVIDIA\\
{\tt\small chris.wyman@acm.org}
\and
Georgios Kopanas\\
Google DeepMind\\
{\tt\small gkopanas@google.com}
\and
Bernhard Kerbl\\
University of Copenhagen\\
{\tt\small beke@di.ku.dk}
\and
Markus Steinberger\\
Graz University of Technology\\
{\tt\small steinberger@tugraz.at}
}

\begin{document}

\twocolumn[{%
\renewcommand\twocolumn[1][]{#1}%
\maketitle

}]

{\renewcommand{\thefootnote}{\fnsymbol{footnote}}\footnotetext[1]{Both authors contributed equally to this work.}}

\noindent
\begin{figure}
    \centering
    \includegraphics[width=\linewidth]{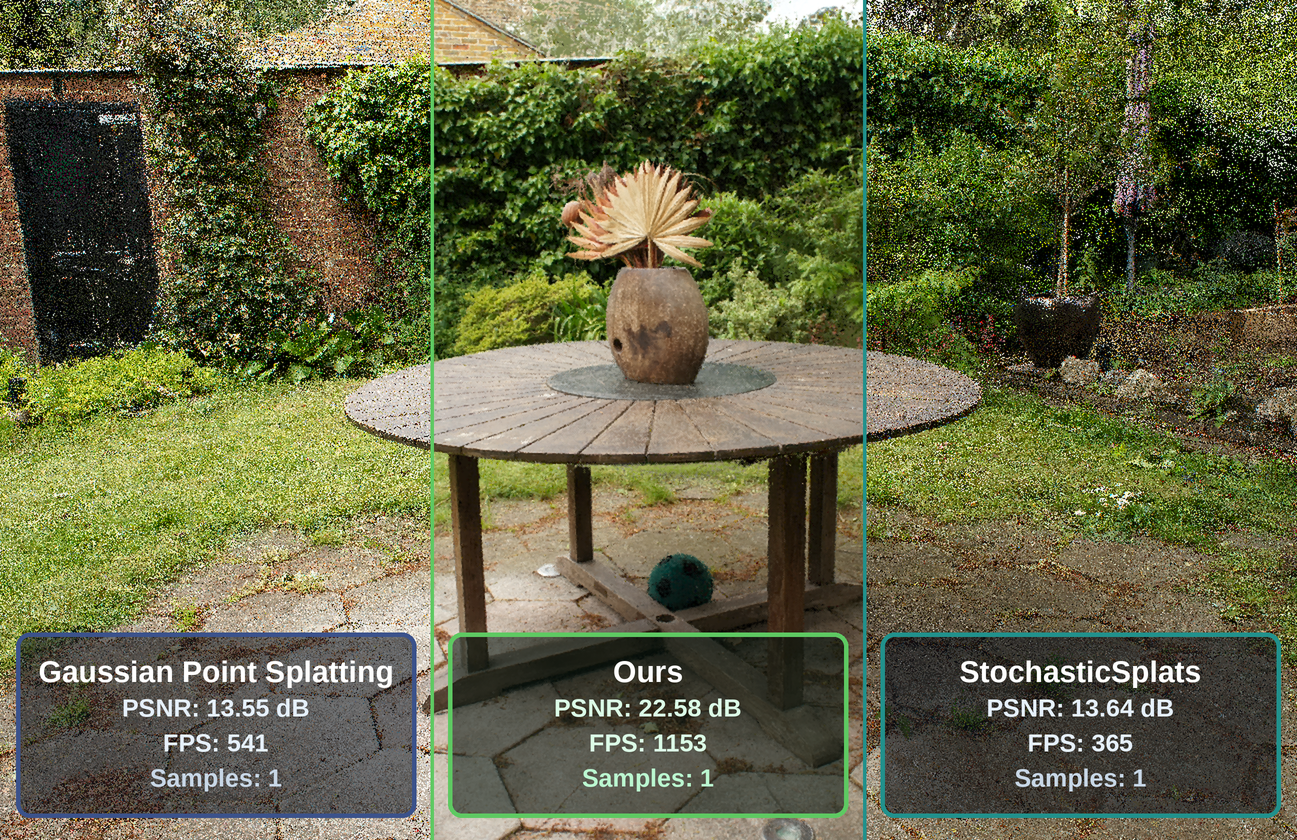}
    \caption{Through optimized training and inference, our stochastic OIT renderer achieves significantly better image quality than prior art with the same number of random samples. Our approach makes single-sample rendering a viable candidate for applications that demand high quality, as well as extremely high frame rates. Measurements were performed using each method's original implementation.}
    \label{fig:teaser}
\end{figure}%

\begin{abstract}
Stochastic order-independent transparency enables efficient and elegant rendering of primitive-based radiance fields like 3D Gaussian Splatting models, but remains impractical due to the inherent visible noise in the output. We propose a principled approach to minimize high-frequency noise, addressing its sources at the representation and image synthesis level. During stochastic rendering, our history-based spatial resampling scheme drastically accelerates image convergence, while temporal importance resampling ensures coherence under camera movement. During training, a color regularizer implicitly reduces the variance along view rays in the 3DGS models. With these properties, our optimized, Vulkan-based renderer effectively mitigates output noise at low and high sample counts, achieving a substantial 13~dB PSNR increase in quality over previous stochastic methods at 1 sample per pixel and quickly converging to sorted 3DGS with an average L1 error of less than $10^{-4}$.
\end{abstract}

\section{Introduction}
\label{sec:intro}

\definecolor{ViridisDark}{HTML}{218f8d}   %
\definecolor{ViridisGreen}{HTML}{5EC962}  %

The rapid evolution of novel view synthesis has been largely driven by the advent of radiance fields and, more recently, primitive-based representations in particular. 
Methods such as 3D Gaussian Splatting (3DGS)~\cite{kerbl20233dgs} have demonstrated exceptional visual fidelity and real-time rendering capabilities by modeling complex scenes as unstructured collections of volumetric primitives. 
Traditionally, fast rendering of these representations relies on alpha-blending sorted primitives using a tile-based rasterizer.

Stochastic order-independent transparency (OIT) presents an elegant, efficient alternative for rendering primitive-based radiance fields. 
By probabilistically sampling visible primitives rather than sorting them, stochastic OIT dramatically simplifies and accelerates the rendering pipeline. 
However, despite its structural advantages, this approach has remained largely impractical for novel view synthesis. 
The stochastic nature of the visibility evaluation inherently introduces high-frequency noise in the resulting images (\cf~\cref{fig:teaser}). 
Across novel view trajectories, this uncorrelated noise manifests itself as severe temporal flickering and ``boiling'' artifacts~\cite{schied2017spatiotemporal}, to which human visual perception is sensitive. 
Mitigating this issue and suppressing these artifacts below visible thresholds requires accumulating hundreds of samples per pixel, which directly negates the performance advantages of sorting-free rendering~\cite{kheradmand2025stochasticsplats}.

In this work, we propose solutions that make stochastic order-independent transparency practical for high-fidelity rendering by aggressively tackling noise at both ends of the pipeline: training and inference. All of our contributions are implemented in a high-performance Vulkan-based renderer:
\begin{itemize}
    \item \textbf{Color Variance Regularization:} A training objective that penalizes high-variance primitive configurations, implicitly conditioning the 3D model to produce less noise during stochastic sampling.

    \item \textbf{History-Based Spatial Resampling:} An efficient, inference-time resampling technique to massively accelerate convergence of image synthesis with minimal bias.

    \item \textbf{Temporal Consistency:} We demonstrate that under mild assumptions, spatio-temporal resampling can be implemented as a recursive exponential moving average filter to efficiently mitigate \emph{temporal} artifacts. Stratifying samples using per-primitive Spatio-Temporal Blue Noise further improves convergence.
\end{itemize}
We demonstrate that our combined approach successfully renders clean, high-fidelity images at low sample counts, achieving higher frame rates than vanilla 3DGS with near-indistinguishable quality.

\section{Related Work}
\label{sec:related_work}

\paragraph{3D Gaussian Splatting and the Sorting Bottleneck.}
3D Gaussian Splatting (3DGS) has emerged as a highly efficient representation for real-time novel view synthesis~\cite{kerbl20233dgs}. Traditionally, it relies on a tile-based rasterizer that strictly sorts semi-transparent primitives based on a per-primitive depth to perform deterministic alpha-compositing. However, this global sorting operation is computationally expensive, scales poorly with increasing scene complexity, especially on low-end hardware, and introduces temporal discontinuities known as ``popping'' artifacts when slight camera movements alter the primitive order. Several methods have attempted to address these limitations. \emph{StopThePop}~\cite{radl2024stopthepop} eliminates popping through computationally intensive per-pixel sorting based on the maximal Gaussian response. Alternatively, exact volumetric intermixing can be achieved by ray tracing of constant density ellipsoids~\cite{mai2024ever} or by reconstructing the transmittance function from analytic moments~\cite{muller2025moment}. Other approaches approximate alpha-blending via proxy-guided weighted blending~\cite{liu2025duplex} or order-independent accumulation with a K-buffer~\cite{hahlbohm2025hybrid} to bypass full radix sorting.
Weighted sum rendering~\cite{hou2025sort, du2026mobile} completely eliminates sorting during training and inference at the cost of reduced image quality.

\paragraph{Order-Independent and Stochastic Transparency.}
To fundamentally circumvent the sorting bottleneck, recent literature has adapted stochastic order-independent transparency (OIT). Originally introduced to 3D graphics by Enderton et al.~\cite{enderton2010stochastic}, stochastic transparency replaces sorted alpha-blending with a randomized Monte Carlo estimator. \emph{StochasticSplats}~\cite{kheradmand2025stochasticsplats} successfully integrates this estimator into 3DGS by keeping each fragment as an opaque sample with probability equal to its alpha value, completely bypassing depth-sorting while accurately blending overlapping Gaussians. Similarly, \emph{Gaussian Point Splatting}~\cite{rijsdijk2026gaussian} treats Gaussians as discrete point clouds, sampling pixel-sized opaque points independently and splatting them via 64-bit atomics to achieve massive parallelism. While these approaches eliminate the need for sorting, they inherently introduce severe high-frequency noise into the rendered images and must rely on expensive sample accumulation or temporal anti-aliasing (TAA)~\cite{karis2014high,yang2020survey}. 
Stochastic approaches can also be applied to ray tracing 3D Gaussians for relightable reconstruction~\cite{xu2026stochastic} and semi-transparency~\cite{sun2025stochastic}. \emph{Stochastic Gaussian Splatting}~\cite{aira2025modeling} moves the stochasticity into the Gaussian parameters themselves in order to estimate reconstruction uncertainty.

\paragraph{Denoising and Resampling.}
Early real-time denoising relied heavily on analytical spatial and temporal filters, most notably Spatiotemporal Variance-Guided Filtering (SVGF)~\cite{schied2017spatiotemporal} to accumulate temporal history while preserving edges. As deep learning matured, denoising networks such as kernel prediction networks~\cite{mildenhall2018burst, vogels2018denoising, fan2021real}, neural bilateral grids~\cite{meng2020real}, and Deep Learning Super Sampling~\cite{nvidia_dlss_developer, nvidia_dlss_rayreconstruction} gained popularity. 
Spatiotemporal Reservoir Resampling (ReSTIR)~\cite{wyman2023gentle} and descendant methods~\cite{zhang2024area, kettunen2023conditional, ouyang2021restir, hedstrom2025restir, zeng2025restir} directly improve the quality of the Monte Carlo integration in a principled manner before the denoising step. By leveraging temporal history and spatial neighbors, the impact of each individual sample is amplified.

\section{Preliminaries}
\label{sec:preliminaries}

In this section, we detail how the stochastic paradigm diverges from traditional 3DGS rendering.

\subsection{3D Gaussian Splatting (3DGS)}
3D Gaussian Splatting represents complex scenes as an unstructured collection of anisotropic 3D Gaussians~\cite{kerbl20233dgs}. Each Gaussian primitive $i$ is defined by a set of optimizable parameters $\{\mu_i, \Sigma_i, \sigma_i, c_i\}$, which encode its mean position ($\mu_i$), 3D covariance matrix ($\Sigma_i$), opacity ($\sigma_i$), and view-dependent color ($c_i$), typically modeled using spherical harmonics. 

To synthesize a novel view, these 3D Gaussians are ``splatted'' to a screen-space 2D Gaussian distribution $G^{2D}$ and depth $z_i$ using an affine approximation for projection~\cite{zwicker2001EWA}. We denote by $\mathcal{G}(x,y)$ the set of Gaussians intersecting the pixel $(x,y)$, by $\alpha_i(x,y) = \sigma_i \cdot G_i^{2D}(x,y)$ the alpha-blending factor of Gaussian $i$, and by $C_{x,y}$ the final pixel radiance, which is classically computed by evaluating the volume rendering equation through deterministic alpha-compositing~\cite{kerbl20233dgs}:
\begin{equation}
    C_{x,y} = \sum_{i \in \mathcal{G}(x,y)} c_i \alpha_i(x,y) \underbrace{\prod_{\mathclap{k \in \mathcal{G}(x,y) \mid z_k < z_i}} (1-\alpha_k(x,y))}_{T_i(x,y)} \, ,
    \label{eq:blending}
\end{equation}
where $z_k < z_i$ denotes all intersecting primitives located in front of primitive $i$ along the ray. In standard 3DGS, primitives are sorted by $z_i$ prior to blending, which simplifies the transmittance $T_i(x,y)$ to $\prod_{k<i} (1-\alpha_k(x,y))$.

While this sorted alpha-blending ensures a stable and noise-free image, the global sorting operation imposes significant computational overhead and memory bandwidth bottlenecks, limiting scalability, especially on low-power devices, and producing temporal popping artifacts~\cite{radl2024stopthepop}.

\subsection{Stochastic and Order-Independent 3DGS}
\label{sec:stochastic}

\begin{figure}[t]
    \centering
    \includegraphics[width=\linewidth]{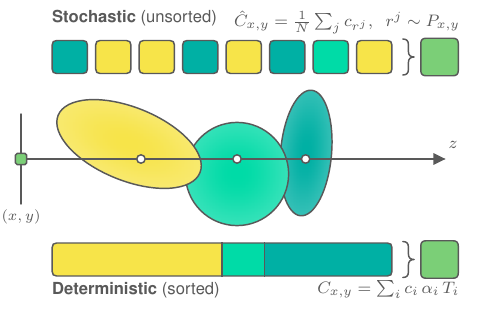}
    \caption{Illustration of deterministic and stochastic rendering of 3D Gaussians.}
    \label{fig:stochastic}
\end{figure}

\emph{StochasticSplats}~\cite{kheradmand2025stochasticsplats} fundamentally alters the rendering pipeline by adapting stochastic order-independent transparency (OIT)~\cite{enderton2010stochastic}. Rather than sorting primitives globally, stochastic transparency replaces deterministic alpha-blending with an unbiased Monte Carlo estimator. The two approaches are contrasted in \cref{fig:stochastic}.

The core idea is to probabilistically assign fully opaque samples to a pixel according to the fragments' alpha-blending factors, evaluating the color without any sorting. The radiance $C_{x,y}$ is estimated by randomly and independently sampling $N$ Gaussian indices $r^1, \dots, r^N$ from a probability mass function $P_{x,y}$:
\begin{equation}
    \hat{C}_{x,y} = \frac{1}{N} \sum_{j=1}^N c_{r^j} \, , \quad \text{with } r^j \sim P_{x,y} \, ,
\end{equation}
\begin{equation*}
    P_{x,y}(i)= \begin{cases}
        \alpha_i(x,y) \, T_i(x,y) & \text{if } i\in \mathcal{G}(x,y) \, , \\
        0 & \text{otherwise,}
    \end{cases}
\end{equation*}
\ie, each Gaussian is sampled with exactly its blending weight from \cref{eq:blending}.
A single draw $r \sim P_{x,y}$ thus yields an unbiased one-sample estimate $c_r$ of the pixel color, $\mathbb{E}_{r \sim P_{x,y}}[c_r] = \sum_{i} P_{x,y}(i) \, c_i = C_{x,y}$.

In practical hardware implementations, $r$ is sampled by leveraging standard, unsorted rasterization with a Z-buffer. For each fragment, a uniform random number $u_i \sim U(0, 1)$ is sampled; the fragment is kept as fully opaque if $u_i < \alpha_i(x,y)$, and discarded otherwise. The Z-buffer then retains the nearest kept fragment, whose index $r$ therefore follows $P_{x,y}$.

\emph{Gaussian Point Splatting} produces samples from the same underlying distribution $P_{x,y}$ using a massively parallelized point-sampling technique. As such, its output is near-identical to \emph{StochasticSplats} except for small biases introduced by approximations. Further details can be found in \cref{GPS}.

\paragraph{Temporal Reuse.}
Because 1-sample stochastic splats produce noisy raw frames, temporal accumulation is essential for stability. Existing methods address this in two distinct ways. 
\emph{StochasticSplats} maintains an unweighted Cumulative Moving Average (CMA),
\begin{equation*}
    \hat{C}^t_{x,y} = \Big(1 - \frac{1}{N^t_{x,y}}\Big) \, \hat{C}^{t-1}_{x',y'} + \frac{1}{N^t_{x,y}} \, c_{r^t} \, ,
\end{equation*}
where $r^t \sim P^t_{x,y}$ is the sample drawn in frame $t$, $(x',y')$ is the reprojection of $(x,y)$ into the previous frame, and $N^t_{x,y} = N^{t-1}_{x',y'} + 1$ counts the accumulated samples, resetting the sample counter $N^t_{x,y} \gets 1$ whenever relative depth error exceeds a threshold $\tau$.

\emph{Gaussian Point Splatting} instead uses a fixed-weight exponential moving average (EMA) with a bounded temporal window, borrowing AABB color clamping from standard temporal anti-aliasing (TAA)~\cite{karis2014high,salvi2016temporal} to suppress ghosting by restricting reprojected history to within $\gamma$ standard deviations of the local color mean.

\section{Method}
\label{sec:method}

The noise produced by stochastic 3DGS methods \cite{kheradmand2025stochasticsplats, rijsdijk2026gaussian} is a direct result of the variance in the Monte Carlo estimator $\hat{C}_{x,y}$, which determines its root-mean-square error:
\begin{equation}
\sqrt{\mathbb{E}\big[\lVert \hat{C}_{x,y} - C_{x,y} \rVert^2\big]} = \sqrt{\operatorname{Var}[\hat{C}_{x,y}]} = \frac{\sqrt{\operatorname{Var}[c_r]}}{\sqrt{N}} \, .
\end{equation}
To reduce this error, either the color variance $\operatorname{Var}[c_r]$ at each pixel must be decreased or the sample count $N$ increased.

\begin{figure}[t]
    \centering
    \includegraphics[width=0.85\linewidth]{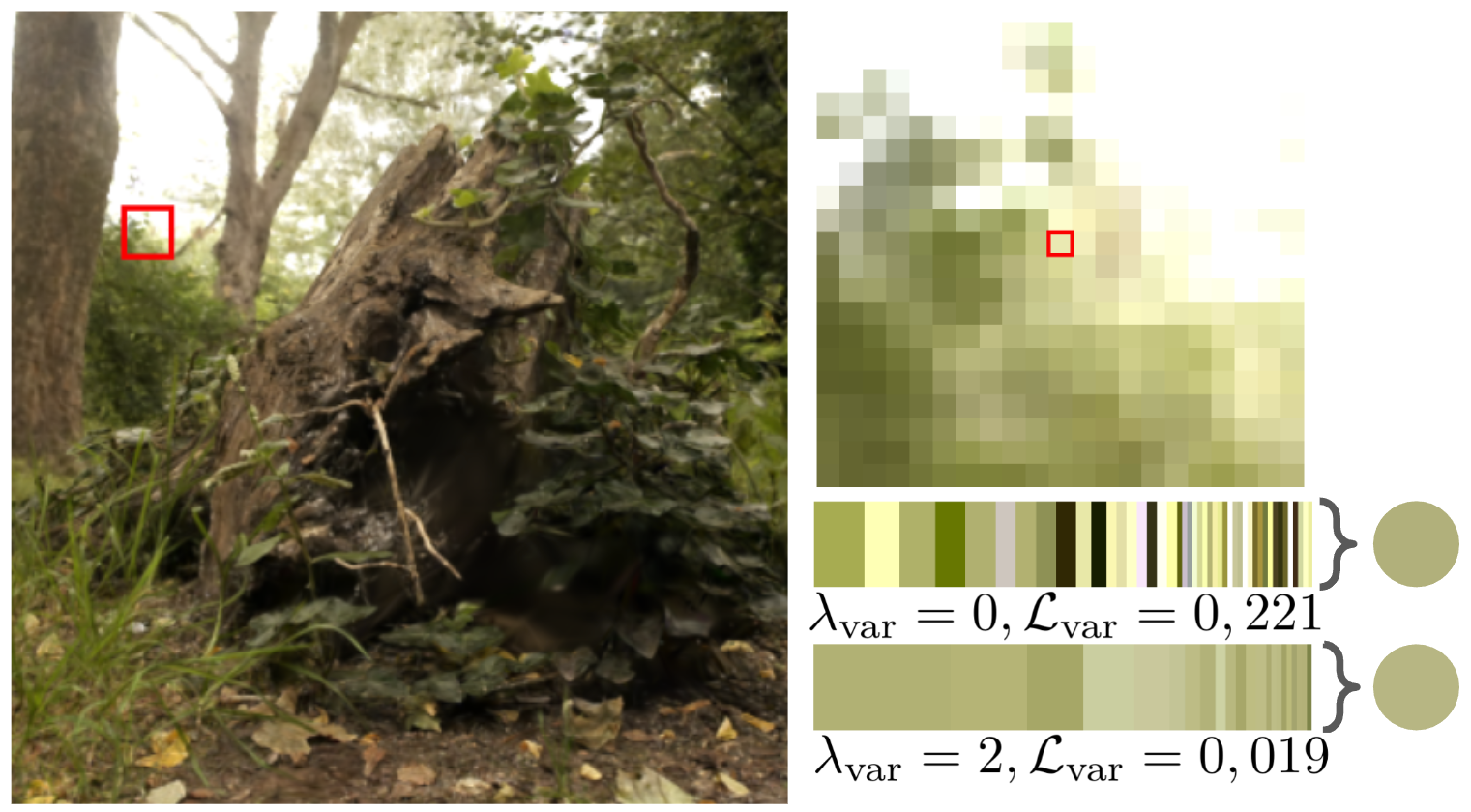}
    \caption{By default, 3DGS optimization blends disparate colors to match the ground truth. With $\lvar > 0$, the blended colors align to reduce color variance.}
    \label{fig:variance_loss}
\end{figure}

\subsection{Color Variance Regularization}
\label{sec:color_var}

Previous stochastic methods for 3DGS \cite{rijsdijk2026gaussian, kheradmand2025stochasticsplats} reduce the error by increasing $N$ through expensive supersampling or artifact-prone temporal accumulation, while ignoring the variance of the sampled colors $\operatorname{Var}[c_r]$, which directly represents the power of the noise produced by the renderer:
\begin{align}
\operatorname{Var}_{r \sim P_{x,y}}[c_r] &= \mathbb{E}_{r \sim P_{x,y}}\!\left[\lVert c_r - \mathbb{E}_{r \sim P_{x,y}}[c_r] \rVert^2\right] \label{eq:Variance}\\ 
&= \mathbb{E}_{r \sim P_{x,y}}\!\left[\lVert c_r - C_{x,y} \rVert^2\right] \\ 
&= \sum_{i \in \mathcal{G}(x,y)} P_{x,y}(i) \cdot \lVert c_i - C_{x,y} \rVert^2 \, .
\end{align}

Our key insight is that standard, deterministic 3DGS training produces significantly higher color variance $\operatorname{Var}[c_r]$ than necessary to represent the scene at high fidelity. To minimize $\operatorname{Var}[c_r]$ during optimization, we add a color variance term $\Lvar$ normalized by the resolution $H \times W$ to the loss function:
\begin{equation}
    \Lvar = \frac{1}{H\cdot W}\sum_{x=1}^{W} \sum_{y=1}^{H} \sum_{i \in \mathcal{G}(x,y)} P_{x,y}(i) \cdot \lVert c_i - C_{x,y} \rVert^2 \, .
\end{equation}

Intuitively, this encourages all Gaussians that splat onto the same pixel to share the same color, instead of matching the ground truth by blending disparate colors. \Cref{fig:variance_loss} illustrates how $\Lvar$ aligns colors splatted onto a pixel. Note that $P_{x,y}(i)$ is already computed for each Gaussian during deterministic 3DGS rasterization and is commonly referred to as the Gaussian weight.

Our complete loss function is:
\begin{equation}
    \mathcal{L} = \lvar \cdot \Lvar + (1-\lambda) \cdot \mathcal{L}_{\ell_1} + \lambda \cdot \mathcal{L}_{\text{D-SSIM}} \,,
\end{equation}
with hyperparameters $\lambda, \lvar$ and photometric losses $\mathcal{L}_{\ell_1}, \mathcal{L}_{\text{D-SSIM}}$. The resulting models significantly reduce variance in the estimator and image noise with no modifications to the rendering procedure. For efficiency, the expectation $C_{x,y}$ can be replaced with the ground truth color $C^{\mathrm{gt}}_{x,y}$ during training, assuming $C_{x,y} \approx C^{\mathrm{gt}}_{x,y}$ for a converged model.

While this training objective has already found application in the context of sparse-view reconstruction~\cite{jena2025sparfels} and mesh extraction~\cite{radl2026confidence, zhang2025radiance}, its significant applications to stochastic rendering have thus far been unexplored.

\begin{figure*}[t]
    \centering
    \settoheight{\figheight}{\includegraphics[width=0.26\textwidth]{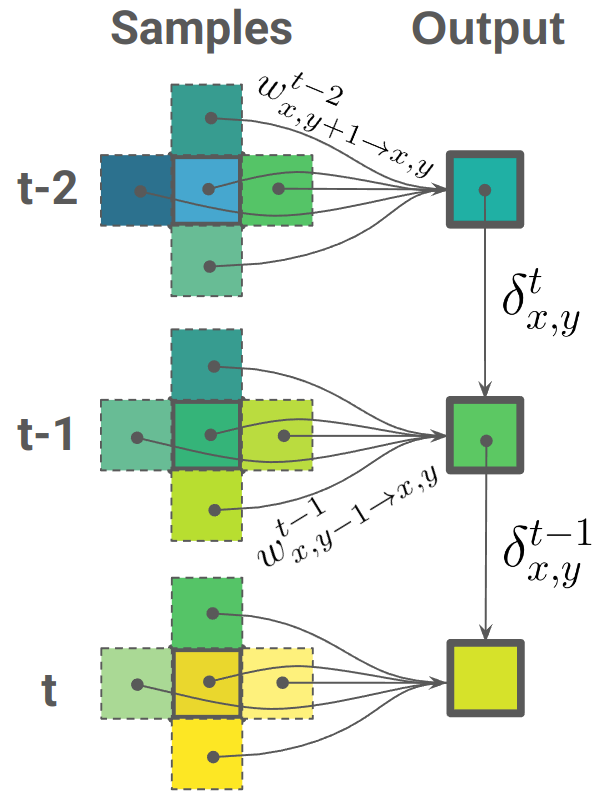}}
    \begin{subfigure}[t]{0.72\textwidth}
        \centering
        \parbox[c][\figheight][c]{\linewidth}{%
            \centering
            \includegraphics[width=0.99\linewidth]{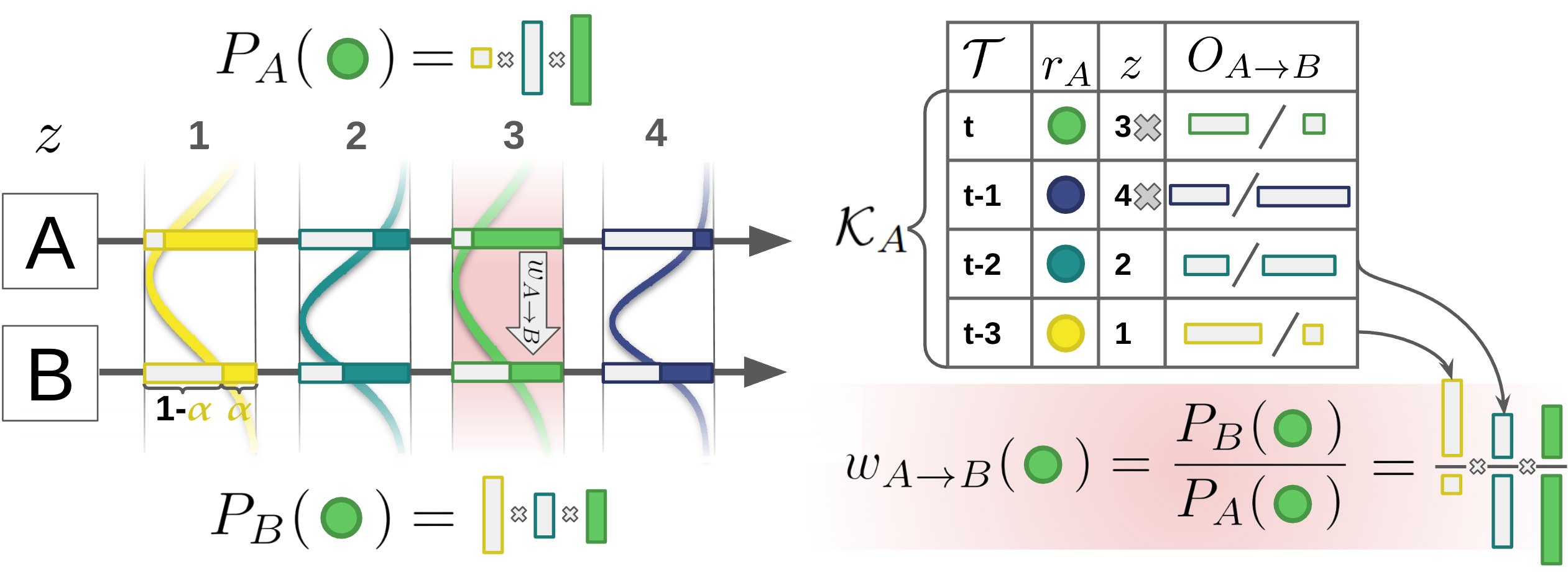}%
        }
        \caption{The resample weight from neighbor $n$ (A) to pixel $(x,y)$ (B) for the current frame is computed using the buffer $\mathcal{K}_n$ ($\mathcal{K}_A$) of previous samples.}
        \label{fig:spatio}
    \end{subfigure}%
    \hfill
    \rule[-0.5\figheight]{0.5pt}{\figheight}%
    \hfill
    \begin{subfigure}[t]{0.26\textwidth}
        \centering
        \parbox[c][\figheight][c]{\linewidth}{%
            \centering
            \includegraphics[width=\linewidth]{Figures/Method/TemporalReuse.png}%
        }
        \caption{Temporal resampling is computed recursively.}
        \label{fig:temp}
    \end{subfigure}
    \caption{Overview of spatial and temporal sample reuse.}
    \label{fig:overview}
\end{figure*}

\subsection{Spatial Resampling}
\label{sec:spatial}

By exploiting the fact that color distributions are similar across neighboring pixels (e.g., $P_{x,y}(i) \approx P_{x+1,y}(i)$), the information gain from each individual sample can be amplified by mixing the colors of neighboring samples. Naive reuse causes blurring artifacts, which can be mitigated by computing resampling weights that minimize the introduced bias.

\paragraph{Neighbor Reuse.}
Let $\mathcal{N}(x,y)$ define the set of neighboring pixels around $(x,y)$ from which samples will be reused. The final rendered color for a single sample is then given by:
\begin{equation}
    \Cspat{x,y} = \frac{\sum_{n \in \mathcal{N}(x,y)} \wgt{n}{x,y} (r_n) \cdot c_{{r}_{n}}}{\sum_{n \in \mathcal{N}(x,y)} \wgt{n}{x,y} (r_n)} \, ,
    \label{eq:SpatialReuse}
\end{equation}
where $r_n \sim P_n$ is the sample chosen at pixel $n$ and $\wgt{n}{x,y}$ is the \emph{spatial resampling weight}. Here and in all following sums over $\mathcal{N}(x,y)$, the pixel's own sample $r \sim P_{x,y}$ is included with weight $\wgt{x,y}{x,y}(r) = 1$; $|\mathcal{N}(x,y)|$ counts only the neighbors. The normalization factor is imposed by \emph{self-normalized importance sampling} (SNIS). \Cref{eq:SpatialReuse} can also be interpreted as a convolution of the 1-sample image with a dynamic kernel defined by $\mathcal{N}$ and $\wgt{n}{x,y}$. If $\wgt{n}{x,y}$ is chosen simply as a Gaussian falloff according to pixel distance, this operation is equivalent to a Gaussian blur, which successfully reduces noise but converges to a biased, blurry result.

To achieve noise reduction while minimizing bias, the expected value of the weighted samples from pixel $n$ must match the expected value of the target distribution at pixel $(x,y)$:
\begin{equation}
\mathbb{E}_{r_n \sim P_n}[\wgt{n}{x,y}(r_n) \cdot c_{{r}_n}] = \mathbb{E}_{r \sim P_{x,y}}[c_{r}].
\end{equation}
By expanding both expectations into their respective sums over the domain $\mathcal{G}$, we obtain:
\begin{equation}
\sum_{i \in \mathcal{G}(n)} P_n(i) \cdot \wgt{n}{x,y}(i) \cdot c_i = \sum_{i \in \mathcal{G}(x,y)} P_{x,y}(i) \cdot c_i \, .
\end{equation}
For this equality to hold for any arbitrary signal $c_i$, the corresponding terms inside the summations must be equal. Equating the coefficients yields:
\begin{equation}
    P_n(i) \cdot \wgt{n}{x,y}(i) = P_{x,y}(i) \, .
    \label{eq:coefficient_equality}
\end{equation}
Solving this directly provides the unbiased resample weights: $\wgt{n}{x,y}(r_n) = \frac{P_{x,y}(r_n)}{P_n(r_n)}$.
If $\wgt{n}{x,y}(r_n) > 1$, then $r_n$ is more likely to be sampled at pixel $(x,y)$ than at pixel $n$, and vice versa. The assumption $\mathcal{G}(x,y) = \mathcal{G}(n)$ used in \cref{eq:coefficient_equality} is justified in \cref{par:tile-based}.

\paragraph{Approximating $\mathbf{\wgt{n}{x,y}}$.}
While $\frac{P_{x,y}(r_n)}{P_{n}(r_n)}$ represents the ideal, unbiased weights, they are difficult to compute in practice due to their dependence on occlusion from closer primitives:
\begin{equation}
    \frac{P_{x,y}(r_n)}{P_{n}(r_n)} = \frac{\alpha_{r_n}(x,y)}{\alpha_{r_n}(n)} \prod_{\substack{k \in \mathcal{G}(x,y) \\ \land \ z_k < z_{r_n}}} \overbrace{\frac{1 - \alpha_k(x,y)}{1 - \alpha_k(n)}}^{\occ{k}{n}{x,y}} \, .
    \label{eq:weight}
\end{equation}
The first term accounts for the change in 2D Gaussian opacity between the two pixels and is evaluated directly by sampling the 2D Gaussian $G^{2D}_{r_n}$ on neighboring pixels. However, the second term $\occ{k}{n}{x,y}$, representing the change in occlusion of $r_n$ between pixels, presumes knowledge of the set of occluding Gaussians $\{k \in \mathcal{G}(x,y) \mid z_k < z_{r_n}\}$, which would require an expensive global sorting step or second sampling pass.

Our key insight is that the expected number of occluding Gaussians on each pixel is very small (cf.~\cref{fig:blending}) if the scene is trained with $\Lvar$. As such, this product is dominated by a small subset of Gaussians that have a high probability of being sampled because they represent the foreground. In turn, it is highly likely that these important occluding Gaussians were sampled recently on the same pixel. In essence, the samples we draw from each pixel act as importance samples for Gaussians that are likely to be dominant occluders. We can therefore approximate the change in occlusion by taking the product not over all Gaussians splatted onto the pixel, but only over those that have been sampled on the same pixel in the last $K$ iterations. 

To this end, for each pixel $(x,y)$, we store the unique values of $\occ{k}{x,y}{n}$ for all neighbors $n$ (computed in an extra rendering pass), the index $k$, and the corresponding depth $z_k$ across the $K$ previous samples. Using the set $\mathcal{K}_n$ stored at the source pixel $n$, we estimate $\wgt{n}{x,y}$ as:
\begin{equation}
\wgt{n}{x,y}(i) = \frac{P_{x,y}(i)}{P_n(i)} \approx \frac{\alpha_{i}(x,y)}{\alpha_{i}(n)} \prod_{\substack{k \in \mathcal{K}_n \\ \land \ z_k < z_i}} \occ{k}{n}{x,y} \, .
\label{eq:K-Reuse}
\end{equation}
Due to the unbounded support of the Gaussian kernel, as $K\rightarrow \infty$, $\{k \in \mathcal{G}(x,y) \mid z_k < z_i\} \subseteq \mathcal{K}_n$ holds with probability approaching one and the equality becomes exact. \Cref{fig:spatio} sketches the computation of this reuse weight.

Note that the stored buffer of past samples is deduplicated: we do not weight the importance of a Gaussian by its sample frequency (which would introduce substantial noise), but rather take the condition that it was sampled at least once as a signal to include its occlusion factor.

\begin{figure*}[!t]
    \centering
    \begin{subfigure}[b]{0.22\textwidth}
        \centering
        \includegraphics[width=0.95\linewidth]{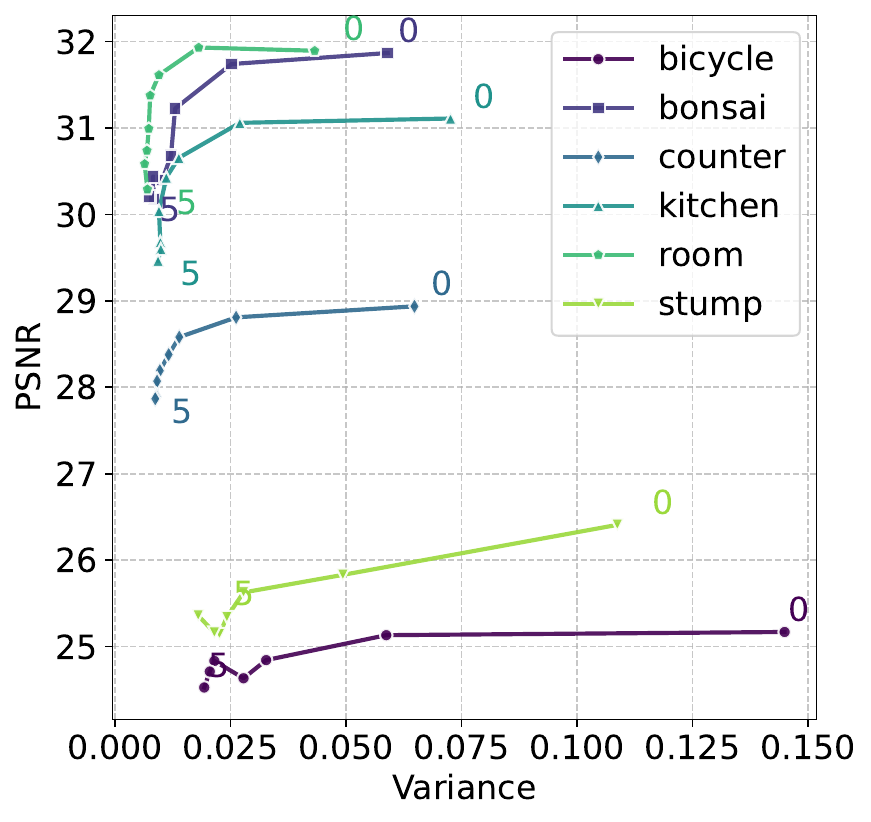}
        \caption{PSNR and variance $\Lvar$ \\ for varying $\lvar \in [0,5]$.}
        \label{fig:VarVSPSNR}
    \end{subfigure}
    \hfill
    \begin{subfigure}[b]{0.75\textwidth}
        \includegraphics[valign=t, width=0.93\linewidth]{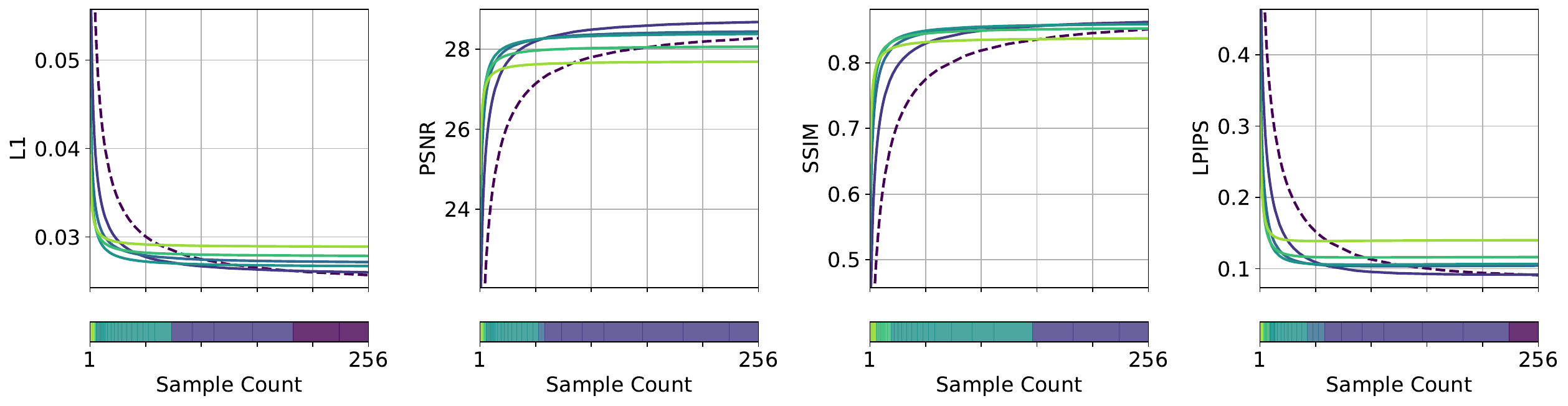}
        \includegraphics[valign=t, width=0.055\linewidth]{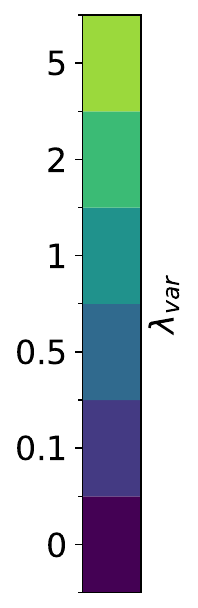}
        \caption{Comparison of average quality metrics across varying stochastic sample counts for models trained with different $\lvar$. The bottom bar denotes the optimal $\lvar$ for each sample count. The baseline ($\lvar=0$) is shown dashed.}
        \label{fig:curves}
    \end{subfigure}
    \caption{Average quality metrics on the Mip-NeRF 360 dataset for varying $\lvar$ with deterministic (\cref{fig:VarVSPSNR}) and stochastic (\cref{fig:curves}) blending.}
    \label{fig:main_variance_eval}
\end{figure*}

\subsection{Spatio-Temporal Resampling}
\label{sec:spatiotemporal_resampling}

Color distributions are coherent not only across spatial neighborhoods, but also temporally along motion trajectories. While evaluating full spatio-temporal resampling weights across multiple frames is intractable for real-time rendering, we show that under mild assumptions, full spatio-temporal resampling reduces to an efficient, closed-form recursive filter.

\paragraph{Temporal Importance Resampling.}
The spatial reuse weights $\wgt{n}{x,y}$ extend to the temporal domain by evaluating the target-to-proposal PDF ratio:
\begin{equation}
\wgt{n}{x,y}^{\tilde{t}\rightarrow t}(r_n^{\tilde{t}}) = \frac{P^t_{x,y}(r_n^{\tilde{t}})}{P^{\tilde{t}}_{n}(r_n^{\tilde{t}})} \, ,
\end{equation}
where $r_n^{\tilde{t}}$ is the Gaussian sample drawn at pixel $n$ on prior frame $\tilde{t} \in \mathcal{T}$. The full spatio-temporal SNIS estimator is then:
\begin{equation}
\hat{C}_{x,y}^{t} = \frac{\sum_{n \in \mathcal{N}(x,y)} \sum_{\tilde{t} \in \mathcal{T}} \wgt{n}{x,y}^{\tilde{t}\rightarrow t}(r_n^{\tilde{t}}) \cdot c_{r_n^{\tilde{t}}}}{\sum_{n \in \mathcal{N}(x,y)} \sum_{\tilde{t} \in \mathcal{T}} \wgt{n}{x,y}^{\tilde{t} \rightarrow t}(r_n^{\tilde{t}})} \, .
\label{eq:full_st_estimator}
\end{equation}
Assuming sample relevance decays by a per-frame factor $\delta_{x,y}^t \in [0, 1]$ shared by all samples in the history, the spatio-temporal weights factorize as:
\begin{equation}
\wgt{n}{x,y}^{\tilde{t} \rightarrow t}(r_n^{\tilde{t}}) = \wgt{n}{x,y}^{\tilde{t}}(r_n^{\tilde{t}}) \prod_{u=\tilde{t}+1}^t \delta_{x,y}^u \, ,
\label{eq:discount}
\end{equation}
This formulation is illustrated in \cref{fig:temp}.

\begin{restatable}{theorem}{emaReduction}
\label{thm:ema_reduction}
Under the assumption of a per-frame temporal discount shared by all history samples (cf.~\cref{eq:discount}), the spatio-temporal reuse estimator
\begin{equation}
\hat{C}_{x,y}^{t} = \frac{\sum_{n \in \mathcal{N}(x,y)} \sum_{\tilde{t} \in \mathcal{T}} \wgt{n}{x,y}^{\tilde{t}\rightarrow t}(r_n^{\tilde{t}}) \cdot c_{r_n^{\tilde{t}}}}{\underbrace{\sum_{n \in \mathcal{N}(x,y)} \sum_{\tilde{t} \in \mathcal{T}} \wgt{n}{x,y}^{\tilde{t} \rightarrow t}(r_n^{\tilde{t}})}_{D^t_{x,y}}}
\end{equation}
can be expressed as an exponential moving average filter:
\begin{equation}
\begin{split}
\hat{C}_{x,y}^{t} = {} & \frac{\delta^t_{x,y} \cdot D_{x,y}^{t-1}}{D_{x,y}^t} \cdot \hat{C}_{x,y}^{t-1} \\
& + \frac{\sum_{n \in \mathcal{N}(x,y)} \wgt{n}{x,y}^t(r_n^t)}{D_{x,y}^t} \cdot \Cspat[t]{x,y} \, .
\end{split}
\end{equation}
\end{restatable}

A detailed proof is provided in \cref{sec:proofs}. \Cref{thm:ema_reduction} reduces spatio-temporal SNIS to a two-state recursion per pixel: the running estimate $\hat{C}^t_{x,y}$ and the accumulated weight mass $D^t_{x,y} = \delta^t_{x,y} D^{t-1}_{x,y} + \sum_n \wgt{n}{x,y}^t(r_n^t)$.

At steady state ($D_{x,y}^t \approx D_{x,y}^{t-1}$), the recursion in \cref{thm:ema_reduction} is equivalent to an EMA filter with history weight $\delta_{x,y}^t \in (0, 1)$, commonly used in TAA~\cite{karis2014high,yang2020survey}. Both prior temporal schemes are special cases: without spatial reuse ($\sum_n \wgt{n}{x,y}^t = 1$), $\delta^t_{x,y} = 1$ yields the CMA of \SSplats (\cref{sec:stochastic}), whose depth reset corresponds to $\delta^t_{x,y} = 0$, and a constant $\delta^t_{x,y} < 1$ yields the fixed-weight EMA of \GPS. We use the steady-state form, so only the running estimate needs to be stored:
\begin{equation}
\hat{C}^t_{x,y} = \delta^t_{x,y} \, \hat{C}^{t-1}_{x,y} + (1 - \delta^t_{x,y}) \, \Cspat[t]{x,y} \, ,
\label{eq:ema}
\end{equation}
with a fixed nominal history weight $\delta^t_{x,y} = \bar{\delta}$ that is reduced only by soft history rejection (\cref{eq:soft_alpha}).

\paragraph{Moving Camera.}
\label{par:camera_motion}
\Cref{thm:ema_reduction} is stated for a fixed pixel. With a moving camera, we apply \cref{eq:ema} along the motion trajectory and read the history at the reprojection $(x',y')$ of $(x,y)$ into the previous frame, \ie, $\hat{C}^{t-1}_{x,y}$ becomes $\hat{C}^{t-1}_{x',y'}$ (\cref{sec:appendix_temporal}). This is exact only if the reprojection is exact and the surface remains visible; otherwise, the fetched history is an approximation, and the history weight $\delta^t_{x,y}$ controls how much of it is kept. The occlusion buffer of \cref{sec:spatial} is kept in screen space and not reprojected, so under motion it may contain occluders of nearby surface points.

\paragraph{Soft History Rejection.}
Prior temporal approaches for stochastic splatting present a clear trade-off. Pure depth-tested CMA maintains an unbounded sample history that causes persistent ghosting whenever depth differences are small, while its hard binary resets trigger visible popping on foliage and fine geometry. In contrast, standard TAA color clamping directly clamps $\hat{C}^{t-1}_{x',y'}$ to a local neighborhood bounding box, which clips valid Monte Carlo variance tails, introduces bias, and interferes with spatial reuse distributions~\cite{stachowiak2015stochastic,schied2018asvgf,yang2020survey}.

To avoid both issues, we replace color clipping and binary depth resets with history weight modulation~\cite{stachowiak2015stochastic,yang2020survey}. We leave the historical color $\hat{C}^{t-1}_{x',y'}$ unclipped, and instead change the history weight $\delta_{x,y}^t$ based on the $L_\infty$ distance $m$ of the reprojected history $\hat{C}^{t-1}_{x',y'}$ from the mean of the current samples $c_{r_q}$ in a $3\times 3$ neighborhood (in YCoCg space), normalized by velocity-adaptive confidence half-extents (\cref{sec:appendix_temporal}):
\begin{equation}
\delta_{x,y}^t = \bar{\delta} \cdot \min\left(1,\, \frac{1}{m}\right) \, ,
\label{eq:soft_alpha}
\end{equation}
where $\bar{\delta}$ is the fixed nominal history weight. For outlier samples ($m > 1$), this smoothly shifts history retention toward the new spatial estimate ($\delta_{x,y}^t \to 0$) without distorting valid color distributions. Implementation details and our velocity-adaptive confidence bounds are described in \cref{sec:appendix_temporal}.

\paragraph{Per-Primitive Spatio-Temporal Blue Noise.}
Instead of independent white noise, we sample the fragment visibility thresholds $u_i$ (\cref{sec:stochastic}) from a precomputed Spatio-Temporal Blue Noise (STBN) volume~\cite{wolfe2022stbn}. To decorrelate overlapping fragments at every pixel, each primitive applies an $R_2$ low-discrepancy spatial offset~\cite{roberts2018quasirandom}, combined with a golden-ratio Cranley-Patterson rotation~\cite{cranley1976randomization} to prevent temporal repetition. This per-primitive stratification preserves blue-noise spectral properties and estimator unbiasedness, resulting in improved quantitative metrics compared to white noise and visibly smoother images (\cf~\cref{fig:uniform_vs_stbn}).

\section{Evaluation}
\label{sec:experiments}

For the evaluation, we use the \SSplats method of drawing samples and render each sample individually (i.e., no batching) for minimal latency.

\begin{table*}[t]
\centering
\small
\setlength{\tabcolsep}{7pt}
\caption{\textbf{Quantitative comparison on the camera path benchmark} (21 trajectories, 3,152 frames across 7 \textit{Mip-NeRF 360} scenes). We ablate progressive pipeline components against the 1-sample \textit{StochasticSplats} baseline ($\lambda_\text{var}=0$), temporal reuse (CMA, $\tau=0.20$), our variance loss ($\lambda_\text{var}=1$) with spatial reuse, and our full spatio-temporal model (\textit{Ours + Temporal}, $\bar{\delta}=0.7$). All variants are evaluated against sorted $\alpha$-blending ground truth.}
\label{tab:main_pipeline_comparison}
\begin{tabular}{lcccccc}
\toprule
 & & & & & \multicolumn{2}{c}{\textbf{PSNR (dB)} $\uparrow$} \\
\cmidrule(lr){6-7}
\textbf{Method} & \textbf{Avg. PSNR (dB)} $\uparrow$ & \textbf{SSIM} $\uparrow$ & \textbf{LPIPS} $\downarrow$ & \textbf{tPSNR (dB)} $\uparrow$ & \textbf{Indoor} & \textbf{Outdoor} \\
\midrule
StochasticSplats \cite{kheradmand2025stochasticsplats}          & 18.336 & 0.293 & 0.667 & 15.235 & 19.562 & 16.701 \\
StochasticSplats + Temporal \cite{kheradmand2025stochasticsplats} & 21.954 & 0.552 & 0.524 & 19.450 & 23.313 & 20.141 \\
Ours                                                             & \second{31.437} & \second{0.845} & \second{0.345} & \second{28.410} & \second{32.978} & \second{29.381} \\
\textbf{Ours + Temporal}                                         & \first{35.085} & \first{0.937} & \first{0.209} & \first{32.449} & \first{37.161} & \first{32.318} \\
\bottomrule
\end{tabular}
\end{table*}

\paragraph{Variance-Quality Trade-off.}
While $\mathcal{L}_{\text{var}}$ effectively mitigates noise, it also dilutes the training signal of the $\ell_1$ and SSIM photometric losses, resulting in a compromise between noise and visual fidelity. To examine this trade-off, we train the scenes from the Mip-NeRF 360 dataset \cite{barron2022mip} for 15k iterations using original 3DGS training and then another 15k iterations from the same checkpoint with varying values of $\lambda_{\text{var}}$. The results in \cref{fig:VarVSPSNR} show that color variance as defined in \cref{eq:Variance} can be decreased significantly with minimal impact on converged image quality. On the other hand, there appears to be a hard lower limit around $\mathrm{Var} \approx 0.01$ on color variance for each scene, beyond which severe quality loss occurs. 

The quality of stochastic rendering for these models with varying sample count is reported in \cref{fig:curves}. The results show that with increasing sample count, the optimal value of $\lambda_{\text{var}}$ decreases. The reason for this is that error is dominated by noise at low sample counts and by bias at high sample counts. A small amount of variance regularization leads to almost universal gains, significantly outperforming baseline trained models across all image metrics for at least the first 200 samples. For the following evaluations, we use models trained with $\lambda_{\text{var}}=1$ as a compromise between fast and correct convergence.

\paragraph{Spatial Reuse.}
To validate the effectiveness of the spatial reuse component, we evaluate quality metrics over all test poses of the Mip-NeRF 360 scenes on models trained with $\lambda_{\text{var}} = 1$ and rendered stochastically with up to 100 samples. As baseline methods, we use unmodified \SSplats outputs and various fixed $3\times 3$ normalized blur kernels and compare against our kernel with dynamic reuse weights for the four direct pixel neighbors with varying buffer sizes $K$. \Cref{fig:ChoiceOfK} shows the resulting difference in quality metrics compared to deterministic $\alpha$-blending. For low sample counts, spatial reuse and blurring achieve an immediate and significant quality improvement over \SSplats. However, the convergence of the naive blur kernels quickly plateaus at a blurry and biased result, while our method continues to increase in quality, outperforming all baseline methods for 100 samples even with a buffer size of only $K=4$.

\begin{figure}[t]
    \centering
    \includegraphics[width=0.99\linewidth]{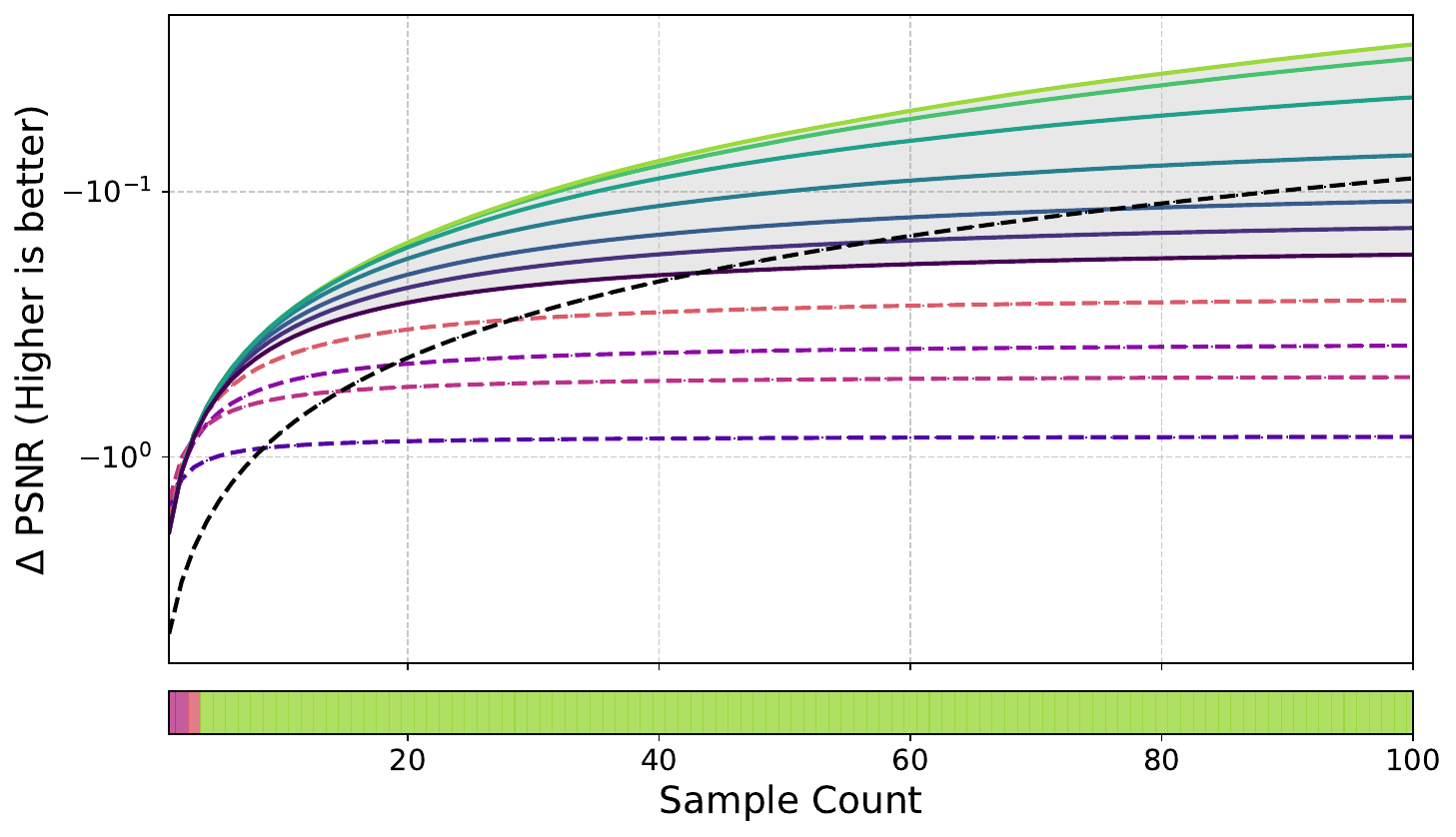}
    \includegraphics[width=0.99\linewidth]{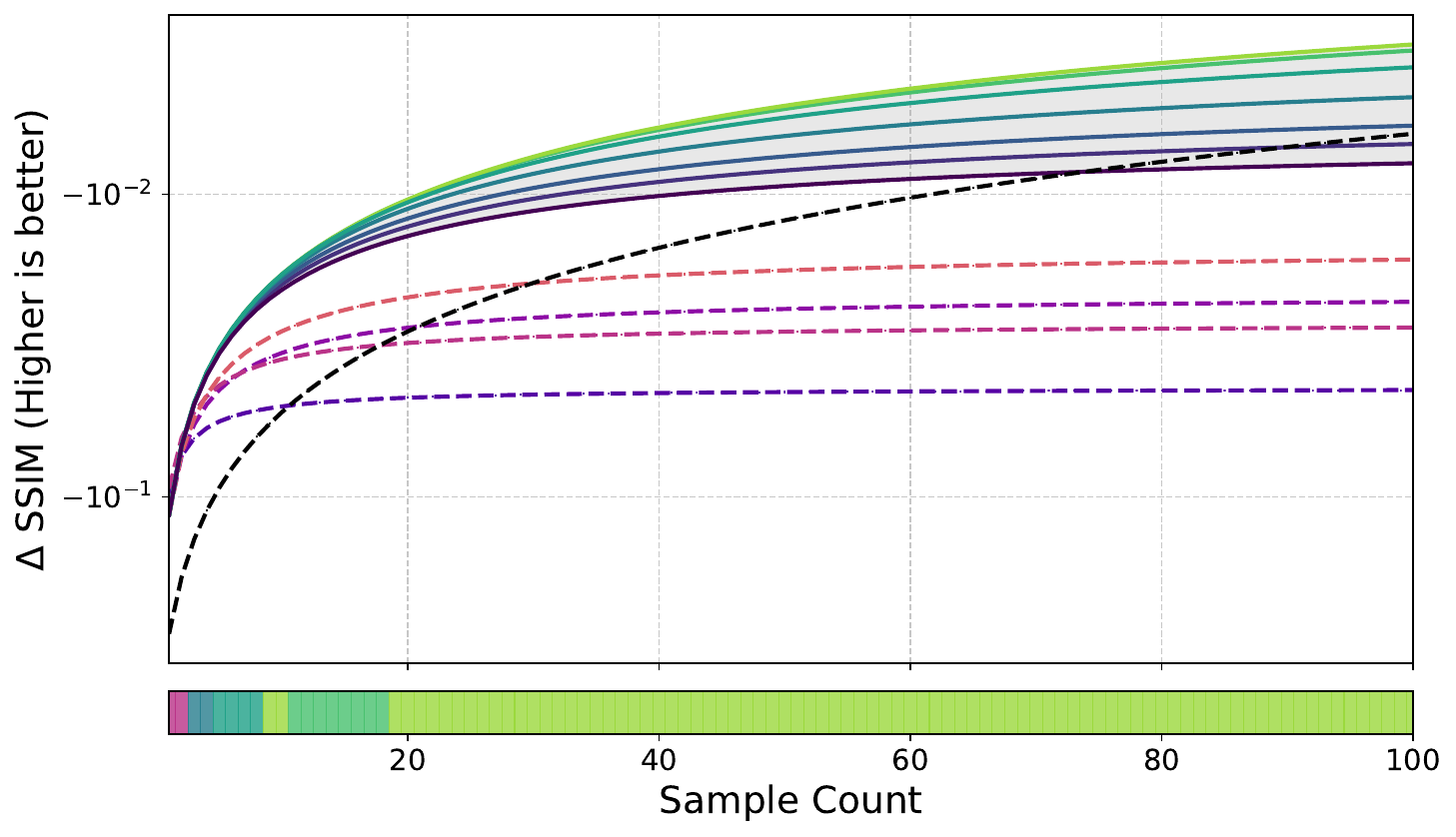}
    \includegraphics[width=0.99\linewidth]{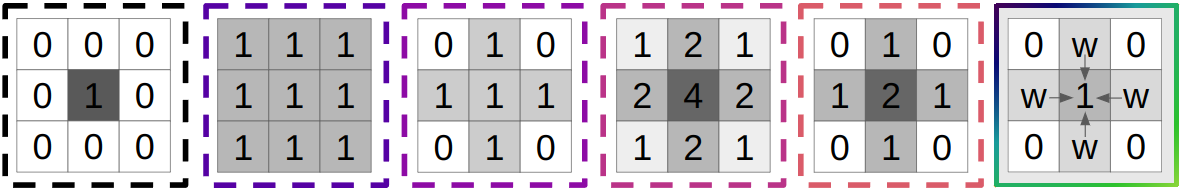}
    \includegraphics[trim={0 0 0 0.75cm},clip,width=0.99\linewidth]{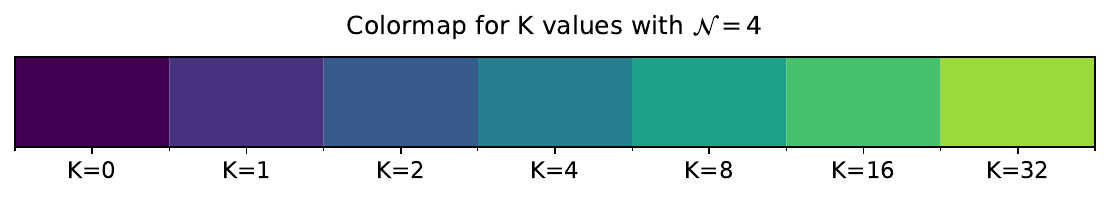}
    \caption{Average residual of test quality metrics on the Mip-NeRF 360 dataset trained with $\lambda_{\text{var}}=1$, comparing static blur kernels against spatial reuse weights with varying $K$. Log-scale is used to visualize small differences.}
    \label{fig:ChoiceOfK}
\end{figure}

\paragraph{Performance.}
While the quality gains for integrating spatial reuse are significant, they also come at a small cost in performance: \cref{fig:PerfKBuffer} compares the performance overhead of spatial reuse against basic stochastic and deterministic blending, implemented in the same optimized Vulkan renderer for fair comparison and run on an RTX 5090. The results show that even for relatively large buffer sizes ($K < 14$), our method achieves higher frame rates than deterministic $\alpha$-blending. The overhead of spatial reuse increases slowly until $K=14$, at which point the shared memory is filled up.

\begin{figure}[t]
    \centering
    \includegraphics[width=0.99\linewidth]{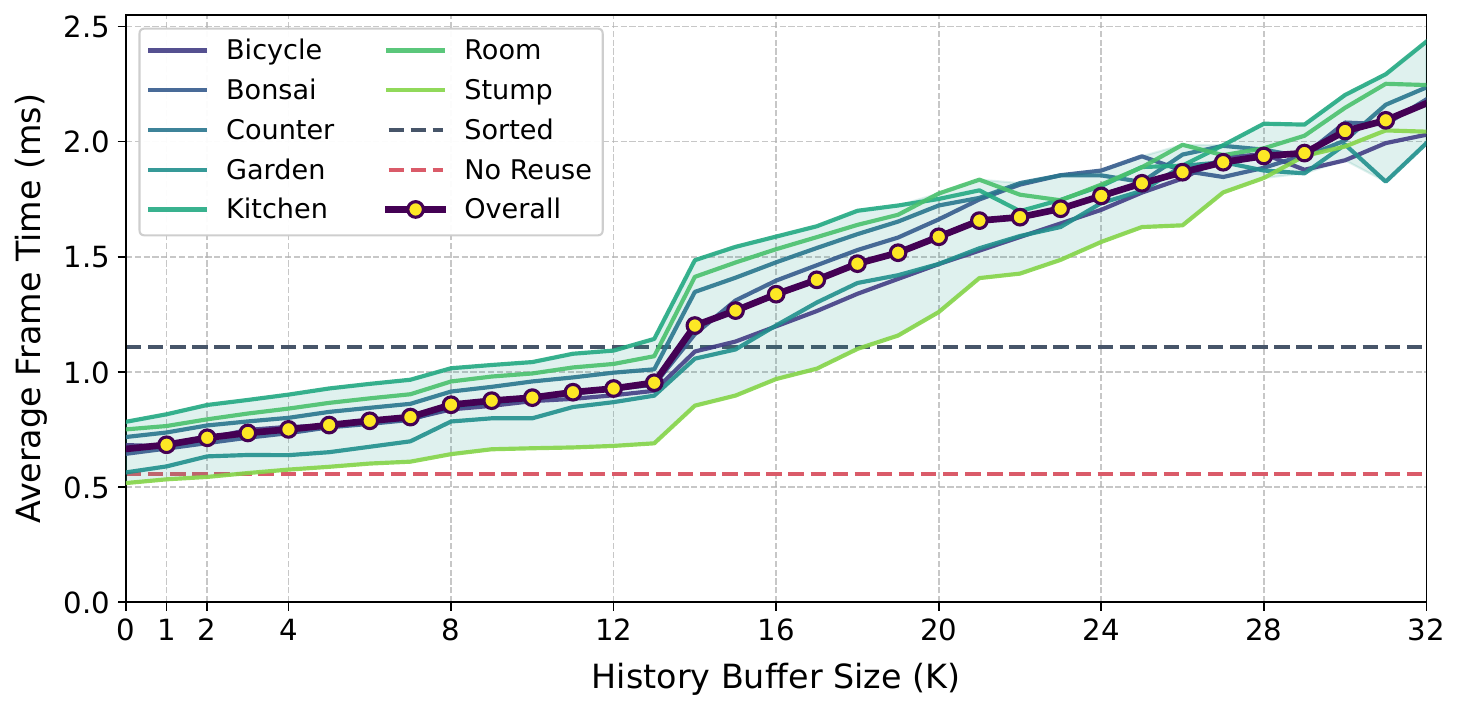}
    \caption{Average frame time for rendering all test views of the Mip-NeRF 360 dataset. We average the frame time for 200 frames per view.}
    \label{fig:PerfKBuffer}
\end{figure}

In order to evaluate the performance-quality trade-off for spatial sample reuse, we introduce the \emph{excess quality factor}:
An excess quality factor of $Q$ means that enabling spatial resampling achieves the same quality as accumulating $Q \cdot N$ samples in the time it takes to draw $N$ samples (with $M < N$ samples due to the resampling overhead). As such, any $Q > 1$ is a net-gain in quality at identical convergence time. \Cref{fig:FreeLunch} shows the excess quality factor achieved for various values of buffer size $K$ and sample count. The largest benefit is observed for low sample counts, with up to $Q=3.5$. The quality factor for small buffers starts high, but drops more quickly, making larger buffer sizes more optimal for longer convergence windows.
\begin{figure}[t]
    \centering
    \includegraphics[width=0.48\linewidth]{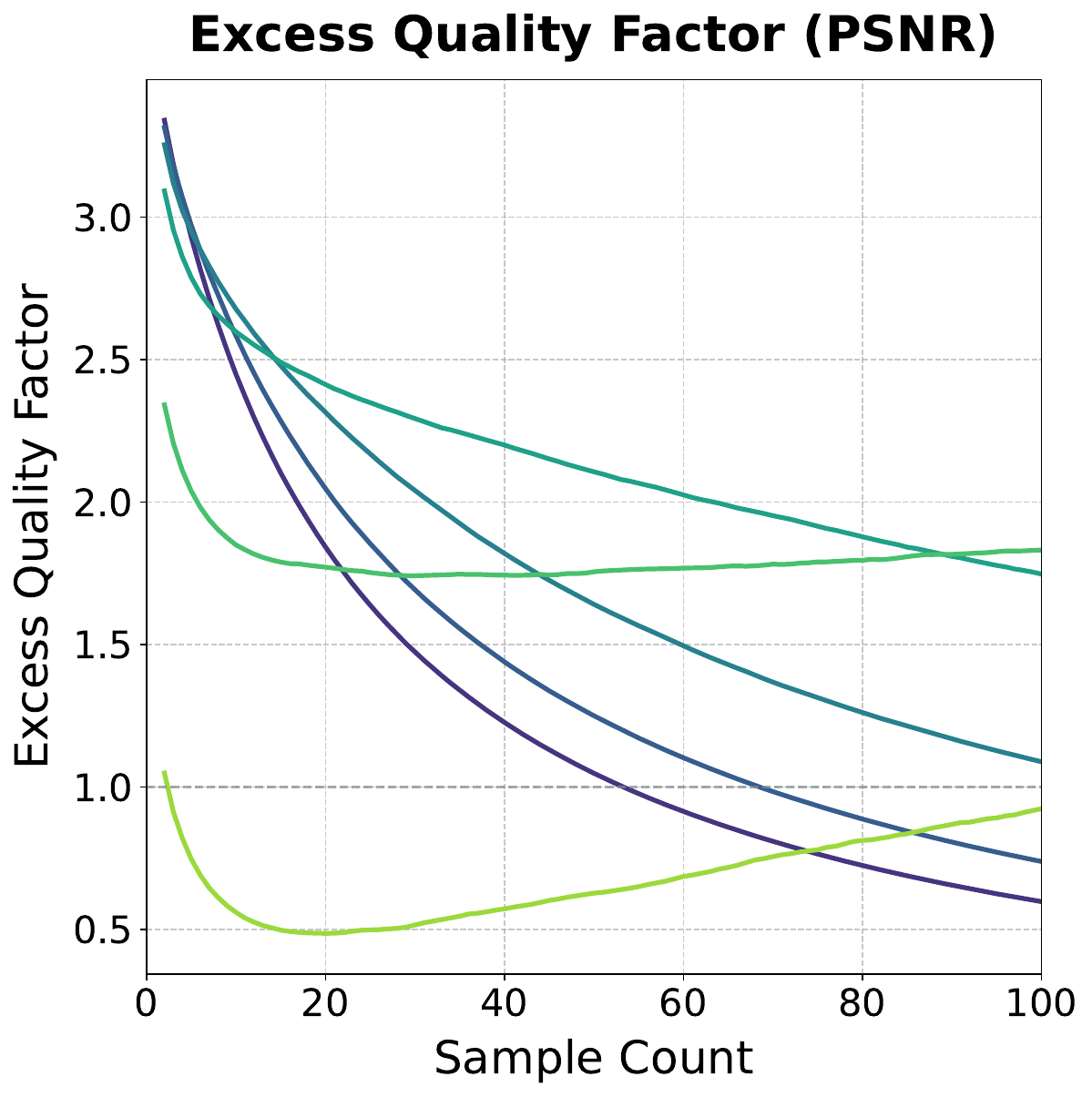}
    \includegraphics[width=0.49\linewidth]{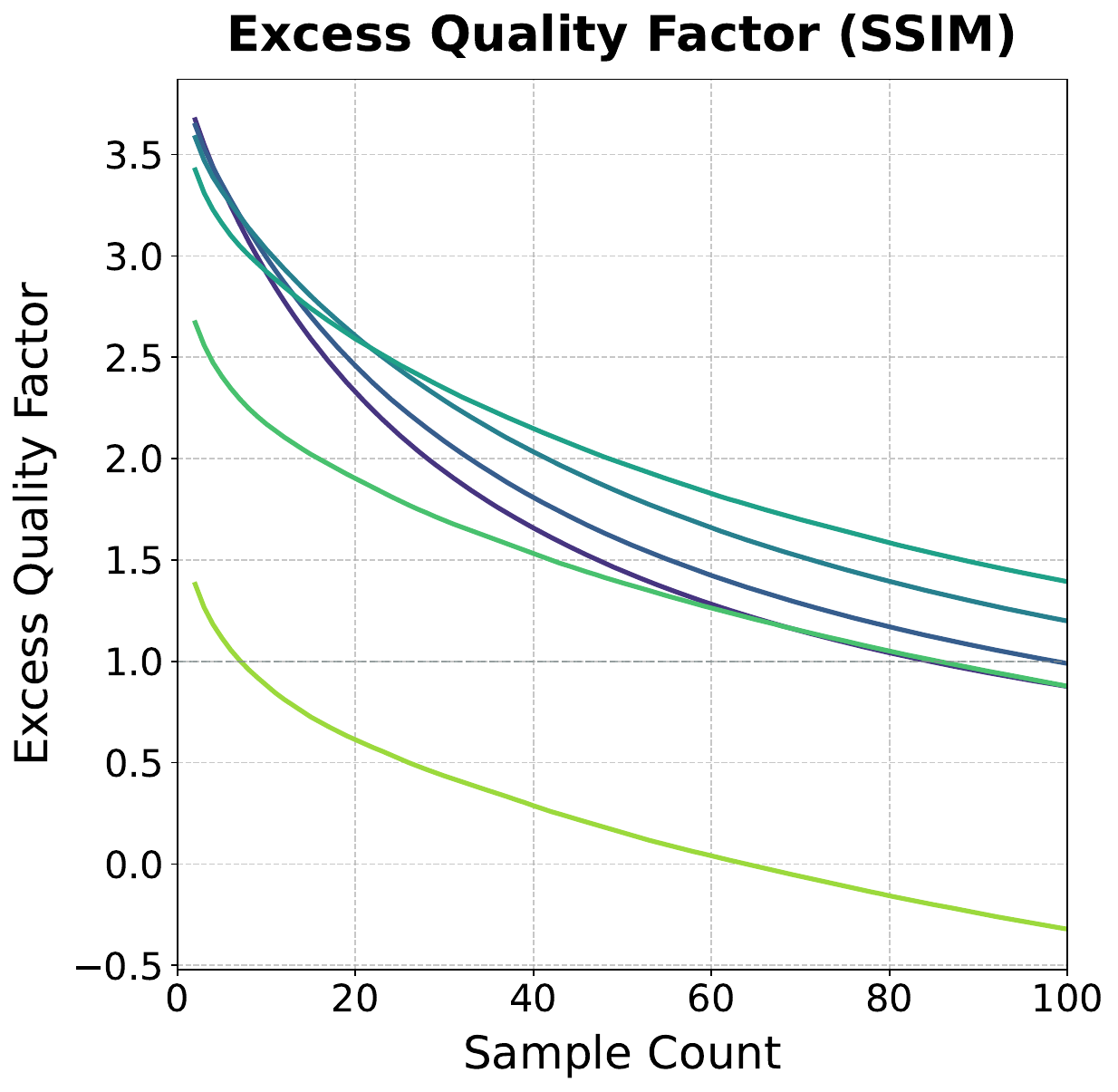}
    \includegraphics[width=0.85\linewidth]{Figures/Reuse_Quality/averaged_metrics_plot_colormap.pdf}
    \caption{Excess Quality Factor on the test views of the Mip-NeRF 360 dataset for PSNR and SSIM quality metrics.}
    \label{fig:FreeLunch}
\end{figure}

\paragraph{STBN.}
\Cref{fig:STBN} compares stochastic rendering quality with and without STBN converging on a static camera. The original RNG used in \SSplats was used as a baseline. The results show that enabling STBN gives a small, but not insignificant boost in quality after roughly three samples that persists for a long time during convergence.

\begin{figure}[t]
    \centering
    \includegraphics[width=0.99\linewidth]{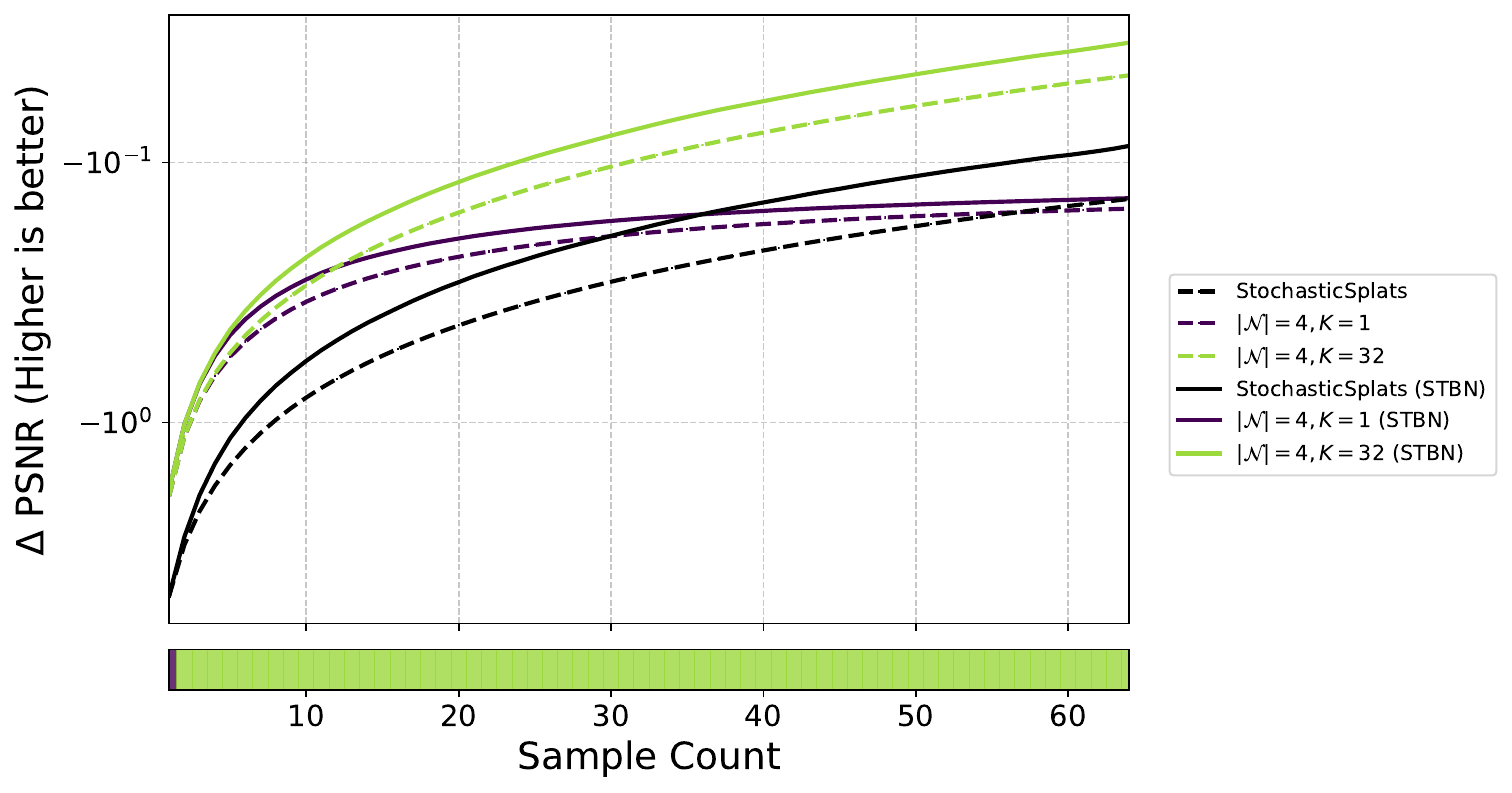}
    \caption{Average residual of test PSNR on the Mip-NeRF 360 dataset trained with $\lambda_{\text{var}}=1$ rendered with and without STBN compared against sorted rendering. Log-scale is used to visualize small differences.}
    \label{fig:STBN}
\end{figure}

\begin{figure*}[!t]
    \centering
    \begin{subfigure}[t]{0.31\linewidth}
        \includegraphics[width=\linewidth]{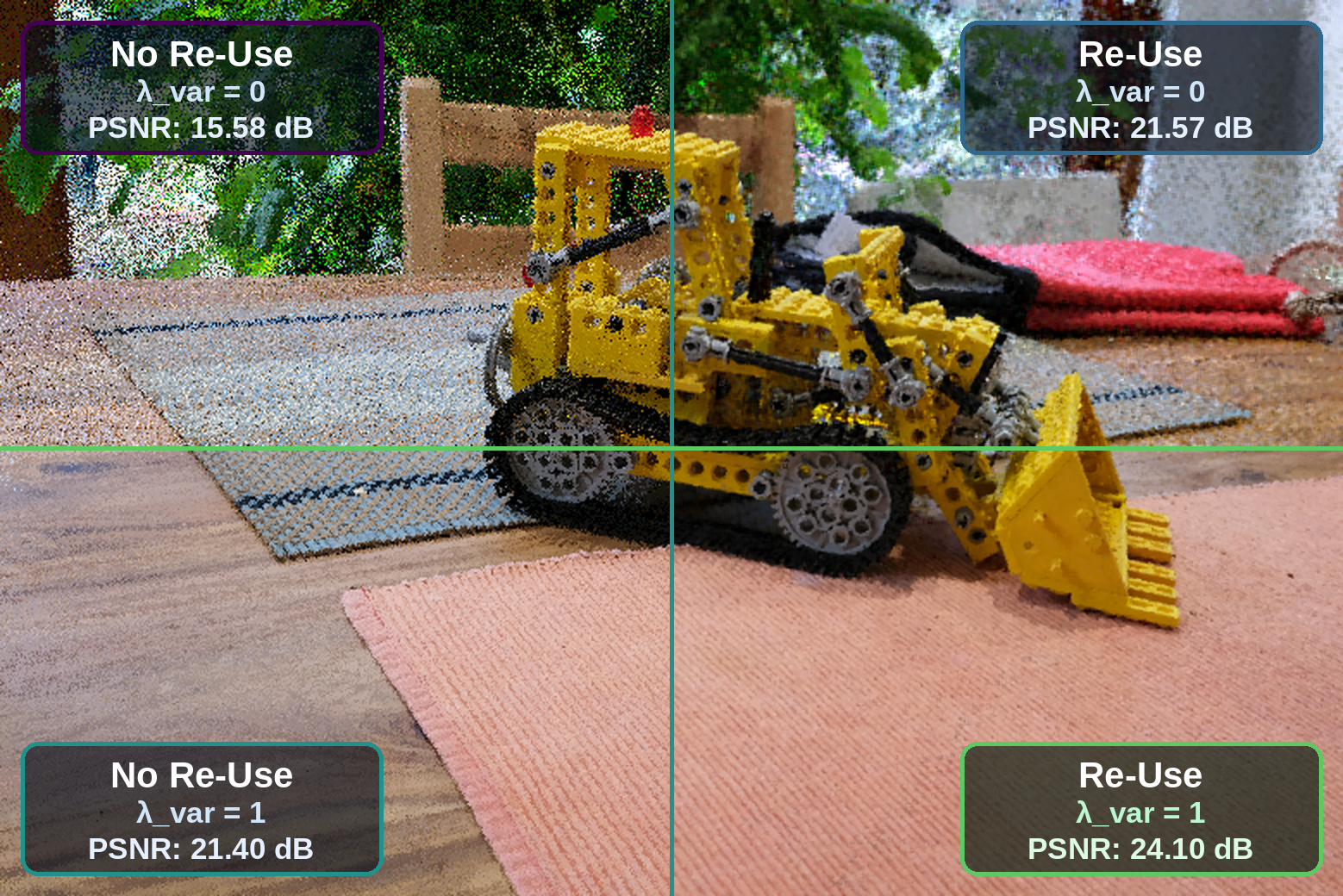}
        \caption{Visual ablation of variance regularization and spatial reuse on 1 sample per pixel.}
    \end{subfigure}
    \hfill
    \begin{subfigure}[t]{0.32\linewidth}
        \includegraphics[width=\linewidth]{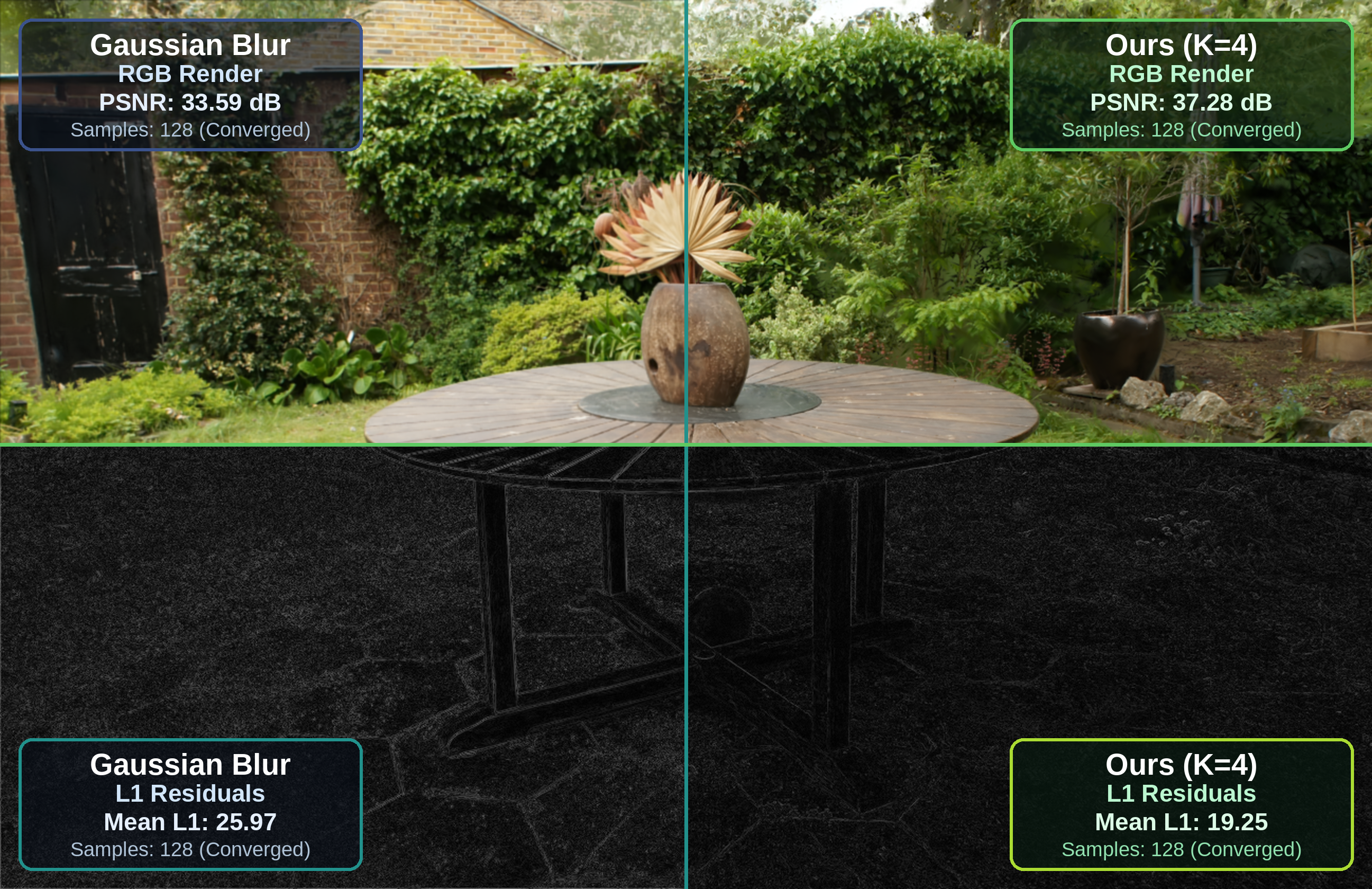}
        \caption{Convergence to sorted 3DGS (128 samples per pixel). Left: Gaussian cross kernel. Right: Spatial reuse with $K=4, |\mathcal{N}|=4$. Bottom: $\ell_1$ error.}
    \end{subfigure}
    \hfill
    \begin{subfigure}[t]{0.31\linewidth}
        \includegraphics[width=\linewidth]{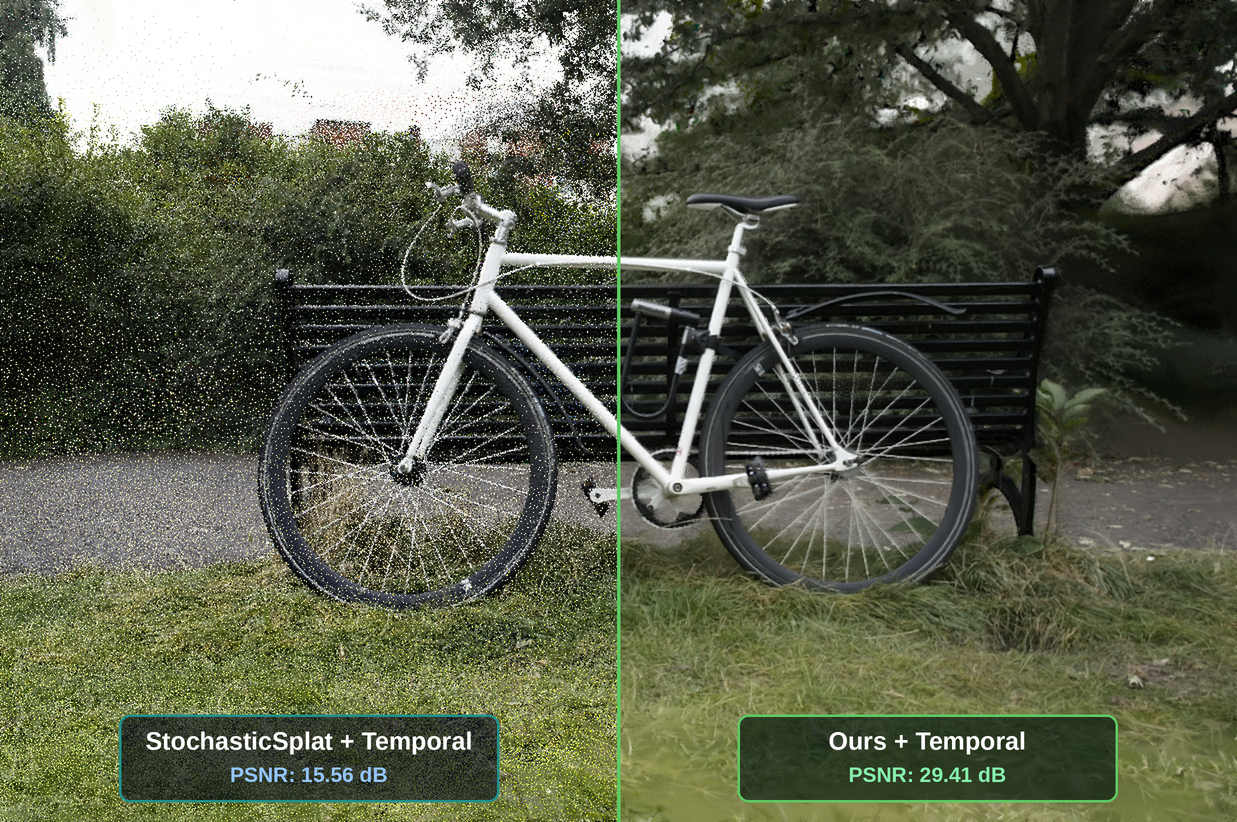}
        \caption{Comparison between \SSplats and Ours with temporal smoothing shows significant noise reduction even under camera motion. }
    \end{subfigure}
    \caption{Visual Ablations.}
    \label{fig:VisualAblation}
\end{figure*}

\begin{table*}[!t]
\centering
\small
\setlength{\tabcolsep}{7pt}
\caption{\textbf{Component ablation on camera path benchmark.} We ablate $\mathcal{L}_{\text{var}}$, Spatial Reuse, and Temporal Reuse combinations across 21 camera trajectories on the Mip-NeRF 360 dataset (3 per scene) against sorted rendering. For Spatial Reuse we set $K=4$, and Temporal Reuse evaluates EMA filtering with our soft history rejection at fixed $\bar{\delta} = 0.70$. STBN was used for all evaluations.}
\label{tab:three_factor_ablation}
\begin{tabular}{ccc cccc}
\toprule
$\bm{\mathcal{L}}_{\textbf{var}}$ & \textbf{Spatial Reuse} & \textbf{Temporal Reuse} & \textbf{tPSNR (dB)} $\uparrow$ & \textbf{SSIM} $\uparrow$ & \textbf{LPIPS} $\downarrow$ & \textbf{Avg. PSNR (dB)} $\uparrow$ \\
\midrule
---          & ---          & ---          & 15.235 & 0.293 & 0.667 & 18.336 \\
---          & \checkmark   & ---          & 21.720 & 0.586 & 0.498 & 24.774 \\
---          & ---          & \checkmark   & 22.204 & 0.619 & 0.522 & 25.227 \\
---          & \checkmark   & \checkmark   & 27.188 & 0.817 & 0.365 & 29.946 \\
\midrule
\checkmark   & ---          & ---          & 22.423 & 0.613 & 0.525 & 25.561 \\
\checkmark   & \checkmark   & ---          & 28.410 & 0.845 & \second{0.345} & 31.437 \\
\checkmark   & ---          & \checkmark   & \second{28.980} & \second{0.863} & 0.357 & \second{31.953} \\
\checkmark   & \checkmark   & \checkmark   & \first{32.449} & \first{0.937} & \first{0.209} & \first{35.085} \\
\bottomrule
\end{tabular}
\end{table*}

\subsection{Camera Motion}

Since temporal filtering cannot be evaluated on isolated static views, we evaluate all methods on continuous camera sequences across 7 scenes of the Mip-NeRF 360 dataset. Following standard practice, we evaluate on the training view resolution, where indoor scenes are downsampled by $2\times$ and outdoor scenes by $4\times$. For each scene, we generate 3 continuous camera paths (21 trajectories, totaling 3,152 frames) by smoothly interpolating keyframe poses from the dataset at 60~FPS. Ground truth sequences are rendered along the exact same trajectories using deterministic sorted 3DGS. We report standard reconstruction metrics (Avg. PSNR, SSIM, and LPIPS) to measure visual quality, as well as temporal PSNR (tPSNR) on inter-frame differences to measure temporal stability and flicker during motion. As summarized in \cref{tab:main_pipeline_comparison}, our pipeline outperforms original \SSplats by a large margin in the temporal domain, yielding significantly higher temporal stability and reduced flicker under continuous camera motion. The significant noise reduction of our other components leads to significant gains in quality under motion even without temporal smoothing.

\paragraph{Component Ablation.}
To evaluate the individual contribution, we conduct a full ablation over our three core components: the variance regularization loss $\mathcal{L}_{\text{var}}$, spatial reuse, and temporal reuse. As reported in \cref{tab:three_factor_ablation}, each component in isolation provides a substantial gain over the raw 1-sample baseline. Enabling all three components achieves the best performance across all metrics. This demonstrates that the three components work synergistically: 
Enabling variance regularization drops the expected number of occluders for each sample (cf.~\cref{fig:occluders}), which in turn increases the accuracy of our reuse weight approximation through historical samples. Further, it reduces noise not only in RGB, but also in depth (cf.~\cref{fig:depth}), which strengthens the accuracy of temporal sample reprojection. Both the variance regularizer and spatial reuse reduce output variance, enabling the dynamic blending weight (cf.~\cref{eq:soft_alpha}) to more effectively distinguish between variance caused by disocclusions or by sample noise, thereby reducing ghosting artifacts. 

We refer to the supplementary material for extensive parameter ablations and per-scene metric breakdowns.

\Cref{fig:VisualAblation} shows visual ablation of all our components.

\section{Discussion}
\label{sec:discussion}

In this work, we have demonstrated that the severe high-frequency noise inherently produced by stochastic order-independent transparency (OIT) rendering methods can be effectively suppressed during both training and inference time with minimal bias. By introducing color variance regularization during optimization and a history-based spatial resampling scheme during inference, our method achieves clean, high-fidelity rendering from just a single primitive sample per pixel.

\paragraph{Different Primitives.}
Our method is not limited strictly to 3D Gaussian Splatting. The principles of our approach could be used on any semi-transparent rasterized primitive that typically requires computationally expensive depth-sorting, extending the benefits of efficient stochastic rendering to a broader range of representations.

\subsection{Limitations}
Despite the massive acceleration in convergence and reduction in noise, our approach introduces a few notable trade-offs:
\begin{itemize}
    \item \textbf{Quality Metric Trade-off:} Our color variance loss ($\mathcal{L}_{\text{var}}$) effectively mitigates noise but also dilutes the training signal of the L1 and SSIM photometric losses, resulting in a slight decrease in peak quality metrics for fully converged models.
    \item \textbf{Overhead of Reuse:} While spatial resampling avoids the exhaustive global sorting of primitives, populating and reading from the K-buffer to evaluate neighboring samples takes additional memory and computational time during inference.
    \item \textbf{Loss of Strict Unbiasedness:} Standard stochastic transparency is unbiased. However, to maintain high performance, our spatial reuse implementation approximates the exact transmittance using only a limited set of $K$ previous samples. Because of this approximation, the spatial resampling is no longer strictly unbiased, exchanging a small degree of mathematical correctness for significantly accelerated convergence.
\end{itemize}

\subsection{Future Work}
Future research should explore dynamic or adaptive sizing of both the spatial neighborhood $\mathcal{N}(x,y)$ and the K-buffer capacity. By intelligently allocating K-buffer size based on local scene complexity or depth variance, it may be possible to reduce the memory overhead and time required for sample reuse without sacrificing noise reduction. Additionally, investigating adaptive weighting for the variance loss could help minimize the trade-off in photometric quality metrics, allowing for optimal visual fidelity across varying primitive distributions. A fast approximation of the full spatio-temporal SNIS estimator could further improve the temporal quality as well.

\section*{Acknowledgements}
We thank Alex Evans for his valuable help in the early stages of the project.

\clearpage

{
    \small
    \bibliographystyle{ieeenat_fullname}
    \bibliography{main}
}

\clearpage
\setcounter{page}{1}
\renewcommand{\thepage}{S\arabic{page}}
\maketitlesupplementary

\noindent In this supplementary document, we provide full algorithm details, additional proofs, and extended ablation results.

\section{Pseudocode}

\Cref{alg:Stochastic} details the sampling procedure introduced in \SSplats to sample $r \sim P_{x,y}$.

\begin{algorithm}[h]
\caption{\SSplats Sampling \cite{kheradmand2025stochasticsplats}}
\label{alg:Stochastic}
\begin{algorithmic}[1] 
\State $r \gets \varnothing, z_r \gets \infty$ \Comment{$\varnothing$: background, $c_\varnothing = c_{\text{background}}$}
\For{each Gaussian $i \in \mathcal{G}(x,y)$ in unsorted order}
\State $u_i \sim U(0, 1)$
\If{$u_i < \alpha_i(x,y)$ \textbf{and} $z_i < z_r$}
    \State $r \gets i$
    \State $z_r \gets z_i$
\EndIf
\EndFor
\State \Return $r, z_r$
\end{algorithmic}
\end{algorithm}

If no fragment survives, \cref{alg:Stochastic} returns the background ($r = \varnothing$). This happens with probability $\prod_{k \in \mathcal{G}(x,y)} (1-\alpha_k(x,y))$, the mass that is missing from $P_{x,y}$; the color variance in \cref{eq:Variance} and $\Lvar$ assume full coverage, where this mass vanishes.

\Cref{alg:KBuffer} gives a full overview of how one frame is rendered with spatial resampling. It is written in scatter form: each pixel $(x,y)$ weights its own sample $r$ for every neighbor $n$ with $\wgt{x,y}{n}(r)$, using its own buffer $\mathcal{K}_{x,y}$, and accumulates the result at $n$. Up to the order of summation, this equals the gather form of \cref{eq:SpatialReuse}, in which the sample of each neighbor $n$ is weighted by $\wgt{n}{x,y}(r_n)$ using $\mathcal{K}_n$, plus the pixel's own sample with weight one.

\begin{algorithm}[h]
\caption{Spatial Resampling}
\label{alg:KBuffer}
\begin{algorithmic}[1]
\State $r, z_r \gets$ \cref{alg:Stochastic} \Comment{$r \sim P_{x,y}$ and its depth}
\State $\text{unique} \gets \textbf{true}$
\For{$n \in \mathcal{N}(x,y)$}
    \State $\wgt{x,y}{n} \gets \frac{\alpha_r(n)}{\alpha_r(x,y)}$
\EndFor
\For{each $k \in \mathcal{K}_{x,y}$} \Comment{Stored: $k$, $z_k$, $\occ{k}{x,y}{n}$}
    \If{$k == r$}
        \State $\text{unique} \gets \textbf{false}$
    \EndIf
    \If{$z_k < z_r$}
        \For{$n \in \mathcal{N}(x,y)$}
            \State $\wgt{x,y}{n} \gets \wgt{x,y}{n} \cdot \occ{k}{x,y}{n}$
        \EndFor
    \EndIf
\EndFor
\If{$\text{unique}$}
    \State insert $r, z_r$ into $\mathcal{K}_{x,y}$, evicting the oldest entry if $|\mathcal{K}_{x,y}| = K$
    \For{$n \in \mathcal{N}(x,y)$}
        \State $\occ{r}{x,y}{n} \gets \frac{1-\alpha_r(n)}{1-\alpha_r(x,y)}$
    \EndFor
\EndIf
\State $S_{x,y} \gets c_r$
\State $W_{x,y} \gets 1$
\State $\textsc{Sync}()$ \Comment{Wait for all writes to finish}
\For{$n \in \mathcal{N}(x,y)$}
    \State $S_{n} \gets S_{n} + \wgt{x,y}{n} \cdot c_r$
    \State $W_{n} \gets W_n + \wgt{x,y}{n}$
\EndFor
\State $\textsc{Sync}()$ \Comment{Wait for all writes to finish}
\State \Return $\Cspat{x,y} = \frac{S_{x,y}}{W_{x,y}}$
\end{algorithmic}
\end{algorithm}

\section{Temporal Recurrence and Soft Rejection Implementation}
\label{sec:appendix_temporal}

This section provides implementation details for the temporal recurrence pass of \cref{sec:spatiotemporal_resampling}, executed immediately after spatial resampling (\cref{sec:spatial}) in the same compute pipeline.

\paragraph{Reprojection and History Fetch.}
For each pixel $p = (x,y)$ with depth $z_p$ of its sampled fragment, we reconstruct its view-space position and reproject it into the previous frame's screen space using the prior view-projection matrix:
\begin{equation}
p' = (x',y') = \pi_{t-1}\big(\pi_t^{-1}(p, z_p)\big) \, .
\end{equation}
The screen-space motion magnitude is $v = \|p - p'\|_2$. 
If $p'$ lies outside the viewport or within two pixels of its boundary, temporal accumulation is bypassed and the spatial estimate $\Cspat[t]{p}$ is output directly. 
Otherwise, we retrieve the historical radiance $\hat{C}^{t-1}_{p'} = \textsc{Sample}(I_{t-1}, p')$ using Catmull-Rom bicubic sampling to minimize interpolation blur.

\paragraph{Neighborhood Moments in Color Space.}
To estimate the local radiance distribution, we evaluate the first and second moments of the current samples $c_{r_q}$ over a $3 \times 3$ window $\mathcal{N}_3(p)$ centered at $p$:
\begin{align}
\bar{C}_p &= \frac{1}{9} \sum_{q \in \mathcal{N}_3(p)} c_{r_q} \, , \\ 
s_p &= \sqrt{\max\Big(0,\, \frac{1}{9} \sum_{q \in \mathcal{N}_3(p)} c_{r_q}^2 - \bar{C}_p^2\Big)} \, .
\end{align}
Moments are computed channel-wise in YCoCg space to decouple luminance variance from chromaticity. 
The confidence half-extents are:
\begin{equation}
e = \max(\gamma s_p, e_{\min}) \, ,
\end{equation}
where $e_{\min} = (0.01, 0.005, 0.005)$ prevents degenerate zero-volume bounds on flat regions.

\paragraph{Velocity-Adaptive Confidence Scaling ($\gamma$).}
Stationary regions require wider confidence bounds to preserve accumulated history and avoid premature rejection, whereas moving regions require tighter bounds to flush stale history across disocclusions. 
Following established real-time practice~\cite{yang2020survey}, we dynamically scale $\gamma$ with screen velocity $v$:
\begin{equation}
\gamma = \lerp\left(\gamma_{\mathrm{static}}, \gamma_{\mathrm{mot}}, \min\left(1,\, \frac{v}{v_0}\right)\right) \, ,
\end{equation}
where $\lerp(a, b, \theta) = (1-\theta)\,a + \theta\, b$, $\gamma_{\mathrm{static}} = 4.0$ covers $99.99\%$ of a normal distribution on stationary surfaces, $\gamma_{\mathrm{mot}} = 1.0$ tightens the rejection envelope under movement, and $v_0 = 4.0$\,px is the velocity threshold.

\paragraph{Statistical Distance and Soft History Blending.}
Given the discrepancy vector $d = \hat{C}^{t-1}_{p'} - \bar{C}_p$, we compute its $L_\infty$ distance from the neighborhood mean, normalized by the confidence half-extents so that $m = 1$ on the boundary of the confidence box:
\begin{equation}
m = \max\left(\left|\frac{d_{\mathrm{Y}}}{e_{\mathrm{Y}}}\right|,\, \left|\frac{d_{\mathrm{Co}}}{e_{\mathrm{Co}}}\right|,\, \left|\frac{d_{\mathrm{Cg}}}{e_{\mathrm{Cg}}}\right|\right) \, .
\end{equation}
When $m > 1$, the historical sample is an outlier. Rather than projecting $\hat{C}^{t-1}_{p'}$ onto the box boundary along $d / m$ (which biases color ratios), we compute a continuous retention factor $\rho = 1 / m \in (0, 1)$ and attenuate the history weight as in \cref{eq:soft_alpha}:
\begin{equation}
\delta_{p}^t = 
\begin{cases}
\bar{\delta} & \text{if } m \le 1 \, , \\
\rho \, \bar{\delta} & \text{if } m > 1 \, ,
\end{cases}
\end{equation}
with nominal history weight $\bar{\delta} = 0.70$, \ie, the new spatial estimate receives a nominal weight of $1 - \bar{\delta} = 0.30$. 
The final temporal radiance is obtained by blending the unclipped historical color directly, which is \cref{eq:ema} with the history read at $p'$ (\cref{par:camera_motion}):
\begin{equation}
\hat{C}_p^t = \delta_{p}^t \, \hat{C}^{t-1}_{p'} + (1 - \delta_{p}^t) \, \Cspat[t]{p} \, .
\end{equation}
\Cref{alg:temporal_soft_rejection} summarizes the complete procedure.

\begin{algorithm}[t]
\caption{Temporal Resampling with Soft History Rejection}
\label{alg:temporal_soft_rejection}
\begin{algorithmic}[1]
\small
\Require Spatial estimate $\Cspat[t]{p}$, current samples $c_{r_q}$, depth $z_p$, history texture $I_{t-1}$, camera projections $\pi_t, \pi_{t-1}$, nominal history weight $\bar{\delta}$, parameters $\gamma_{\mathrm{static}}, \gamma_{\mathrm{mot}}, v_0, e_{\min}$
\Ensure Filtered radiance $\hat{C}^t_p$
\Statex \textbf{// 1. Backward Reprojection}
\State $p' \gets \pi_{t-1}\big(\pi_t^{-1}(p, z_p)\big)$
\State \textbf{if} $p'$ outside viewport \textbf{then return} $\Cspat[t]{p}$
\State $v \gets \|p - p'\|_2$ \Comment{Velocity magnitude in pixels}
\State $\hat{C}^{t-1}_{p'} \gets \textsc{Sample}(I_{t-1}, p')$
\Statex \textbf{// 2. Neighborhood Moments (channel-wise)}
\State $\bar{C}_p \gets \frac{1}{9} \sum_{q \in \mathcal{N}_3(p)} c_{r_q}$
\State $s_p \gets \sqrt{\max\Big(0,\, \frac{1}{9} \sum_{q \in \mathcal{N}_3(p)} c_{r_q}^2 - \bar{C}_p^2\Big)}$
\Statex \textbf{// 3. Velocity-Adaptive Confidence Bounds}
\State $\gamma \gets \lerp\big(\gamma_{\mathrm{static}},\, \gamma_{\mathrm{mot}},\, \min(1.0,\, v / v_0)\big)$
\State $e \gets \max(\gamma s_p,\, e_{\min})$
\Statex \textbf{// 4. Normalized Discrepancy \& Soft Rejection}
\State $d \gets \hat{C}^{t-1}_{p'} - \bar{C}_p$
\State $m \gets \max\big(|d_{\mathrm{Y}} / e_{\mathrm{Y}}|,\, |d_{\mathrm{Co}} / e_{\mathrm{Co}}|,\, |d_{\mathrm{Cg}} / e_{\mathrm{Cg}}|\big)$
\If{$m > 1.0$}
    \State $\rho \gets 1.0 / m$ \Comment{Retention factor}
    \State $\delta_{p}^t \gets \rho\, \bar{\delta}$ \Comment{Attenuate history weight}
\Else
    \State $\delta_{p}^t \gets \bar{\delta}$ \Comment{Inlier sample}
\EndIf
\Statex \textbf{// 5. Temporal Blend}
\State $\hat{C}^t_p \gets \delta_{p}^t\, \hat{C}^{t-1}_{p'} + (1 - \delta_{p}^t)\, \Cspat[t]{p}$
\State \Return $\hat{C}^t_p$
\end{algorithmic}
\end{algorithm}

\section{Additional Experiments}

\subsection{Effects of \texorpdfstring{$\Lvar$}{L\_var}}
In this section, we examine the impact of $\Lvar$ on the trained Gaussian model.

\begin{itemize}
    \item \textbf{Opacity:} The results from \cref{fig:opacity} show that $\Lvar$ naturally acts as a binarization penalty on opacity, pushing Gaussians towards either full opacity or full transparency.
    \begin{figure}[htbp]
        \centering
        \includegraphics[width=0.9\linewidth]{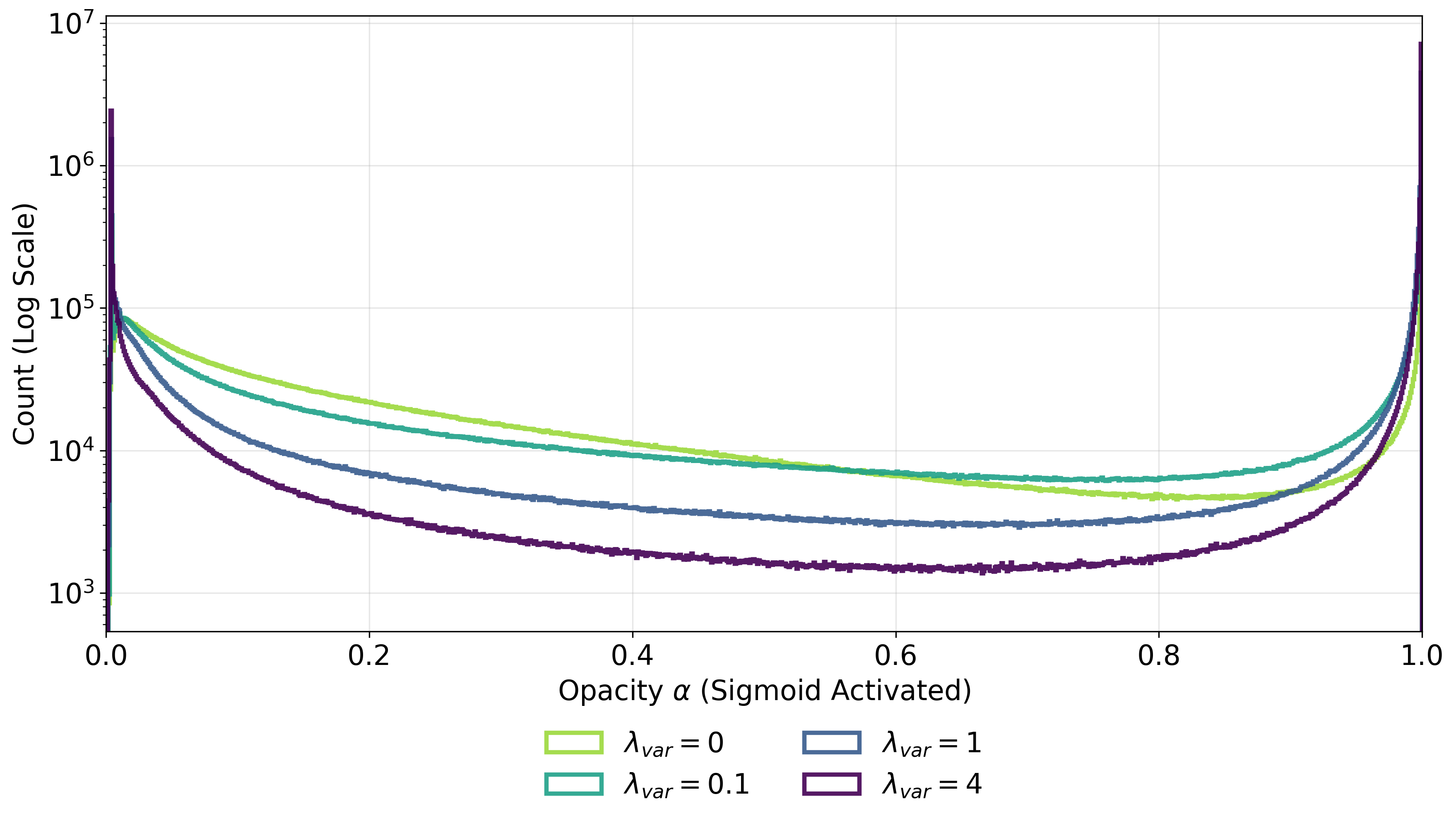}
        \caption{Log-scale histogram of primitive counts with respect to activated opacity over all Mip-NeRF 360 scenes trained with various values of $\lvar$.}
        \label{fig:opacity}
    \end{figure}
    
    \item \textbf{Scale:} Increasing the effect of $\lvar$ shifts the size distribution shown in \cref{fig:scales} towards smaller Gaussians. At the same time, we observe an increase in very large Gaussians for high values of $\lvar$.
    \begin{figure}[htbp]
        \centering
        \includegraphics[width=0.9\linewidth]{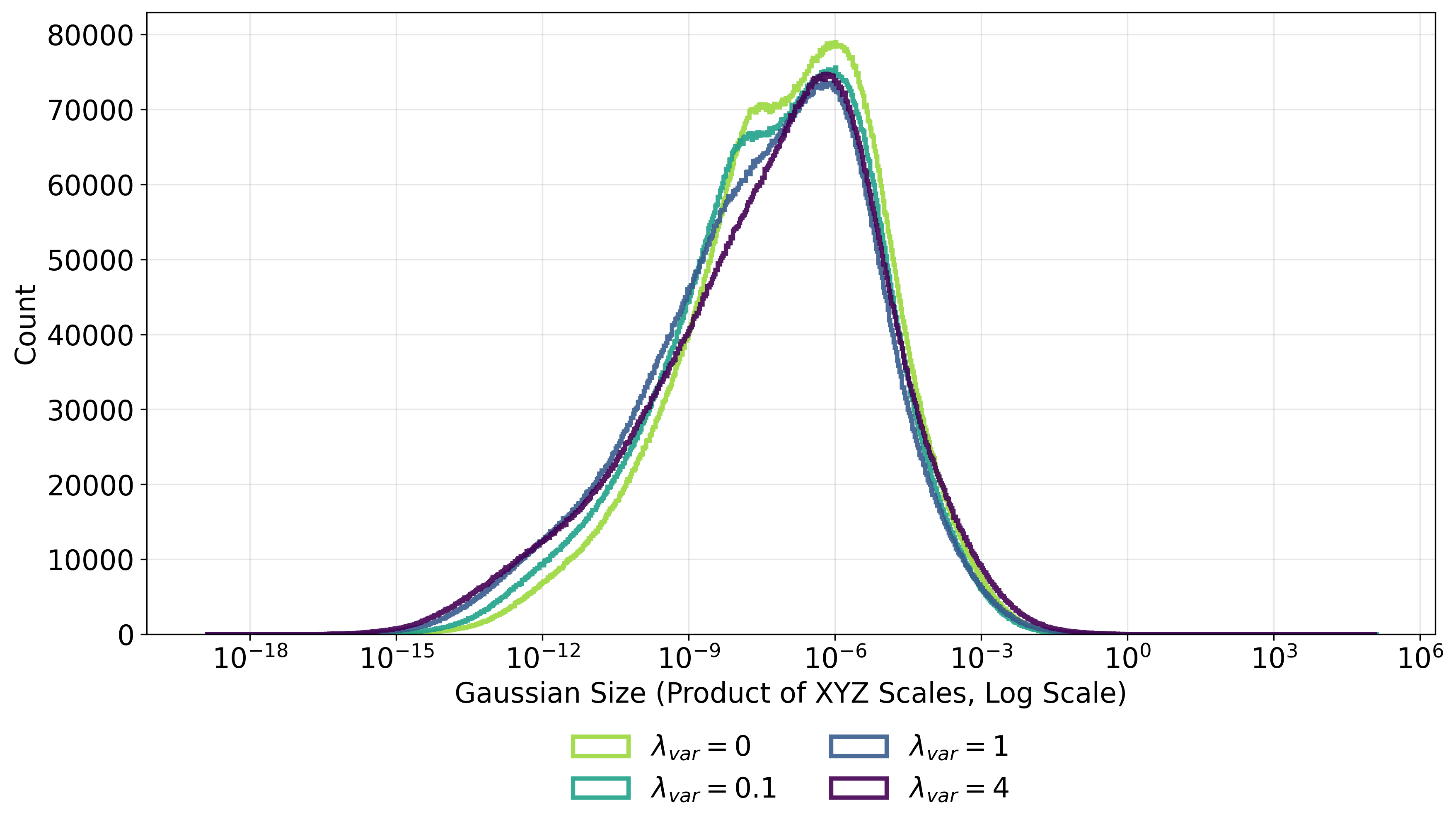}
        \caption{Histogram of primitive counts with respect to Gaussian size (log-scale) over all Mip-NeRF 360 scenes trained with various values of $\lvar$.}
        \label{fig:scales}
    \end{figure}

    \item \textbf{Image Quality:} High values of $\lvar$ restrict the optimizer's ability to use blending for approximating certain view-dependent effects and compensating for inconsistencies in the input images. \Cref{fig:five_images} shows a comparison of rendered images on models trained with varying strength of $\Lvar$. For small values of $\lvar$, the difference is barely noticeable, while for high values a loss of fine details and general blurriness can be observed.

    \begin{figure*}[t]
        \centering
        \begin{subfigure}[b]{0.19\textwidth}
            \centering
            \includegraphics[width=\textwidth]{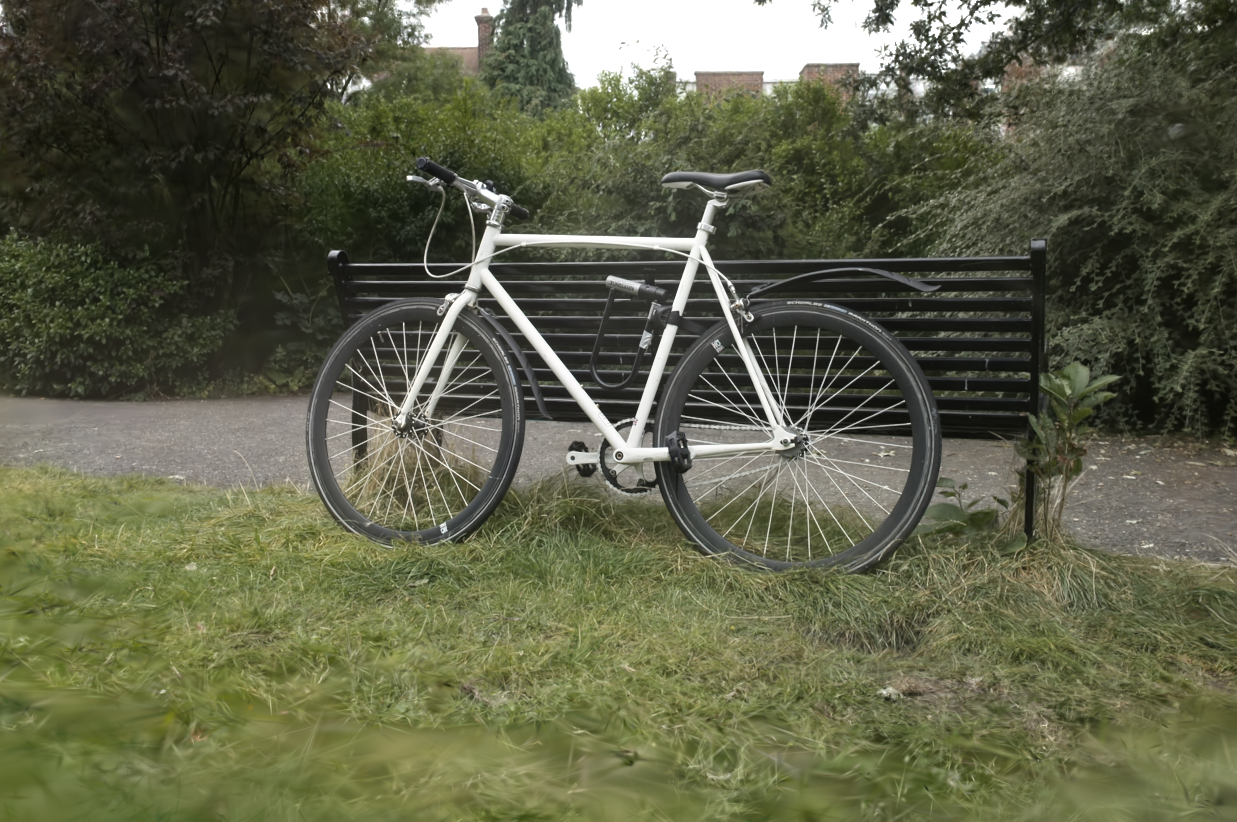}
            \caption{$\lvar = 0$}
            \label{fig:var_bicycle_0}
        \end{subfigure}
        \hfill
        \begin{subfigure}[b]{0.19\textwidth}
            \centering
            \includegraphics[width=\textwidth]{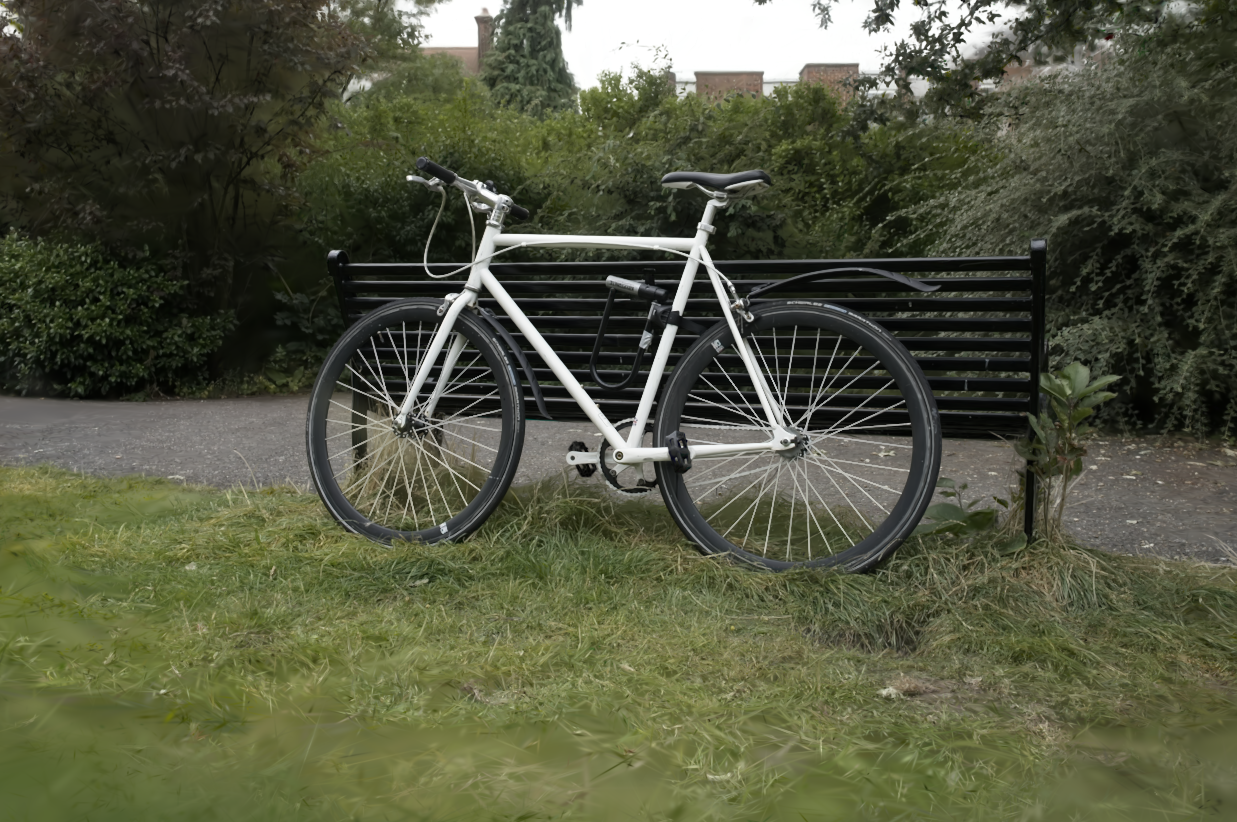}
            \caption{$\lvar = 1$}
            \label{fig:var_bicycle_1}
        \end{subfigure}
        \hfill
        \begin{subfigure}[b]{0.19\textwidth}
            \centering
            \includegraphics[width=\textwidth]{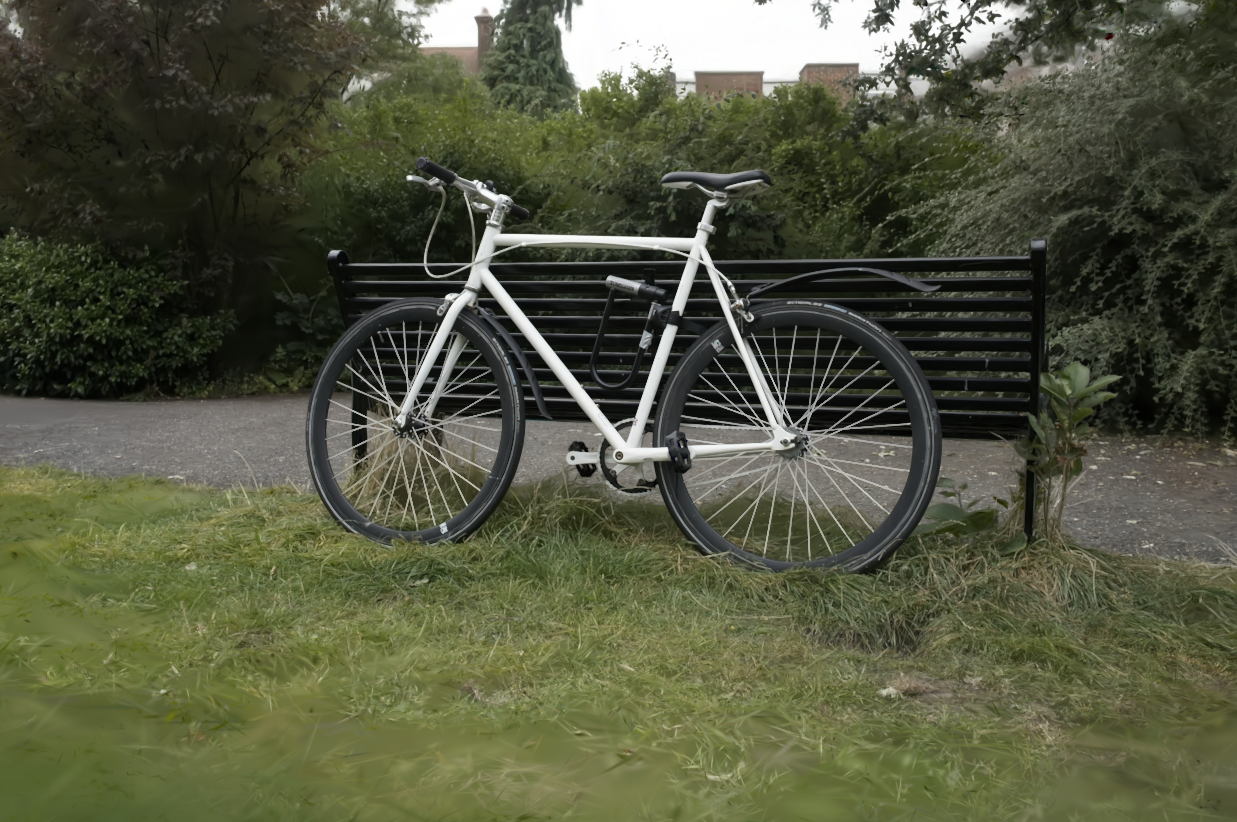}
            \caption{$\lvar = 2$}
            \label{fig:var_bicycle_2}
        \end{subfigure}
        \hfill
        \begin{subfigure}[b]{0.19\textwidth}
            \centering
            \includegraphics[width=\textwidth]{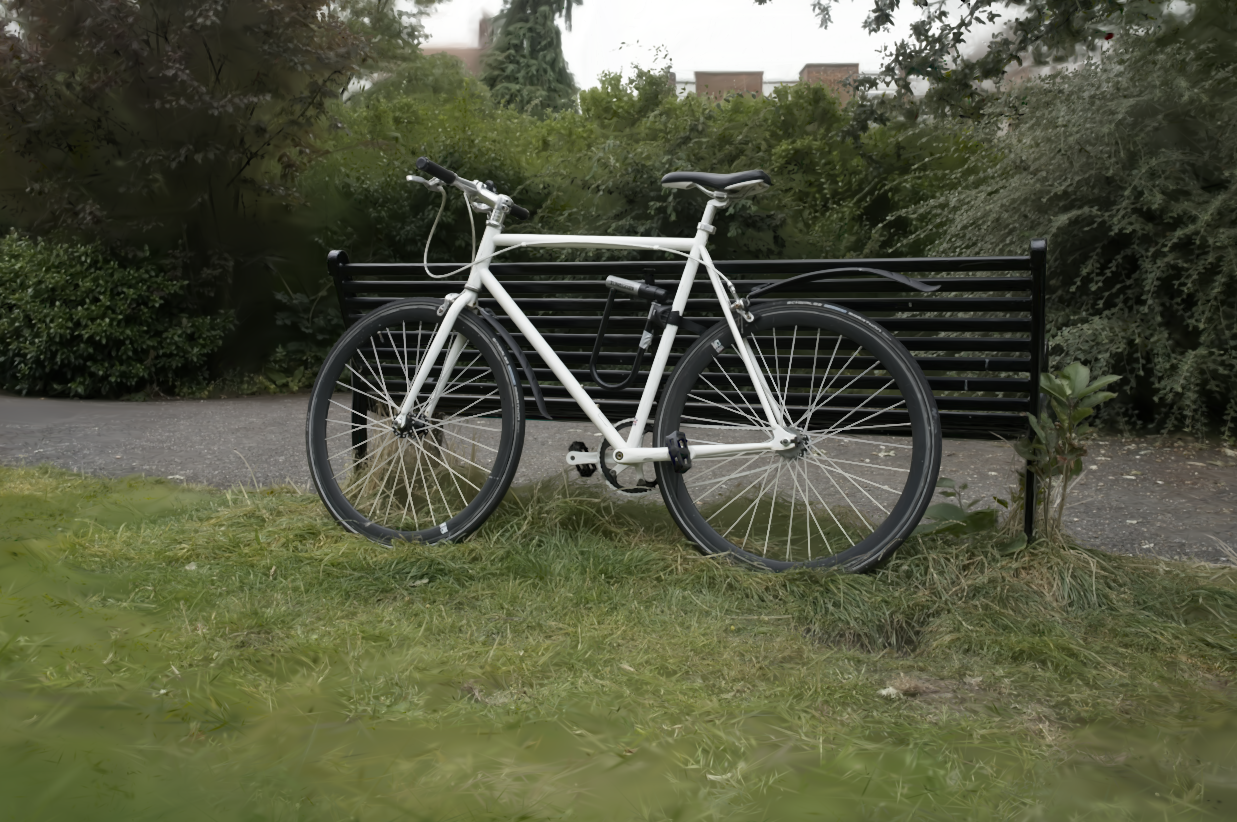}
            \caption{$\lvar = 3$}
            \label{fig:var_bicycle_3}
        \end{subfigure}
        \hfill
        \begin{subfigure}[b]{0.19\textwidth}
            \centering
            \includegraphics[width=\textwidth]{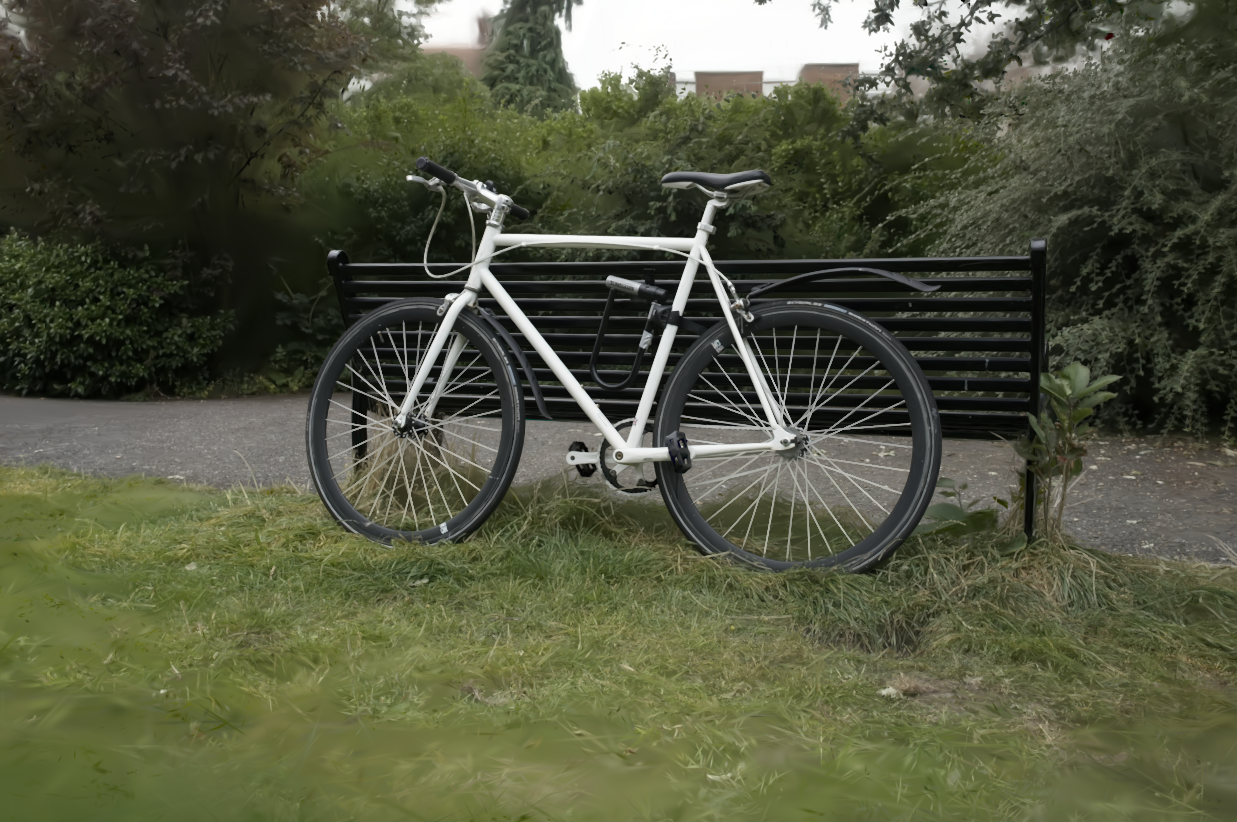}
            \caption{$\lvar = 4$}
            \label{fig:var_bicycle_4}
        \end{subfigure}
        \caption{Test view of the bicycle scene, rendered deterministically with models trained with varying $\lvar$.}
        \label{fig:five_images}
    \end{figure*}

    \item \textbf{Color:} As shown in \cref{fig:colors}, $\Lvar$ suppresses outliers in RGB space. In particular, it reduces the number of bright primitives, which are often floater artifacts.

    \item \textbf{Blending:} We render test images using vanilla 3DGS for all scenes of the Mip-NeRF 360 dataset and store for each pixel the number of Gaussians that are blended before early exit, the expected transmittance, and the expected number of occluders on \SSplats samples. The results in \cref{fig:blending} show that increasing $\lvar$ leads to more Gaussians being blended for each pixel, but also simultaneously decreases the expected number of occluders and expected transmittance. This indicates that the sampling distribution $P_{x,y}$ is shifted towards the camera. Whereas standard 3DGS often represents surfaces as a set of opaque Gaussians and a set of semi-transparent ``fuzz'' on top to add details, the variance loss shifts the radiance mass forward along the ray and thereby greatly reduces the number and opacity of occluders for importance-sampled Gaussians.

    \item \textbf{Depth:} The variance regularizer not only reduces variance in color, but also in depth, as illustrated in \cref{fig:depth}. For more discussion of the geometric effects of $\Lvar$, we refer to Radl et al.~\cite{radl2026confidence}.
    \begin{figure}[h]
        \centering
        \includegraphics[width=\linewidth]{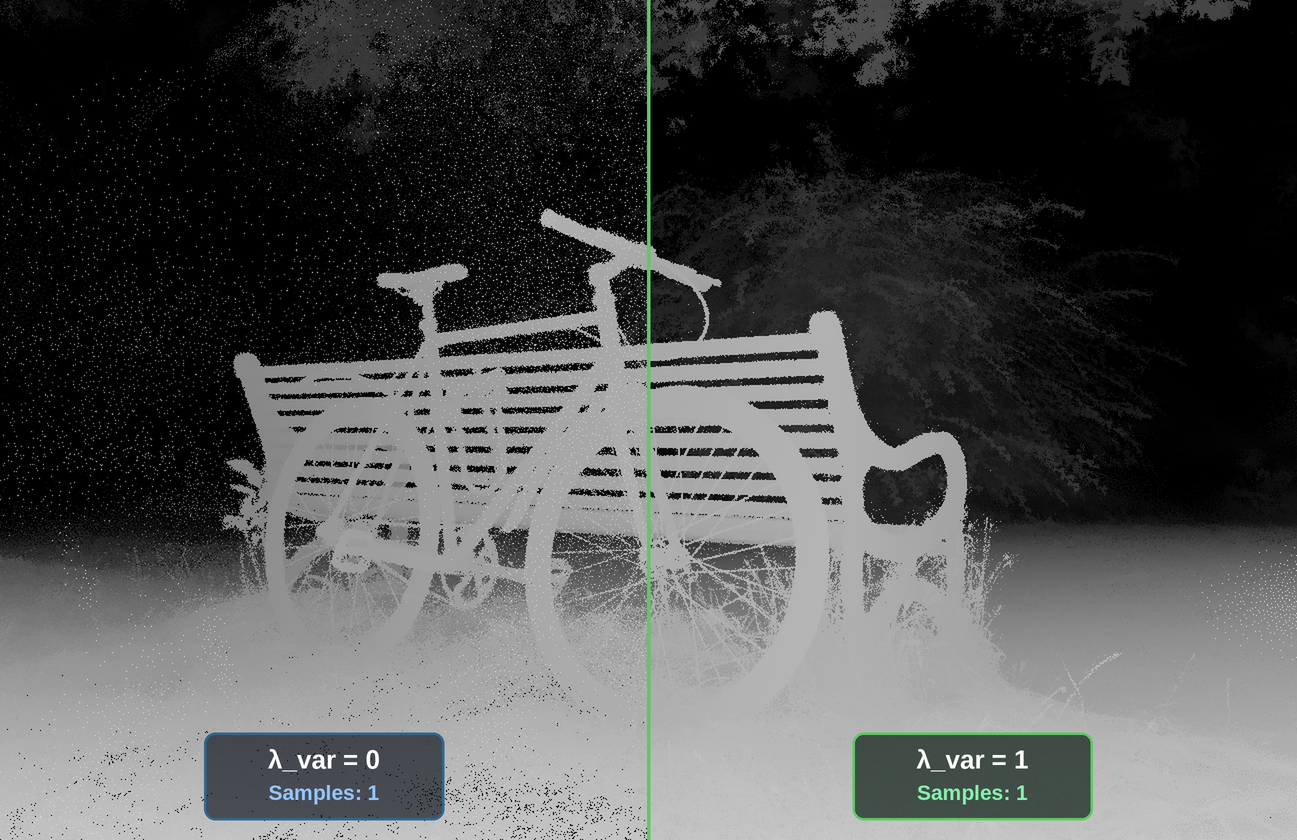}
        \caption{Increasing $\lvar$ also reduces variance in stochastically rendered depth. Images were rendered at 1 sample per pixel.}
        \label{fig:depth}
    \end{figure}
\end{itemize}

\begin{figure}[h]
    \centering
    \textbf{Bicycle}\par\medskip
    \begin{subfigure}[b]{0.48\linewidth}
        \centering
        \includegraphics[trim=32cm 2cm 32cm 10cm, clip, width=\textwidth]{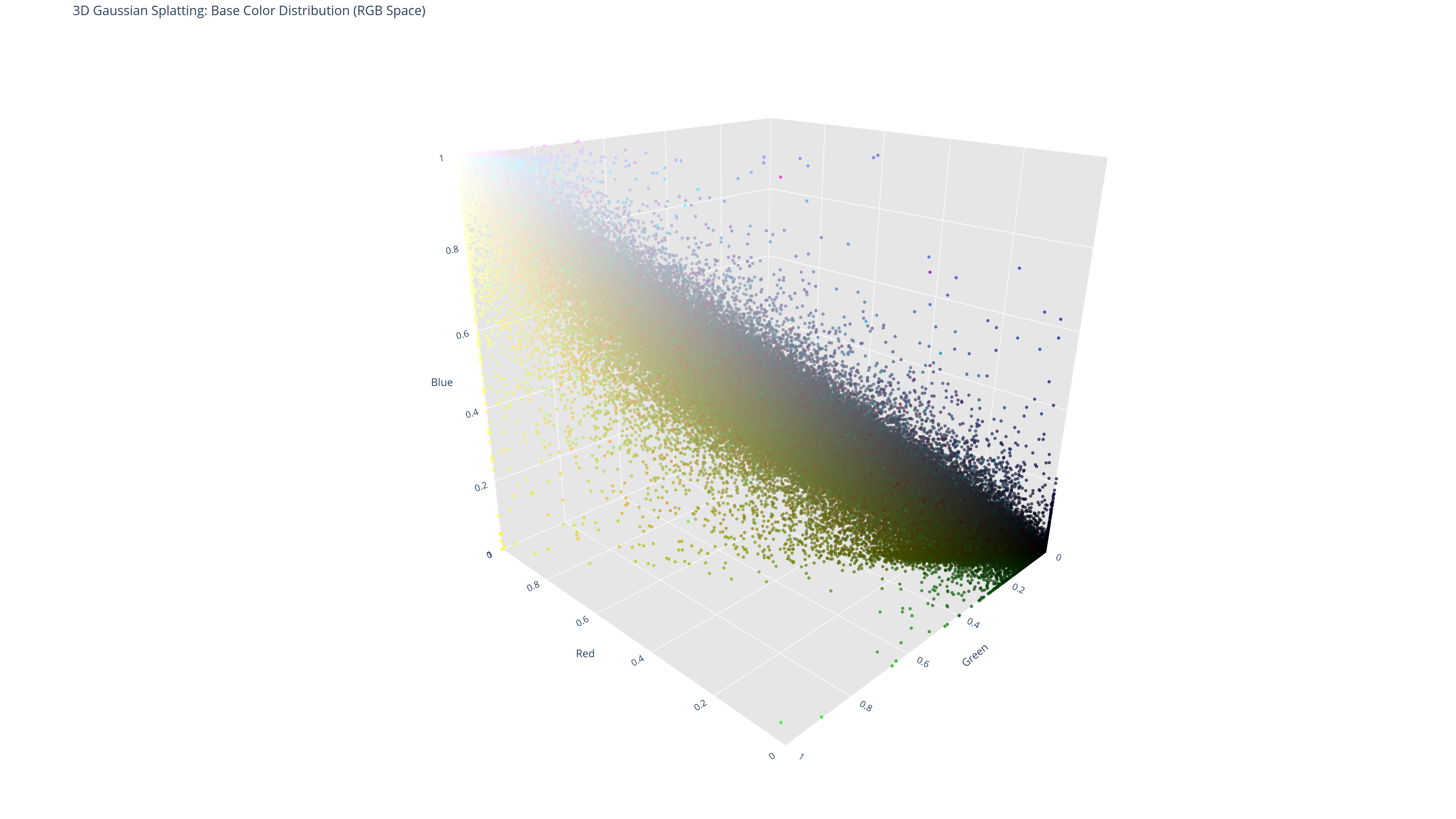}
        \caption{$\lvar=0$}
        \label{fig:color_bicycle_0}
    \end{subfigure}
    \hfill
    \begin{subfigure}[b]{0.48\linewidth}
        \centering
        \includegraphics[trim=32cm 2cm 32cm 10cm, clip, width=\textwidth]{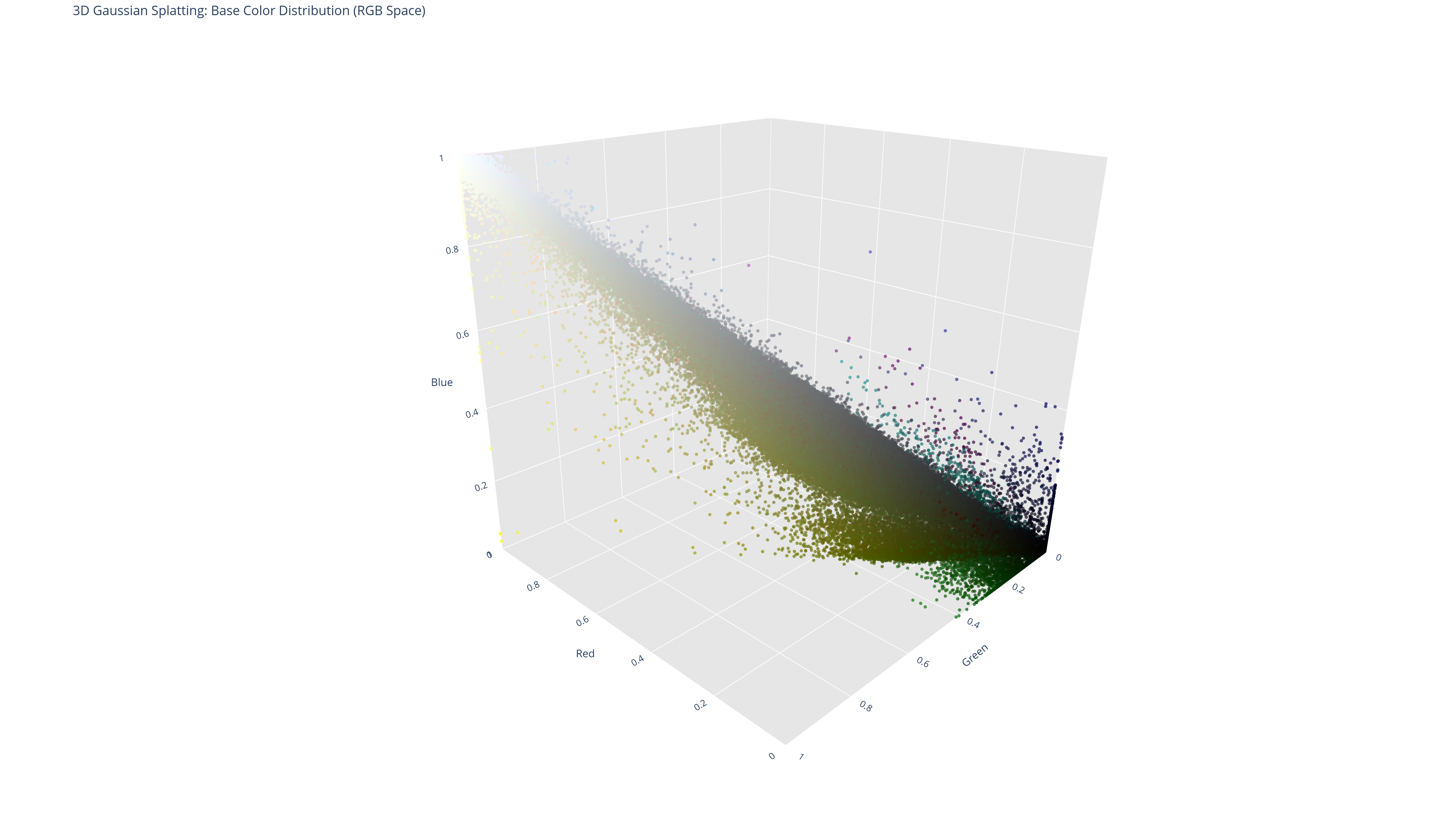}
        \caption{$\lvar=1$}
        \label{fig:color_bicycle_1}
    \end{subfigure}
    
    \vspace{1.5em}

    \textbf{Kitchen}\par\medskip
    \begin{subfigure}[b]{0.48\linewidth}
        \centering
        \includegraphics[trim=32cm 2cm 32cm 10cm, clip, width=\textwidth]{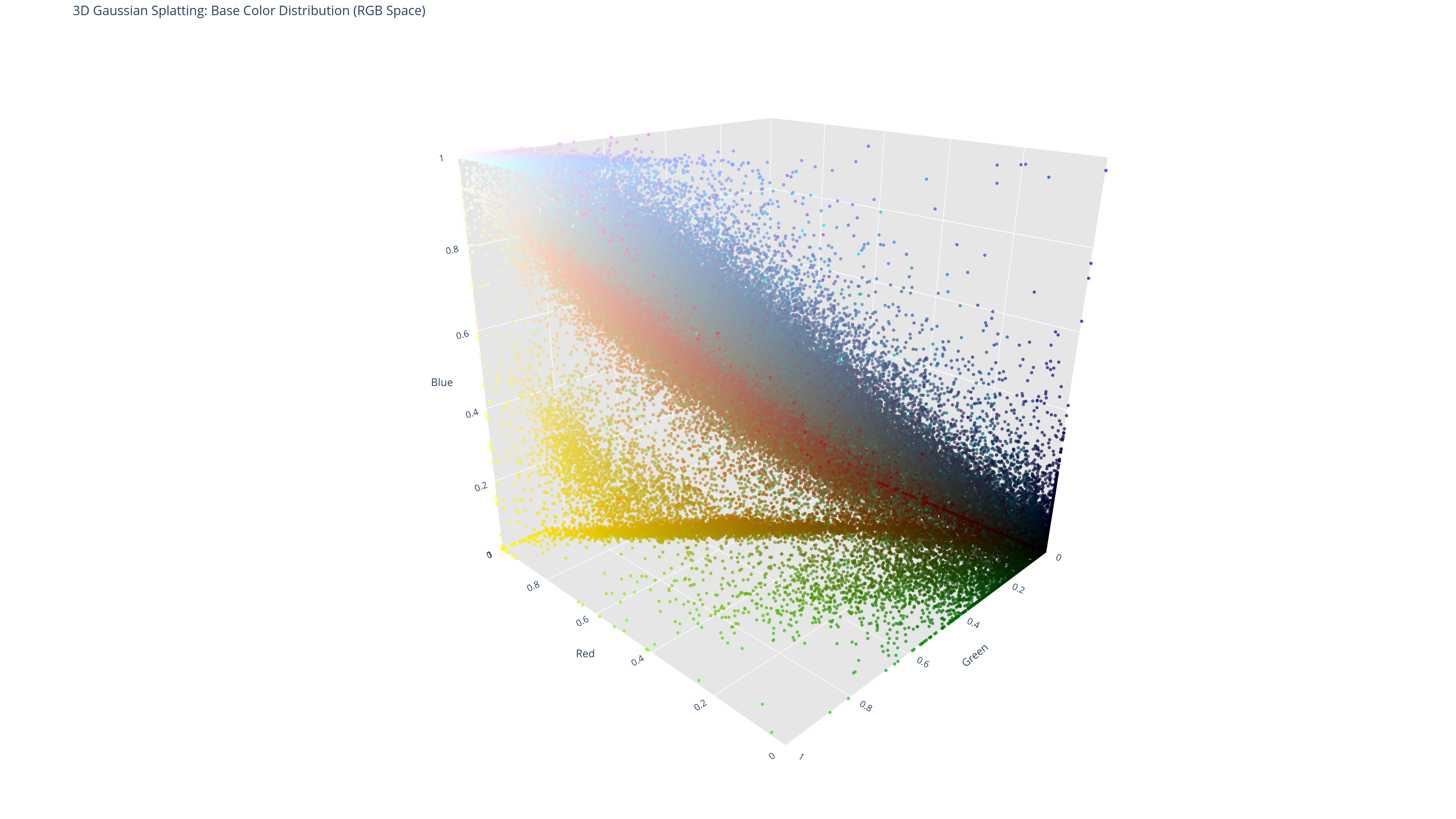}
        \caption{$\lvar=0$}
        \label{fig:color_kitchen_0}
    \end{subfigure}
    \hfill
    \begin{subfigure}[b]{0.48\linewidth}
        \centering
        \includegraphics[trim=32cm 2cm 32cm 10cm, clip, width=\textwidth]{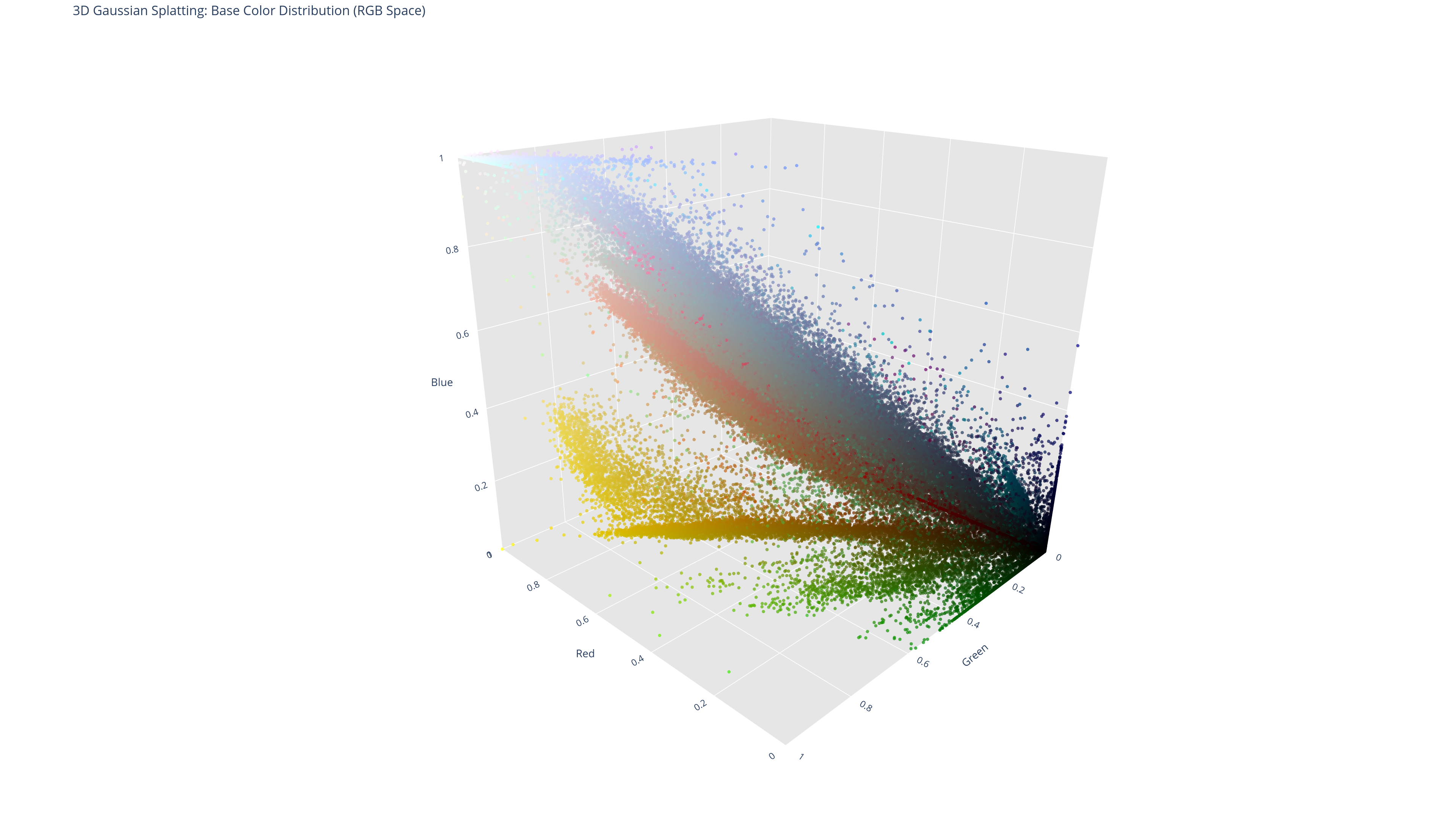}
        \caption{$\lvar=1$}
        \label{fig:color_kitchen_1}
    \end{subfigure}

    \vspace{0.5em}
    \caption{Distribution of primitive base colors in RGB space for scenes of Mip-NeRF 360 trained with and without variance loss.}
    \label{fig:colors}
\end{figure}

\begin{figure*}[t]
\centering
\begin{subfigure}[t]{0.3\textwidth}
    \centering
    \includegraphics[width=\linewidth]{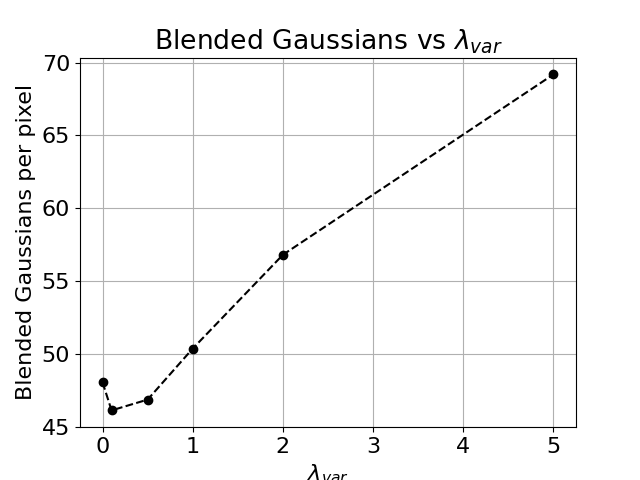}
    \caption{Number of Gaussians blended per pixel before early exit.}
\end{subfigure}
\hfill
\begin{subfigure}[t]{0.3\textwidth}
    \centering
    \includegraphics[width=\linewidth]{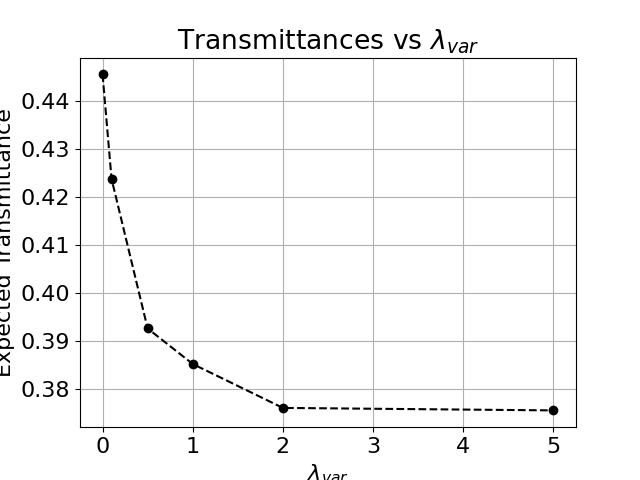}
    \caption{Expected transmittance of sampled Gaussian.}
    \label{fig:blended_sub}
\end{subfigure}
\hfill
\begin{subfigure}[t]{0.3\textwidth}
    \centering
    \includegraphics[width=\linewidth]{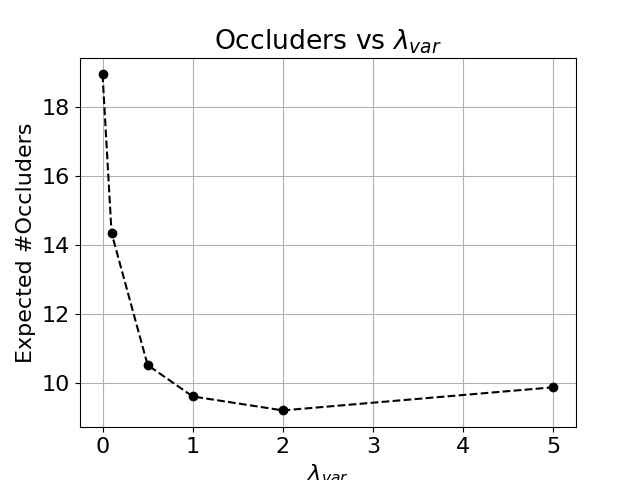}
    \caption{Expected number of occluding primitives for a sampled Gaussian.}
    \label{fig:occluders}
\end{subfigure}
\caption{Effect of $\Lvar$ on alpha-blending.}
\label{fig:blending}
\end{figure*}

\subsection{Additional Quality Metrics}
\Cref{fig:VarVSPSNR_continued,fig:ChoiceOfK_continued,fig:FreeLunch_continued} contain additional quality metrics corresponding to the evaluations from \cref{fig:VarVSPSNR,fig:ChoiceOfK,fig:FreeLunch}, respectively.

\begin{figure}[h]
    \centering
    \includegraphics[width=0.8\linewidth]{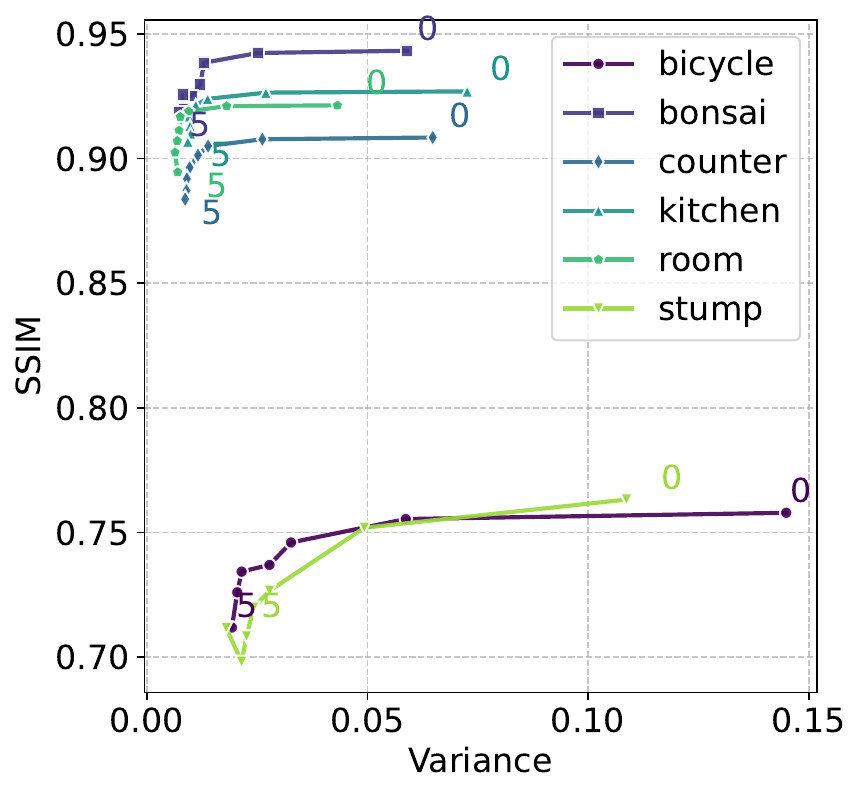}
    \includegraphics[width=0.8\linewidth]{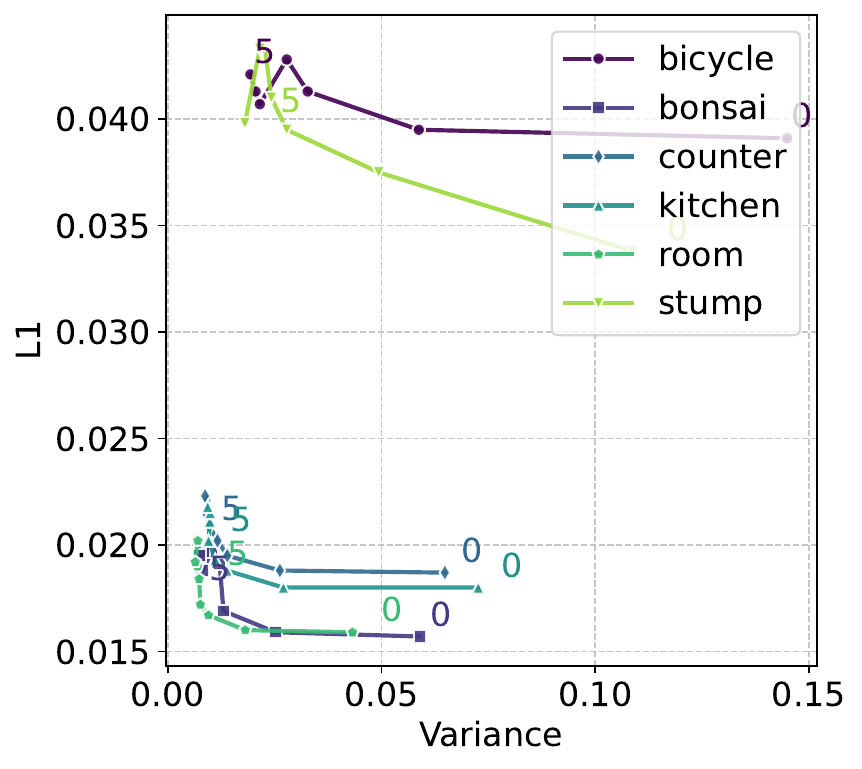}
    \caption{Comparison of quality metrics and variance for trained models with varying $\lvar$ on scenes from the Mip-NeRF 360 dataset.}
    \label{fig:VarVSPSNR_continued}
\end{figure}

\begin{figure}[h]
    \centering
    \includegraphics[width=0.99\linewidth]{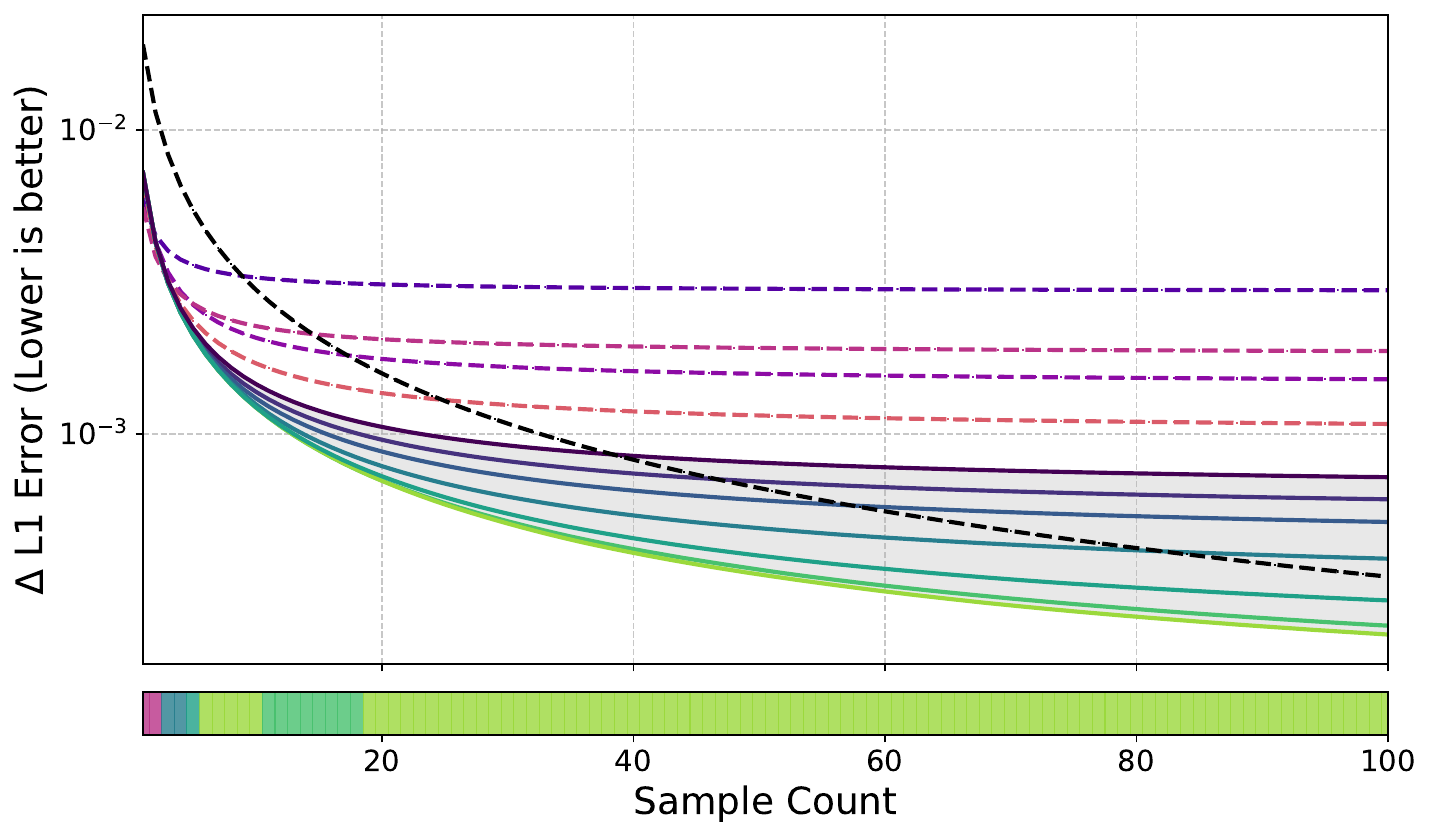}
    \includegraphics[width=0.99\linewidth]{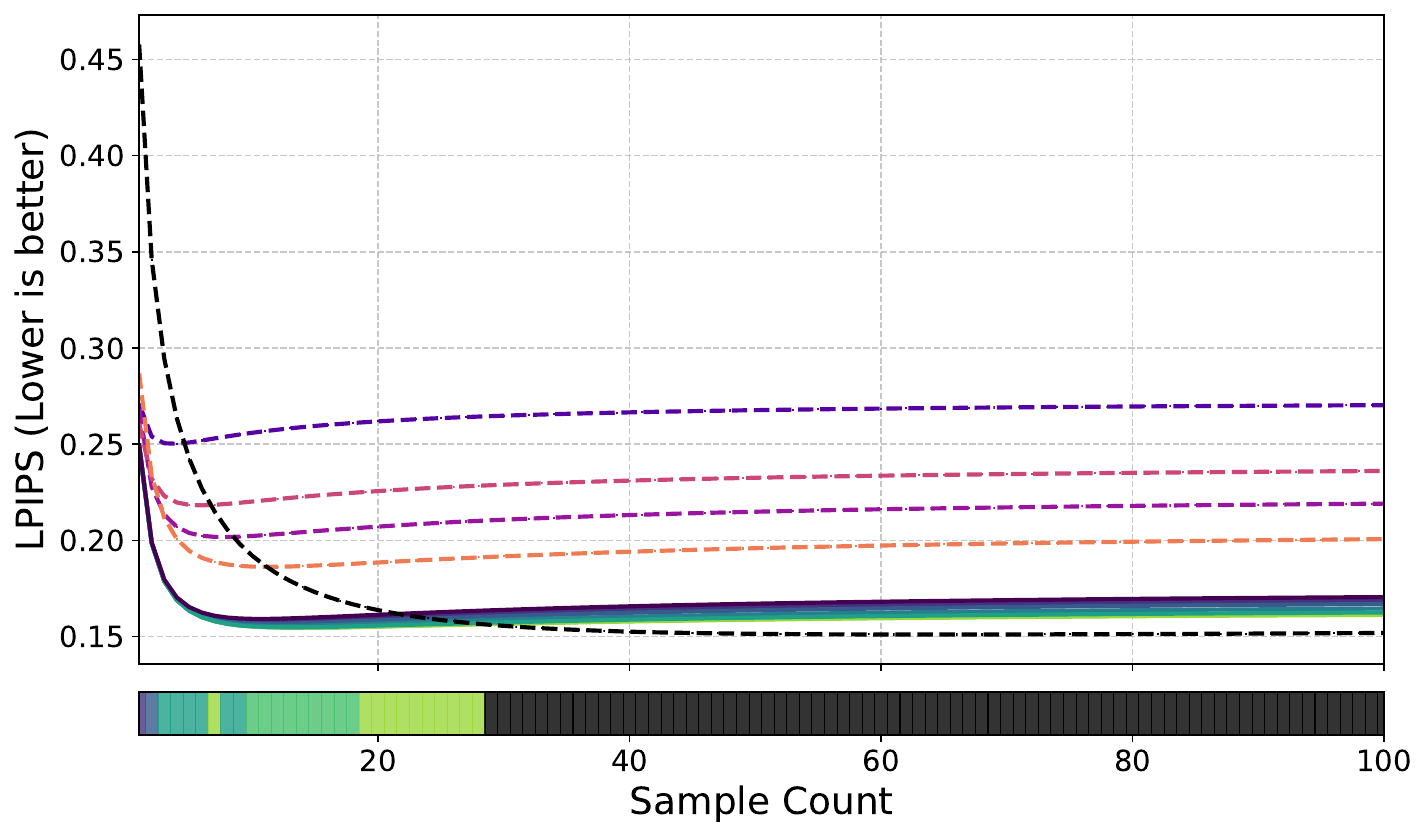}
    \includegraphics[width=0.99\linewidth]{Figures/Reuse_Quality/Legend.png}
    \includegraphics[trim={0 0 0 0.75cm},clip,width=0.99\linewidth]{Figures/Reuse_Quality/averaged_metrics_plot_colormap.pdf}
    \caption{Average residual of L1 score and values for LPIPS on the Mip-NeRF 360 dataset trained with $\lvar=1$, comparing static blur kernels against spatial reuse weights with varying $K$. Log-scale is used for L1 to visualize small differences.}
    \label{fig:ChoiceOfK_continued}
\end{figure}

\begin{figure}[h]
    \centering
    \includegraphics[width=0.6\linewidth]{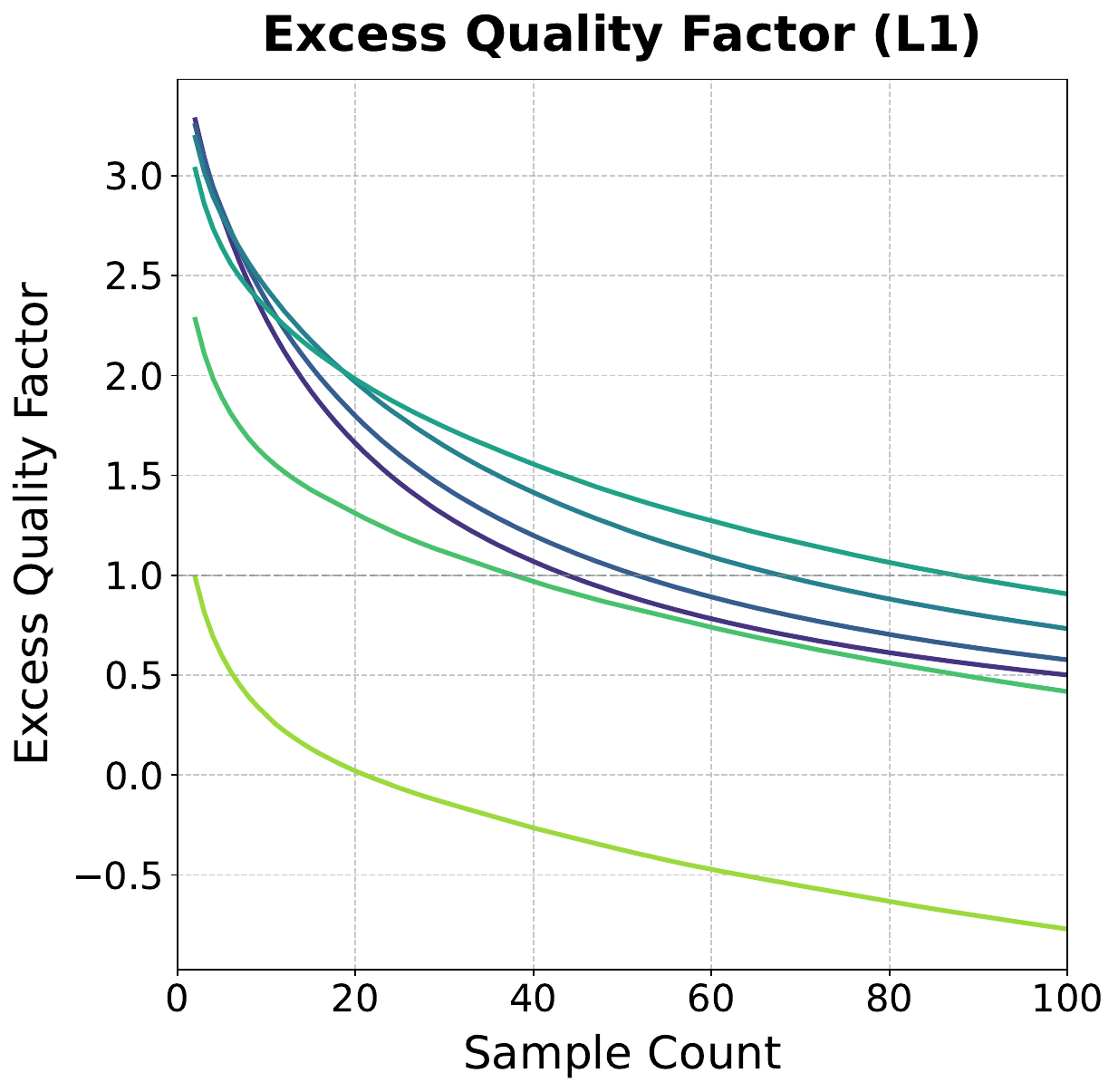}
    \includegraphics[width=0.6\linewidth]{Figures/Reuse_Quality/averaged_metrics_plot_colormap.pdf}
    \caption{Excess Quality Factor on the test views of the Mip-NeRF 360 dataset for the L1 quality metric.}
    \label{fig:FreeLunch_continued}
\end{figure}

\subsection{Pareto Analysis}

\begin{figure*}[t]
    \centering
    \begin{subfigure}[b]{0.46\linewidth}
        \centering
        \includegraphics[width=0.99\linewidth]{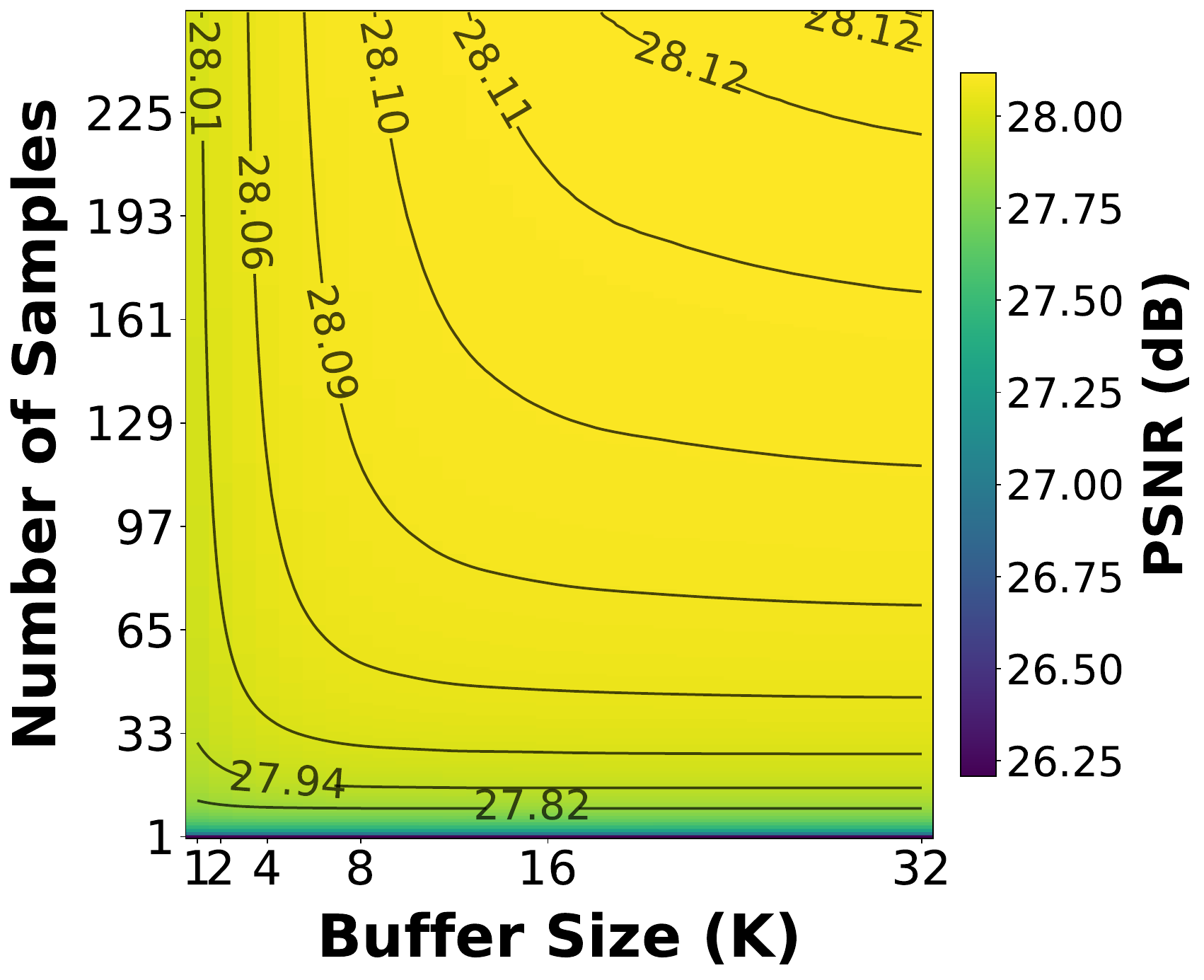}
        \includegraphics[width=0.99\linewidth]{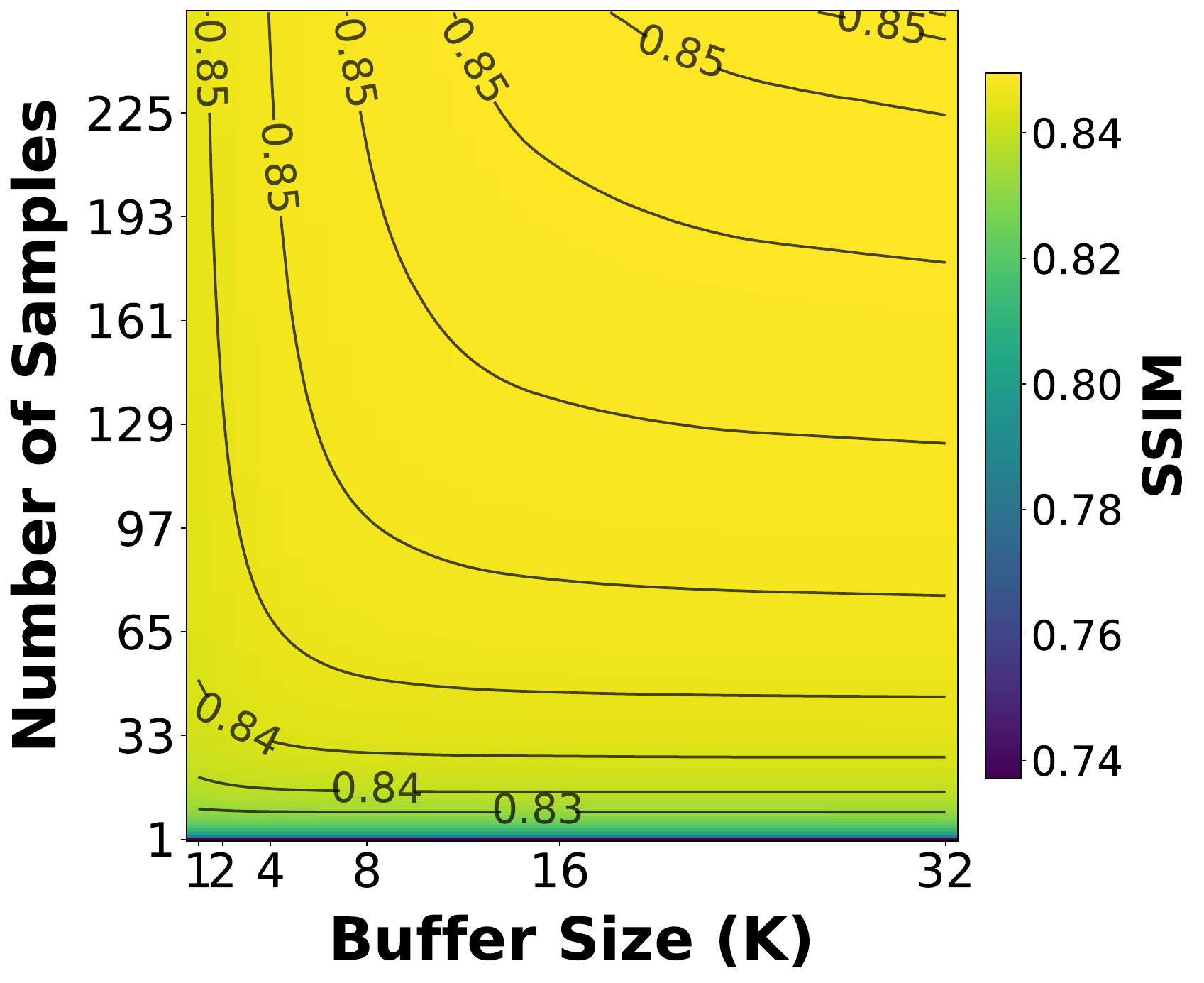}
        \caption{$|\mathcal{N}|=4$}
    \end{subfigure}
    \hfill
    \begin{subfigure}[b]{0.46\linewidth}
        \centering
        \includegraphics[width=0.99\linewidth]{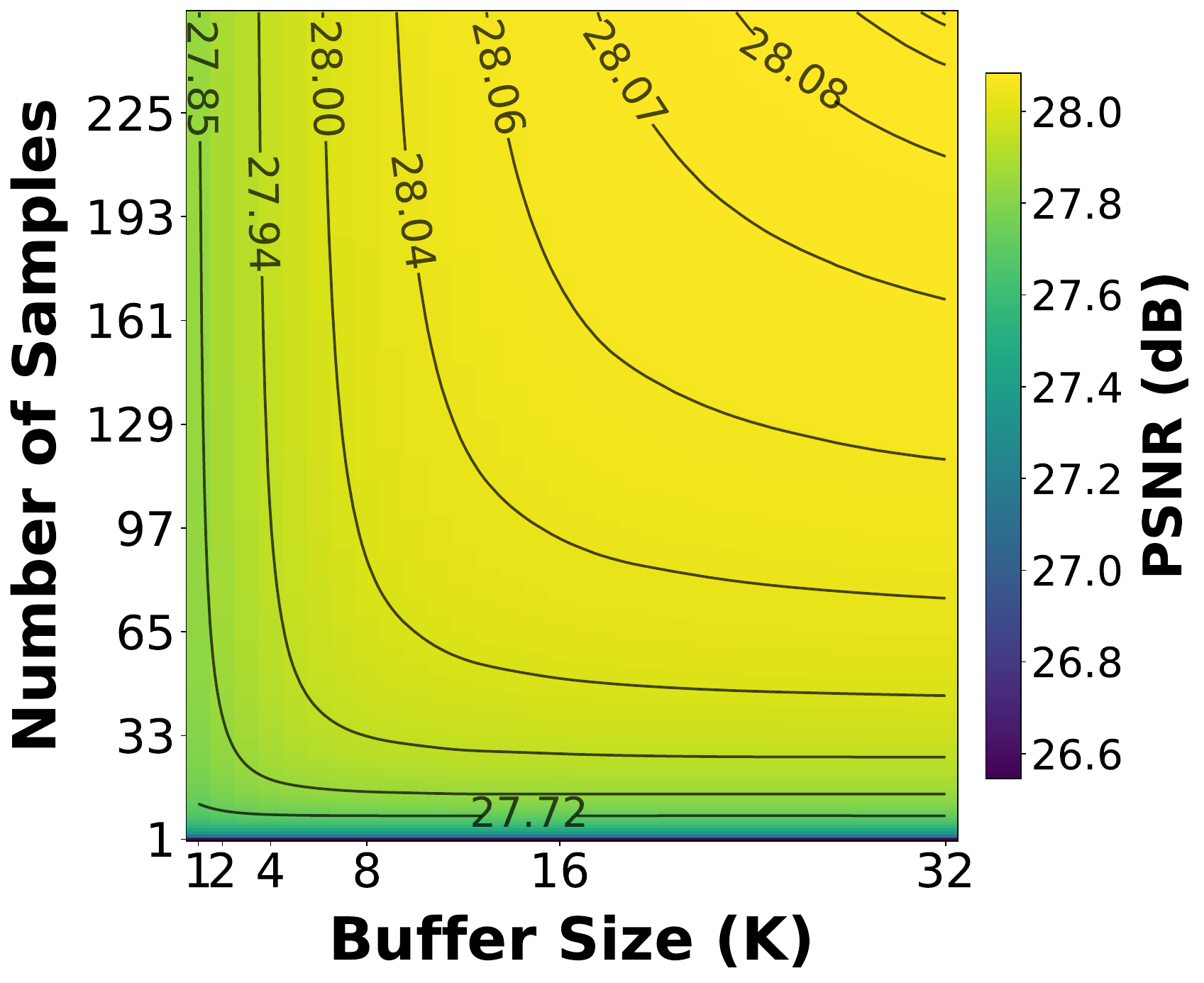}
        \includegraphics[width=0.99\linewidth]{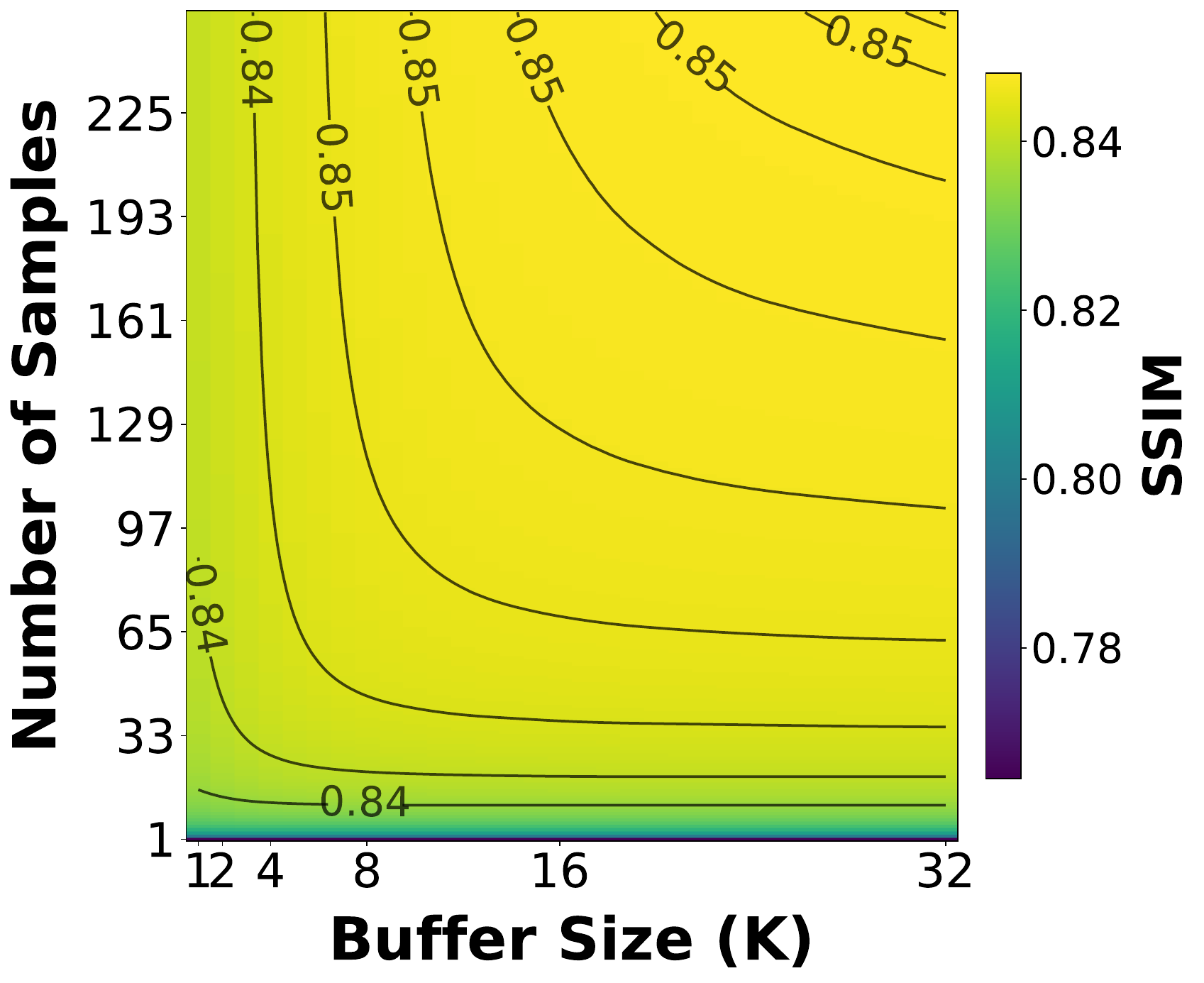}
        \caption{$|\mathcal{N}|=8$}
    \end{subfigure}
    \caption{Heatmaps of quality metrics over sample counts and buffer sizes on Mip-NeRF 360.}
    \label{fig:heatmaps}
\end{figure*}

The parameter $K$ in our spatial reuse component represents specific trade-offs between fast and exact convergence. Each combination of sample count and $K$ yields a specific image quality in a certain amount of sampling time. To evaluate the trade-off, we find the Pareto-optimal combinations of $K$ and sample count, i.e., those combinations that produce a specific image quality (measured across different metrics) in the shortest rendering time. \Cref{fig:pareto} compares the quality/time trade-off against standard \SSplats (our Vulkan reimplementation) with and without a blur kernel. \SSplats has the fastest rendering time, but is immediately overtaken in quality by the noise reduction enabled by using a Gaussian blur. Initially, the blur kernel wins out due to its low overhead, but all quality metrics agree that after 5 samples its quality is superseded by spatial reuse with a small buffer. At this point, the more accurate convergence of the spatial reuse weights overcomes the small performance overhead to achieve optimal quality. With more time to render, increasingly large buffer sizes become optimal, which provide better convergence at higher performance cost, until the unbiasedness of pure stochastic rendering catches up in quality, which happens for L1 at 136 samples and for LPIPS at 33. In PSNR and SSIM, a large reuse buffer remains dominant for 256+ samples. In summary, the optimal value for buffer size scales with the available amount of time for the image to converge. For practical rendering, a buffer size of $K\in[4,8]$ is a reasonable choice, representing the optimal configuration between $\sim 5$--$50$ samples.

The heatmaps in \cref{fig:heatmaps} show a more detailed view of the relationship between buffer size, sample count, and quality: as the number of samples increases, the isolines become increasingly curved, meaning that small sample buffers require rapidly increasing amounts of samples to match the quality of larger buffer sizes. For instance, with $|\mathcal{N}|=4$, 256 samples are required with $K=1$ to achieve the same PSNR quality as just 31 samples with $K=16$. We also find that except for very low sample counts, a smaller kernel size achieves better quality overall, validating our design decision to choose $|\mathcal{N}|=4$.

\begin{figure}[h]
    \centering
    \includegraphics[width=0.49\linewidth]{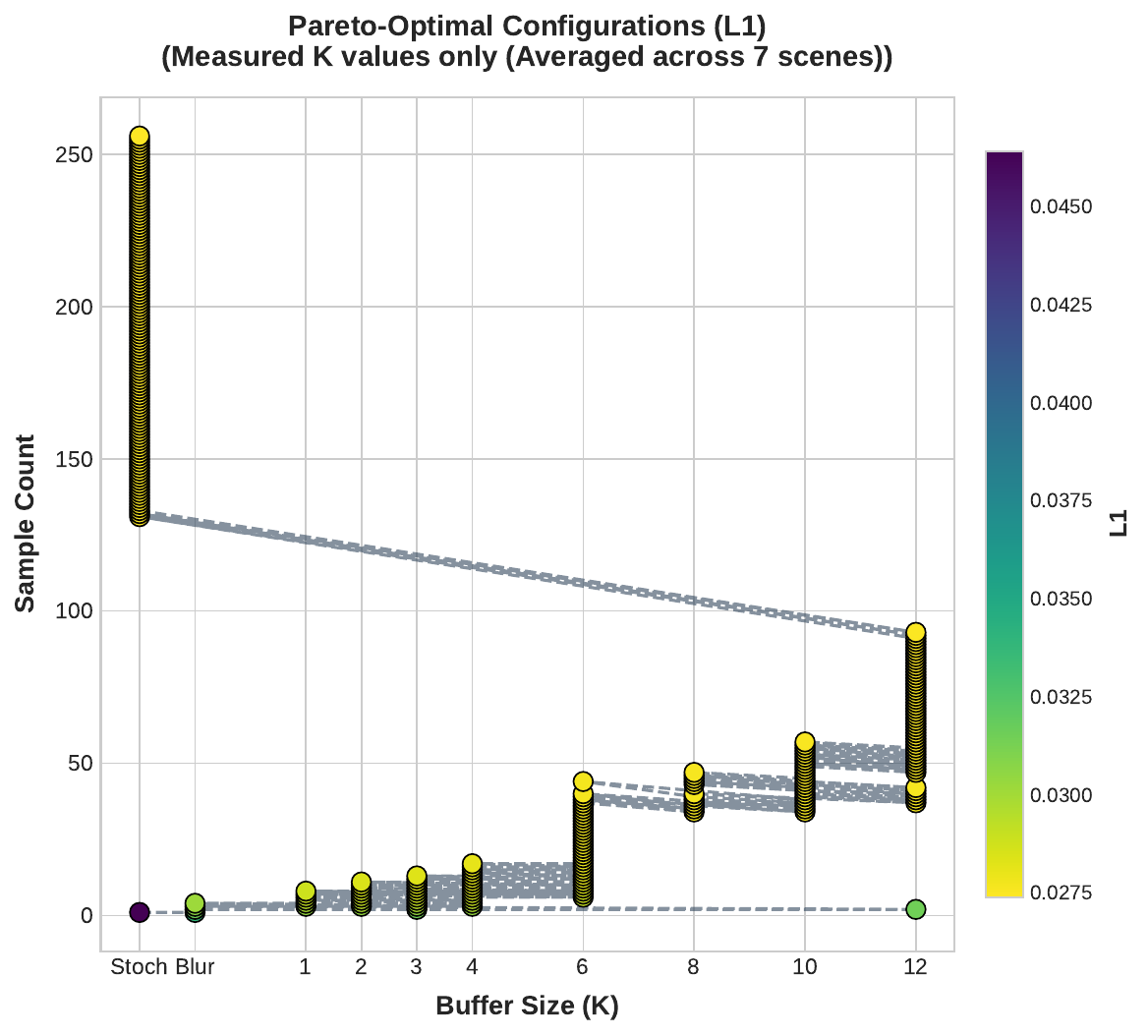}
    \includegraphics[width=0.49\linewidth]{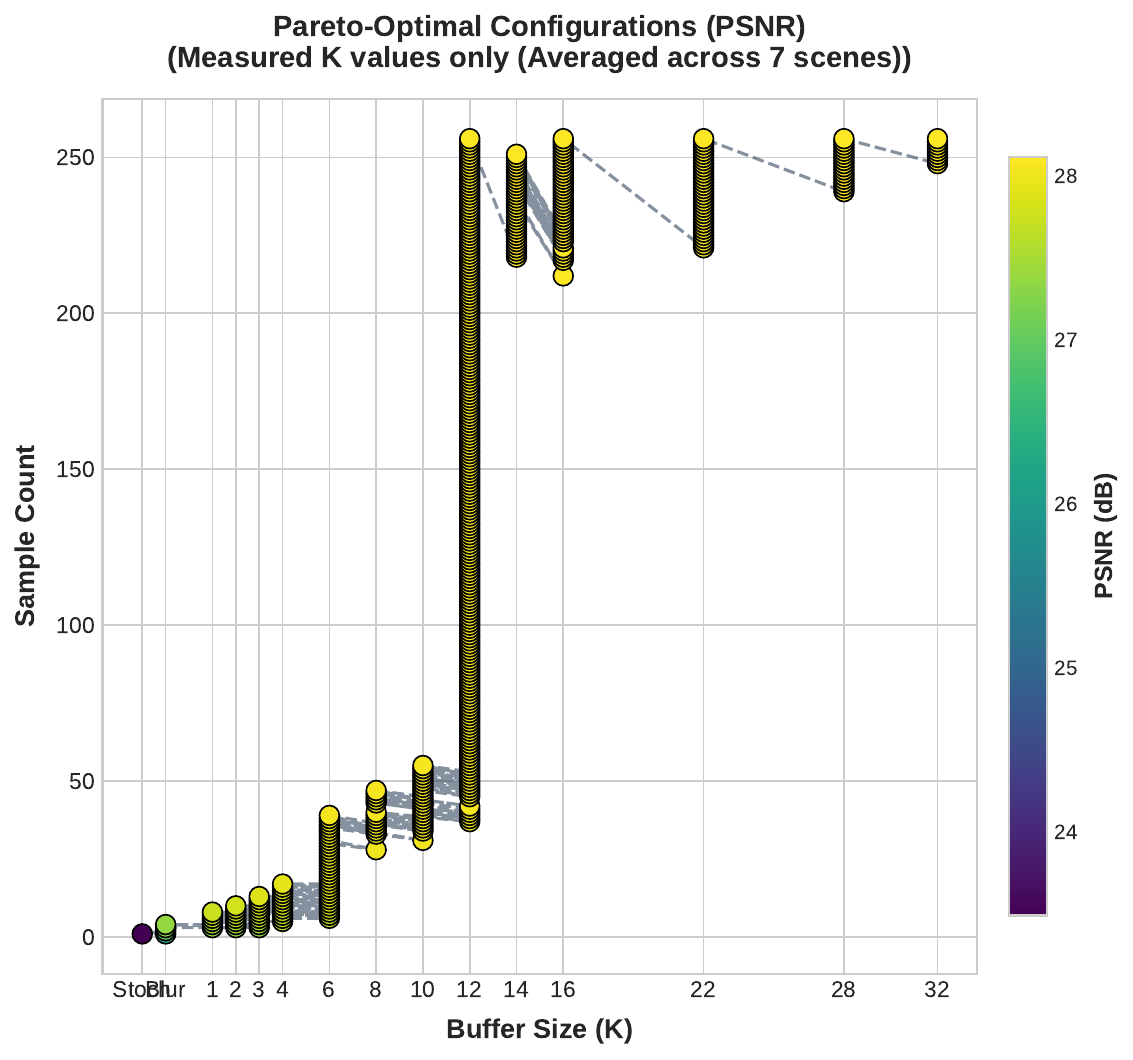}
    \includegraphics[width=0.49\linewidth]{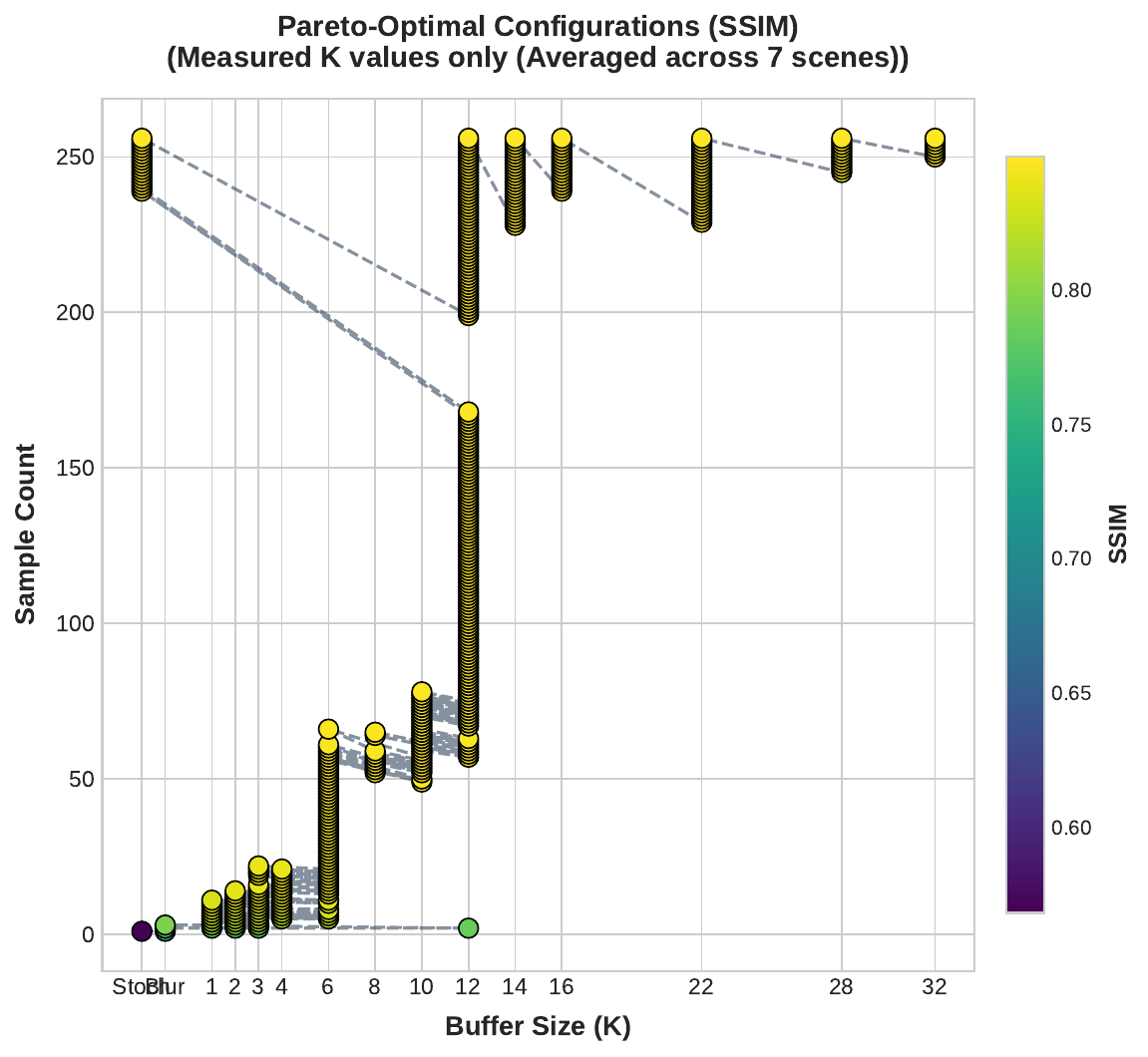}
    \includegraphics[width=0.49\linewidth]{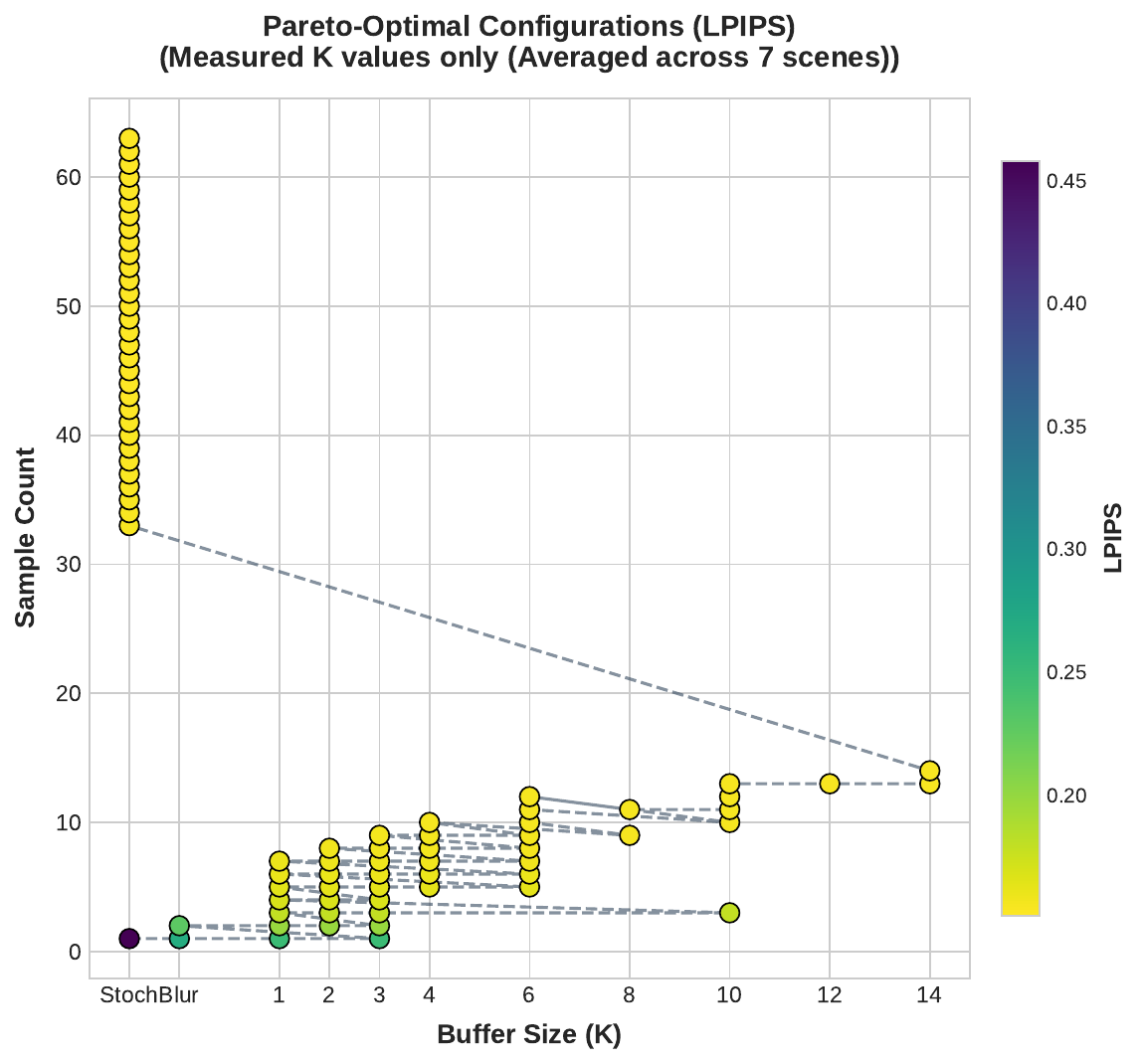}
    \caption{Pareto-optimal combinations of sample count and $K$ for L1, PSNR, SSIM, and LPIPS. \textit{Stoch} denotes stochastic rendering without spatial reuse and \textit{blur} denotes stochastic rendering with a Gaussian cross-blur kernel.}
    \label{fig:pareto}
\end{figure}

\subsection{8-Neighborhood Resampling}
\label{sec:eight}
Spatial resampling weights may be used on neighborhoods of arbitrary size. In \cref{fig:K8}, we compare quality metrics between using a cross kernel ($|\mathcal{N}| = 4$) and a box kernel ($|\mathcal{N}| = 8$) for spatial reuse. While the larger kernel achieves faster convergence for the first 10 samples because the larger sample pool is more effective at reducing variance, it also introduces additional bias, which leads to suboptimal convergence. As larger kernels are also associated with increased performance overhead, we find increasing the kernel size beyond a four-neighborhood to be an overall poor trade-off.

\begin{figure}[h]
    \centering
    \includegraphics[width=0.99\linewidth]{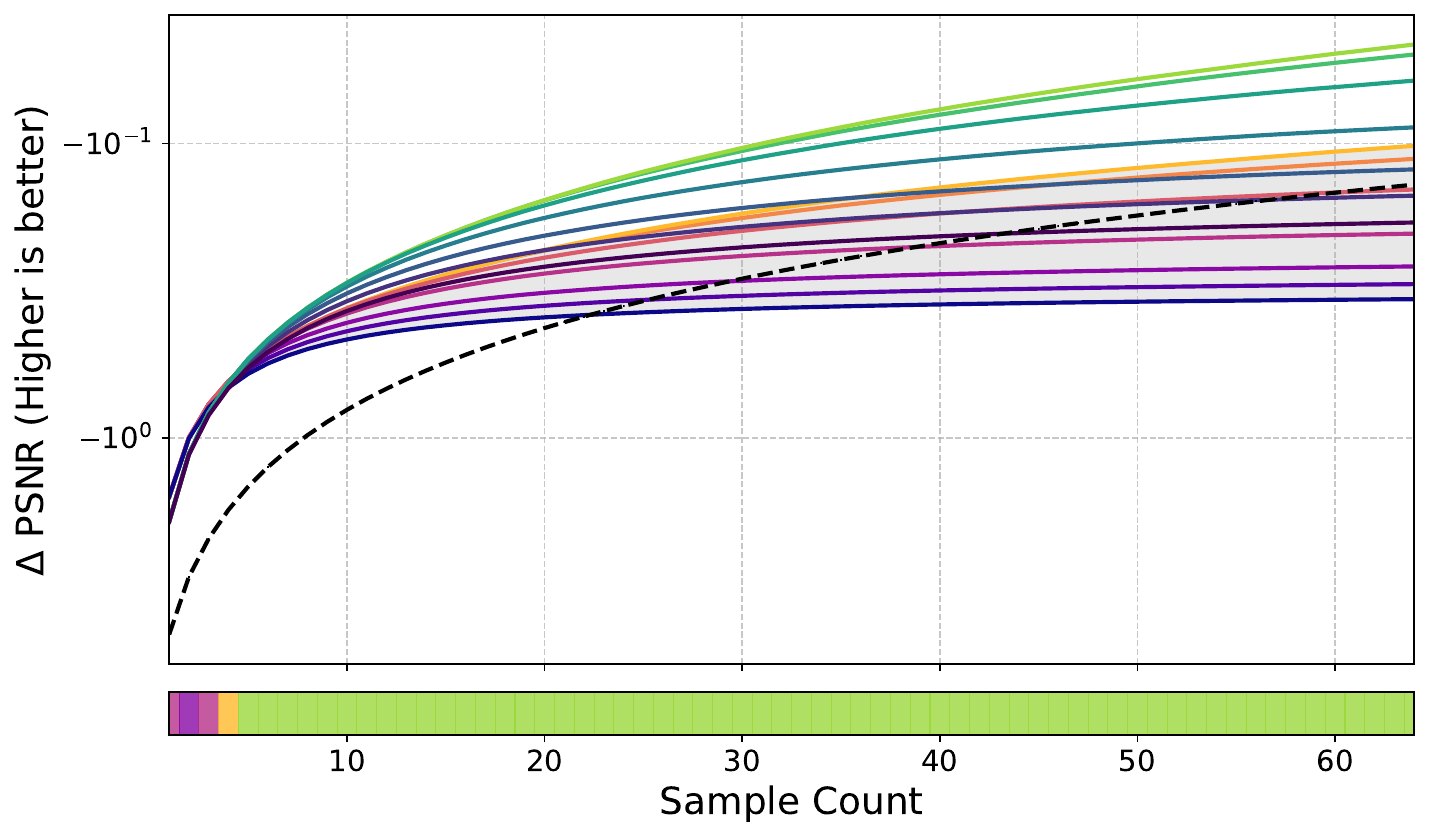}
    \includegraphics[width=0.99\linewidth]{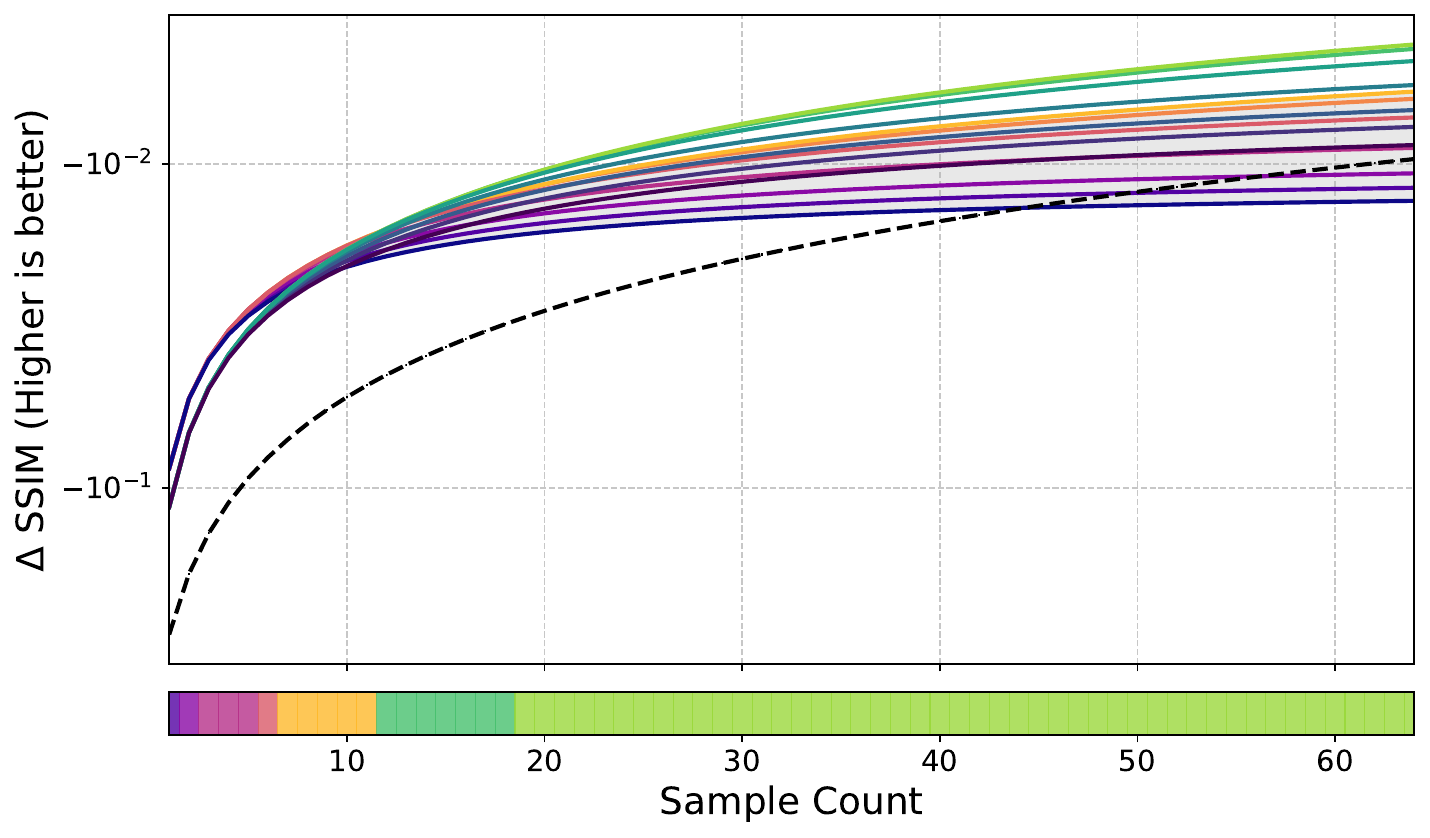}
    \includegraphics[width=0.5\linewidth]{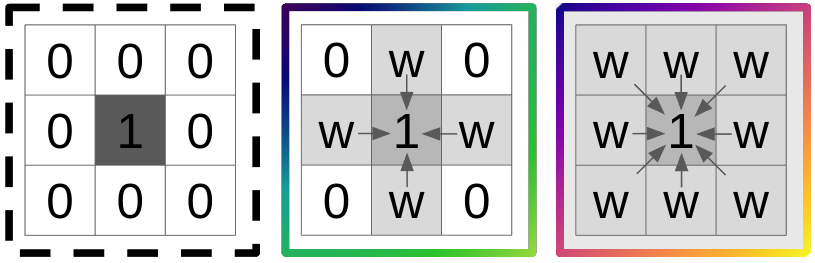}
    \includegraphics[width=0.99\linewidth]{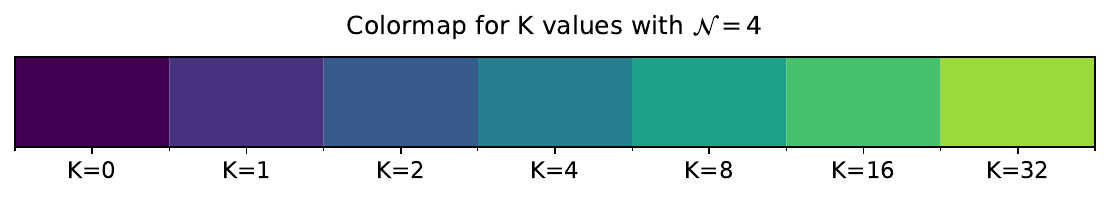}
    \includegraphics[width=0.99\linewidth]{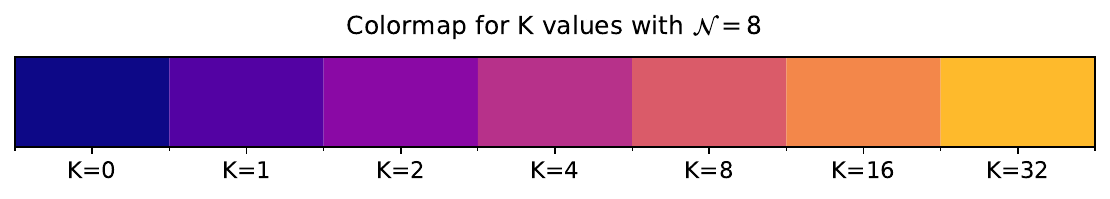}
    \caption{Average residual of test quality metrics on the Mip-NeRF 360 dataset trained with $\lvar=1$, comparing spatial reuse of 4- and 8-neighborhood kernels with varying $K$. Log-scale is used to visualize small differences.}
    \label{fig:K8}
\end{figure}

\section{Bias}
\label{app:bias}
In this section, we examine the sources of bias that stochastic sampling and spatial resampling introduce in comparison to standard 3DGS.

\paragraph{Self-Normalized Importance Sampling.}
Using $\frac{1}{\sum_{n \in \mathcal{N}(x,y)} \wgt{n}{x,y} (r_n)}$ as a normalization factor introduces bias for finite samples, because the denominator is a random variable. This bias decreases as $O(1/N)$ with the number of samples, so the estimator is consistent. As the evaluation demonstrates, the achieved reduction in variance far outweighs the impact of the introduced bias.

\paragraph{Tile-Based Splatting.}
\label{par:tile-based}
In \cref{eq:coefficient_equality}, we make the assumption that $\mathcal{G}(x,y) = \mathcal{G}(n)$ for all $n\in \mathcal{N}(x,y)$. In tile-based rasterization, this holds as long as $n$ is in the same $16\times16$ tile as $(x,y)$. Otherwise, it is possible for a Gaussian to be sampled on a border pixel and then propagated to a neighboring tile, which would not use this Gaussian in deterministic 3DGS rendering. However, this also implies that the maximal $\alpha$ value of that Gaussian in the neighboring tile is less than $\frac{1}{256}$, so that its sampling probability there is $P_{x,y}(r_n) \le \alpha_{r_n}(x,y) < \frac{1}{256}$, which bounds its expected contribution. Because these cases are rare and are guaranteed to introduce only negligible bias, it is not worth the overhead to explicitly check for them.

\paragraph{Early Exit.}
In standard 3DGS, the blending procedure terminates once transmittance falls below a certain threshold (typically $0.0001$) and all subsequent Gaussians are skipped. While this is not the case in \SSplats or \GPS (GPS), the maximal difference caused is strictly bounded by the threshold value and is thus negligible.

\paragraph{Finite History Buffer.}
\label{par:finite-buffer}
\Cref{eq:weight} requires the occlusion product over \emph{all} Gaussians in front of the sample $r_n$ at the target pixel, whereas \cref{alg:KBuffer} evaluates it only over the at most $K$ deduplicated Gaussians stored in the buffer $\mathcal{K}_n$ of the source pixel. An occluder that has not been sampled at the pixel within the buffer's lifetime is missing from the product; omitting its factor $\occ{k}{n}{x,y}$ amounts to assuming that its alpha value $\alpha_k$ is identical at $(x,y)$ and at $n$. The resulting weight error is therefore largest for neighbor samples whose occluders differ strongly between the two pixels, i.e., at depth discontinuities and along the silhouettes of thin structures, and it decreases with $K$ (\cref{fig:ChoiceOfK}); as $K \rightarrow \infty$ the stored set contains every occluder with probability approaching one and the weights become exact (\cref{sec:spatial}). Because the buffer is deduplicated, frequently sampled foreground primitives do not evict rare occluders, so the missing factors are predominantly those of low-opacity Gaussians with $\occ{k}{n}{x,y} \approx 1$.

\section{Gaussian Point Splatting Evaluation} 
\label{GPS}
The stochastic rendering approach in \GPS \cite{rijsdijk2026gaussian} replaces traditional alpha-blended rasterization with a massively parallel stochastic point-sampling technique. Instead of sorting and sequentially evaluating semi-transparent Gaussians, the renderer samples opaque, pixel-sized points and splats them directly into a 64-bit atomic depth buffer. Because threads sample points independently, multiple points from the same Gaussian can collide on the same pixel, which would normally reduce the effective opacity. To counteract this, the method treats the sampling as a Poisson point process and dynamically sets the number of points to sample, $n_g$, using a Poisson distribution with a rate parameter $\Lambda$. The authors derive a corrected probability density function, $f(q)=-\frac{1}{\Lambda A}\ln(1-\alpha(q))$, which concentrates more points toward the center of the Gaussian. This mathematical correction ensures that the probability of splatting exactly zero points into a subpixel of area $A$ is precisely $1-\alpha(q)$, perfectly preserving the original Gaussian's intended opacity while achieving high scalability.

The rendering procedure of \GPS produces stochastic OIT samples almost identical to \SSplats and achieves very similar base quality. As such, all of our components can be used with GPS as a ``sampling backbone'', while also benefiting from its scalability. \Cref{fig:GPS_Eval} shows a comparison between \SSplats, GPS, and our spatial reuse applied to GPS samples. Just as with \SSplats, sample reuse leads to significantly faster convergence.

\begin{figure}[h]
    \centering
    \includegraphics[width=0.8\linewidth]{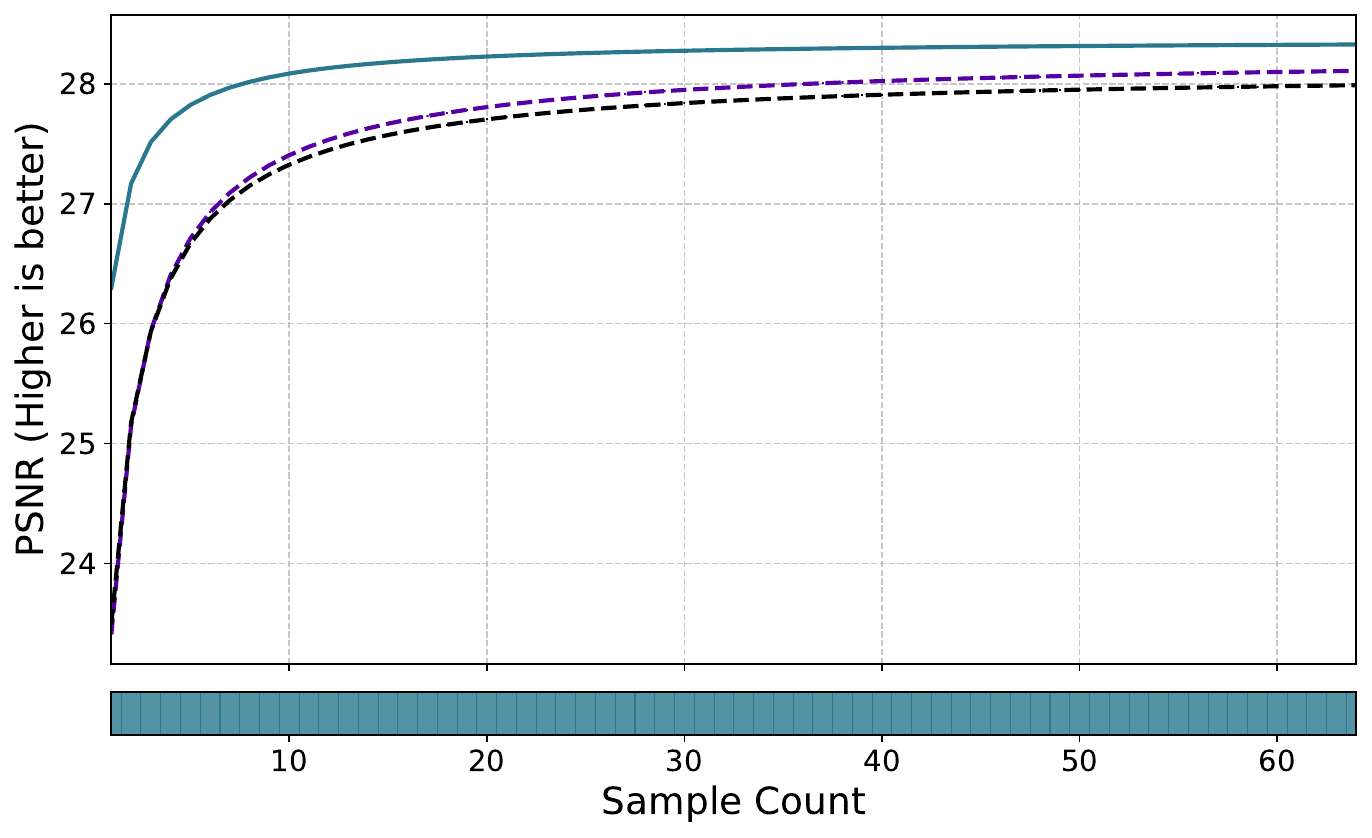}\\
    \includegraphics[width=0.8\linewidth]{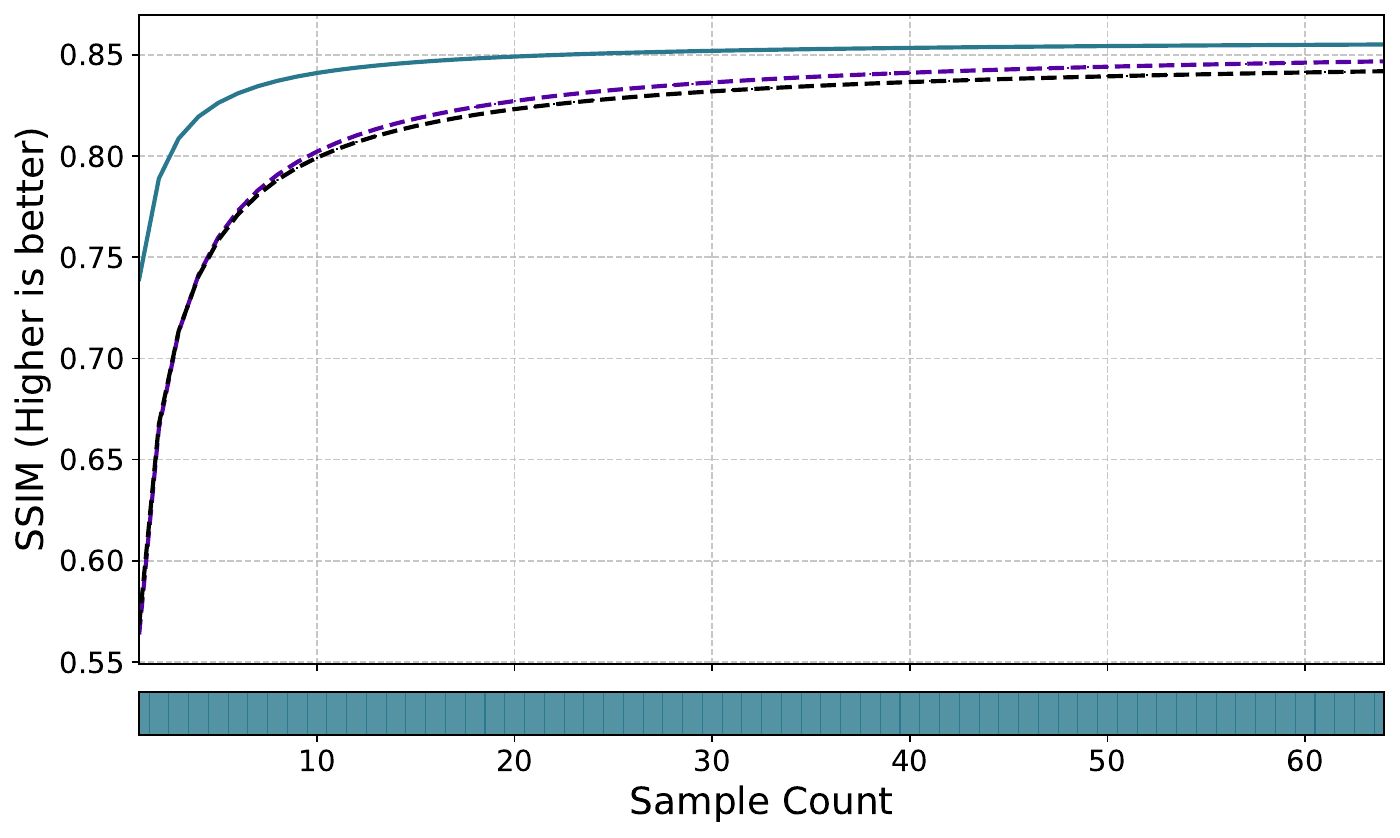}
    \includegraphics[width=0.75\linewidth]{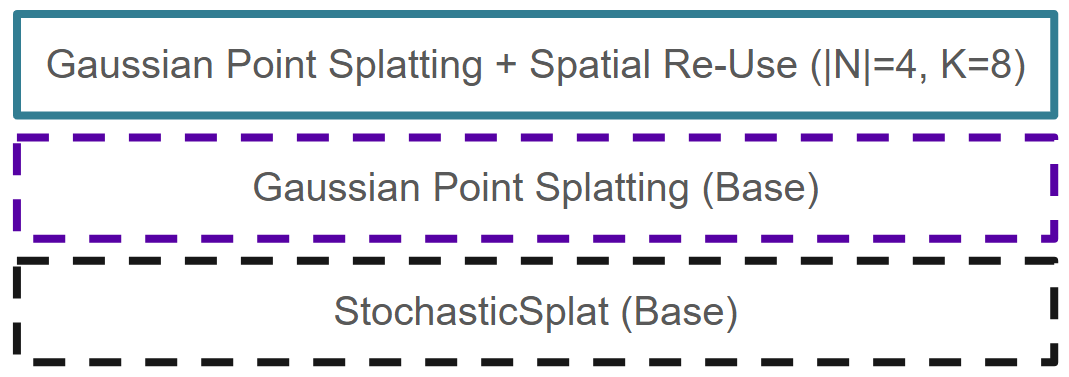}
    \caption{Average quality comparison of \SSplats and Gaussian Point Splatting with and without spatial reuse on the Mip-NeRF 360 dataset. Models were trained with $\lvar=1$.}
    \label{fig:GPS_Eval}
\end{figure}

\section{Spatio-Temporal Blue Noise}
\Cref{fig:uniform_vs_stbn} compares white-noise and spatio-temporal blue-noise sampling of the fragment thresholds $u_i$ (\cref{sec:spatiotemporal_resampling}).
\begin{figure}[h]
    \centering
    \begin{subfigure}[b]{0.4\textwidth}
        \centering
        \includegraphics[width=\textwidth]{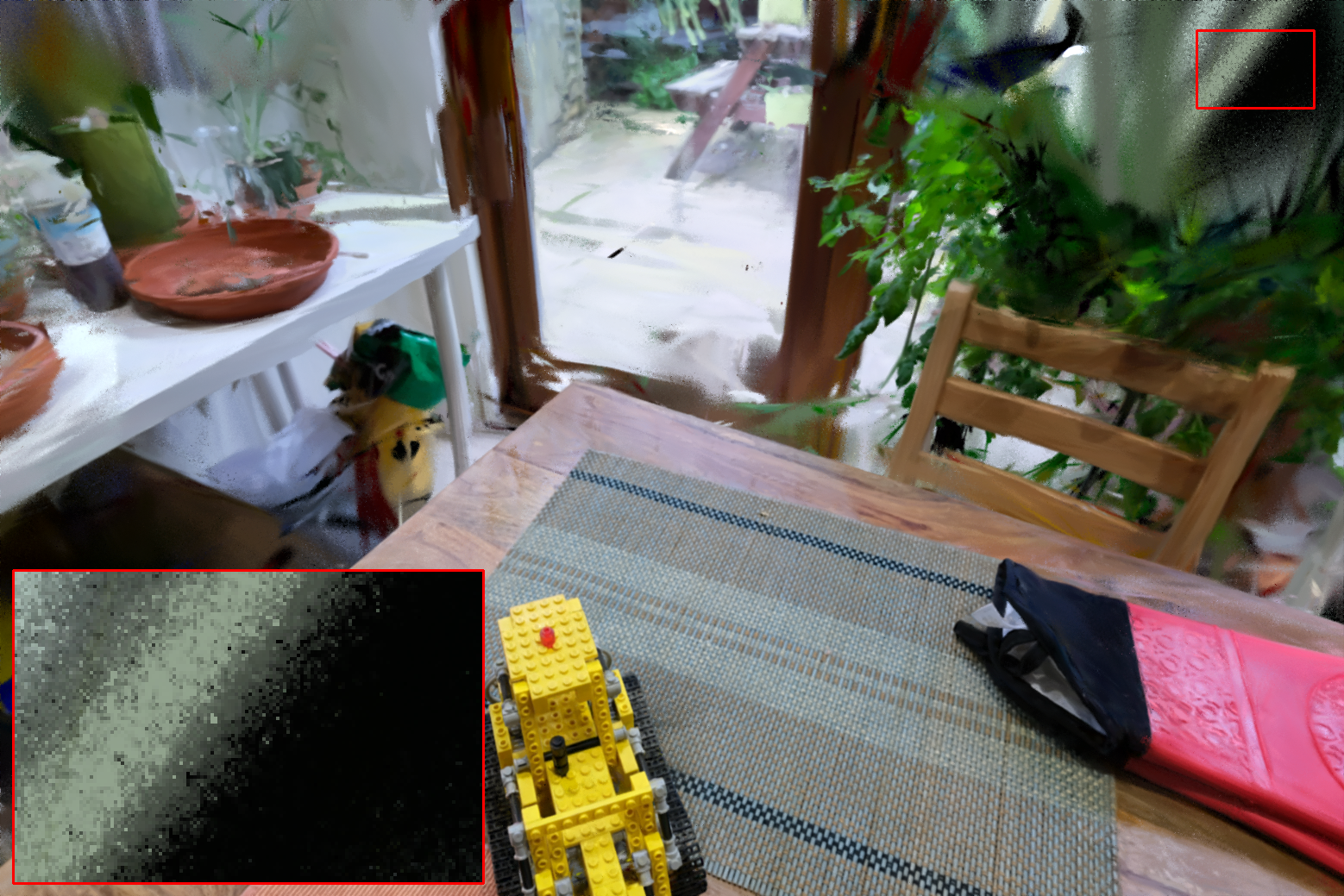}
        \caption{Uniform random noise (PCG32)}
        \label{fig:uniform_noise}
    \end{subfigure}
    \hfill
    \begin{subfigure}[b]{0.4\textwidth}
        \centering
        \includegraphics[width=\textwidth]{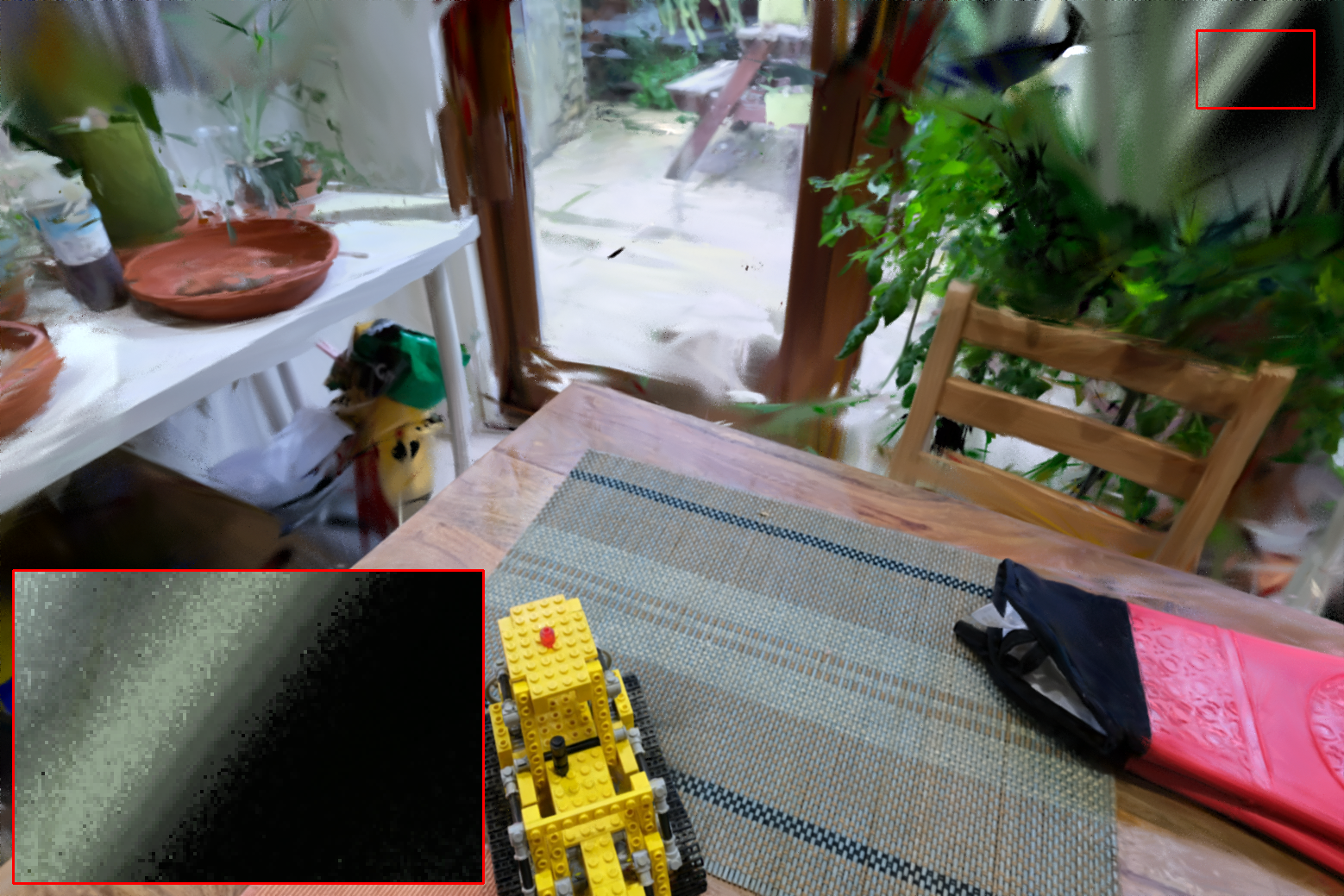}
        \caption{Spatio-temporal Blue Noise ($128\times 128\times 64$)}
        \label{fig:stbn_noise}
    \end{subfigure}
    \caption{Comparison of uniform random noise sampling (PCG32) against per-primitive spatio-temporal blue noise sampling.}
    \label{fig:uniform_vs_stbn}
\end{figure}

\section{Gradients of the Color Variance Loss}
\label{app:Math}

The color variance loss $\Lvar$ represents the average weighted RGB variance for each pixel:
\begin{equation}
    \Lvar = \frac{1}{H\cdot W}\sum_{x=1}^{W} \sum_{y=1}^{H} \sum_{i \in \mathcal{G}(x,y)} P_{x,y}(i) \cdot \lVert c_i - C^{\mathrm{gt}}_{x,y} \rVert^2 \, .
\end{equation}
We derive the gradients of the per-pixel term $\mathcal{L}_{x,y} = \sum_{i \in \mathcal{G}(x,y)} P_{x,y}(i) \, \lVert c_i - C^{\mathrm{gt}}_{x,y} \rVert^2$; the gradient of $\Lvar$ is their average over all pixels.
The gradients with respect to view-dependent color can be computed using the chain rule (similar to the derivation in~\cite{radl2026confidence}):
\begin{equation}
\frac{\partial \mathcal{L}_{x,y}}{\partial c_i} = 2 P_{x,y}(i) (c_i - C^{\mathrm{gt}}_{x,y}) \, .
\end{equation}

We can compute the gradient with regard to $\alpha_i(x,y)$ as follows, with Gaussians indexed in depth order as in deterministic training:
\begin{align}
    \frac{\partial \mathcal{L}_{x,y}}{\partial P_{x,y}(i)} &= \lVert c_i - C^{\mathrm{gt}}_{x,y} \rVert^2 \, , \\
    \frac{\partial P_{x,y}(i)}{\partial \alpha_i(x,y)} &= \prod_{j=1}^{i-1} (1 - \alpha_j(x,y)) = T_i(x,y) \, , \\
    \begin{split}
    \frac{\partial \mathcal{L}_{x,y}}{\partial \alpha_i(x,y)} &= \frac{\partial \mathcal{L}_{x,y}}{\partial P_{x,y}(i)} \cdot \frac{\partial P_{x,y}(i)}{\partial \alpha_i(x,y)} \\
    &= \lVert c_i - C^{\mathrm{gt}}_{x,y} \rVert^2 \, T_i(x,y) \, .
    \end{split}
\end{align}

This gradient does not account for the effect of $\alpha_i(x,y)$ on subsequent Gaussians $j > i$, which include $(1-\alpha_i(x,y))$ in their weight $P_{x,y}(j)$. Because $\alpha_i(x,y) < 1$ by definition, $\frac{\partial P_{x,y}(j)}{\partial \alpha_i(x,y)} = -\frac{P_{x,y}(j)}{1 - \alpha_i(x,y)}$ holds, giving, with $j$ ranging over all Gaussians in $\mathcal{G}(x,y)$ behind $i$:
\begin{equation}
\begin{split}
    &\sum_{j > i} \lVert c_j - C^{\mathrm{gt}}_{x,y} \rVert^2 \left(-\frac{P_{x,y}(j)}{1 - \alpha_i(x,y)}\right) \\
    &\quad = -\frac{1}{1 - \alpha_i(x,y)} \sum_{j > i} P_{x,y}(j) \, \lVert c_j - C^{\mathrm{gt}}_{x,y} \rVert^2 \, .
\end{split}
\end{equation}

The final gradient for $\alpha_i(x,y)$ is then:
\begin{equation}
\begin{split}
    \frac{\partial \mathcal{L}_{x,y}}{\partial \alpha_i(x,y)} = {} & \lVert c_i - C^{\mathrm{gt}}_{x,y} \rVert^2 \, T_i(x,y) \\
    & - \frac{1}{1 - \alpha_i(x,y)} \sum_{j > i} P_{x,y}(j) \, \lVert c_j - C^{\mathrm{gt}}_{x,y} \rVert^2 \, .
\end{split}
\end{equation}

\section{Spatio-Temporal Resampling}
\label{app:spatiotemporal}

\subsection{Proofs}
\label{sec:proofs}

Here we include the full proof for our main theoretical result.

\emaReduction*

\begin{proof}
For brevity, write $\mathcal{T}^t$ for the set of frames contributing at time $t$, so that $\mathcal{T}^t = \mathcal{T}^{t-1} \cup \{t\}$, write $c^{\tilde{t}}_n = c_{r^{\tilde{t}}_n}$, and omit the sample arguments $(r^{\tilde{t}}_n)$ of the weights. Define the discounted color mass:
\begin{equation*}
S^t_{x,y} = \sum_{n \in \mathcal{N}(x,y)} \sum_{\tilde{t} \in \mathcal{T}^t} \Big(\prod_{u=\tilde{t}+1}^{t} \delta^u_{x,y}\Big) \wgt{n}{x,y}^{\tilde{t}} c^{\tilde{t}}_{n} \, ,
\end{equation*}
so that, with $D^t_{x,y}$ as defined, \cref{eq:discount} yields $\hat{C}^t_{x,y} = S^t_{x,y} / D^t_{x,y}$.
Splitting off the term $\tilde{t} = t$ (whose product is empty and equals one) and factoring $\delta^t_{x,y}$ out of every remaining product yields the recursions:
\begin{align*}
S^t_{x,y} &= \delta^t_{x,y} S^{t-1}_{x,y} + \sum_{n \in \mathcal{N}(x,y)} \wgt{n}{x,y}^t c^t_n \, , \\
D^t_{x,y} &= \delta^t_{x,y} D^{t-1}_{x,y} + \sum_{n \in \mathcal{N}(x,y)} \wgt{n}{x,y}^t \, .
\end{align*}
Dividing the first equation by $D^t_{x,y}$ and using $S^{t-1}_{x,y} = D^{t-1}_{x,y} \hat{C}^{t-1}_{x,y}$:
\begin{align*}
\hat{C}^t_{x,y}
&= \frac{\delta^t_{x,y} D^{t-1}_{x,y}}{D^t_{x,y}} \hat{C}^{t-1}_{x,y} \\
&\quad + \frac{\sum_{n} \wgt{n}{x,y}^t}{D^t_{x,y}} \cdot \frac{\sum_{n} \wgt{n}{x,y}^t c^t_n}{\sum_{n} \wgt{n}{x,y}^t} \\
&= \frac{\delta^t_{x,y} D^{t-1}_{x,y}}{D^t_{x,y}} \hat{C}^{t-1}_{x,y} + \frac{\sum_{n} \wgt{n}{x,y}^t}{D^t_{x,y}} \Cspat[t]{x,y} \, ,
\end{align*}
where the last factor is \cref{eq:SpatialReuse} evaluated in frame $t$. The two coefficients are non-negative and, by the recursion for $D^t_{x,y}$, sum to $\big(\delta^t_{x,y} D^{t-1}_{x,y} + \sum_n \wgt{n}{x,y}^t\big) / D^t_{x,y} = 1$, proving the combination is convex. For a moving camera, the same recursion is applied along the motion trajectory (\cref{par:camera_motion}).
\end{proof}

\subsection{Additional Ablations}

\begin{table*}[h]
\centering
\small
\setlength{\tabcolsep}{6.5pt}
\caption{\textbf{Per-scene temporal consistency tPSNR (dB).} Evaluated on 7 Mip-NeRF 360 scenes (3 trajectories each).}
\label{tab:per_scene_factorial_tpsnr}
\begin{tabular}{lcccccccc}
\toprule
\textbf{Configuration} & \textbf{Bicycle} & \textbf{Bonsai} & \textbf{Counter} & \textbf{Garden} & \textbf{Kitchen} & \textbf{Room} & \textbf{Stump} & \textbf{Overall Mean} \\
\midrule
\multicolumn{9}{l}{\textit{Standard 3DGS} ($\lvar = 0$)} \\
\quad No Reuse             & 13.230 & 16.230 & 15.780 & 13.343 & 15.610 & 18.183 & 14.267 & 15.235 \\
\quad Spatial Reuse         & 19.140 & 23.077 & 22.643 & 19.080 & 22.377 & 25.083 & 20.640 & 21.720 \\
\quad Temporal Reuse        & 19.767 & 23.020 & 22.427 & 20.390 & 22.857 & 25.757 & 21.210 & 22.204 \\
\quad Spatio-Temporal Reuse & 23.757 & 28.140 & 27.963 & 24.650 & 28.377 & 31.513 & 25.913 & 27.188 \\
\midrule
\multicolumn{9}{l}{\textit{Variance Regularized} ($\lvar = 1$)} \\
\quad No Reuse             & 20.767 & 23.770 & 23.210 & 21.587 & 22.580 & 24.433 & 20.617 & 22.423 \\
\quad Spatial Reuse         & 26.007 & 30.070 & \second{29.573} & 26.507 & 28.983 & 31.077 & 26.657 & 28.410 \\
\quad Temporal Reuse        & \second{26.473} & \second{30.263} & 29.563 & \second{28.063} & \second{29.643} & \second{31.933} & \second{26.917} & \second{28.980} \\
\quad Spatio-Temporal Reuse & \first{28.540} & \first{34.067} & \first{33.687} & \first{30.070} & \first{33.830} & \first{36.533} & \first{30.417} & \first{32.449} \\
\bottomrule
\end{tabular}
\end{table*}

\vspace{1em}

\begin{table*}[h]
\centering
\small
\setlength{\tabcolsep}{6.5pt}
\caption{\textbf{Per-scene PSNR (dB).} Evaluated on 7 Mip-NeRF 360 scenes (3 trajectories each).}
\label{tab:per_scene_factorial_psnr}
\begin{tabular}{lcccccccc}
\toprule
\textbf{Configuration} & \textbf{Bicycle} & \textbf{Bonsai} & \textbf{Counter} & \textbf{Garden} & \textbf{Kitchen} & \textbf{Room} & \textbf{Stump} & \textbf{Overall Mean} \\
\midrule
\multicolumn{9}{l}{\textit{Standard 3DGS} ($\lvar = 0$)} \\
\quad No Reuse             & 16.340 & 19.307 & 18.880 & 16.397 & 18.750 & 21.313 & 17.367 & 18.336 \\
\quad Spatial Reuse         & 22.217 & 26.120 & 25.697 & 22.063 & 25.490 & 28.157 & 23.673 & 24.774 \\
\quad Temporal Reuse        & 22.727 & 25.980 & 25.590 & 23.410 & 25.957 & 28.803 & 24.123 & 25.227 \\
\quad Spatio-Temporal Reuse & 26.453 & 30.837 & 30.867 & 27.440 & 31.287 & 34.273 & 28.463 & 29.946 \\
\midrule
\multicolumn{9}{l}{\textit{Variance Regularized} ($\lvar = 1$)} \\
\quad No Reuse             & 23.863 & 26.877 & 26.333 & 24.637 & 25.767 & 27.730 & 23.720 & 25.561 \\
\quad Spatial Reuse         & 29.043 & 33.070 & 32.613 & 29.443 & 32.043 & 34.187 & 29.657 & 31.437 \\
\quad Temporal Reuse        & \second{29.337} & \second{33.220} & \second{32.690} & \second{31.013} & \second{32.717} & \second{34.993} & \second{29.703} & \second{31.953} \\
\quad Spatio-Temporal Reuse & \first{31.293} & \first{36.573} & \first{36.450} & \first{32.787} & \first{36.557} & \first{39.063} & \first{32.873} & \first{35.085} \\
\bottomrule
\end{tabular}
\end{table*}

\vspace{1em}

\begin{table*}[h]
\centering
\small
\setlength{\tabcolsep}{6.5pt}
\caption{\textbf{Per-scene SSIM $\uparrow$.} Evaluated on 7 Mip-NeRF 360 scenes (3 trajectories each).}
\label{tab:per_scene_factorial_ssim}
\begin{tabular}{lcccccccc}
\toprule
\textbf{Configuration} & \textbf{Bicycle} & \textbf{Bonsai} & \textbf{Counter} & \textbf{Garden} & \textbf{Kitchen} & \textbf{Room} & \textbf{Stump} & \textbf{Overall Mean} \\
\midrule
\multicolumn{9}{l}{\textit{Standard 3DGS} ($\lvar = 0$)} \\
\quad No Reuse             & 0.272 & 0.293 & 0.301 & 0.276 & 0.297 & 0.344 & 0.265 & 0.293 \\
\quad Spatial Reuse         & 0.538 & 0.612 & 0.607 & 0.538 & 0.601 & 0.646 & 0.559 & 0.586 \\
\quad Temporal Reuse        & 0.579 & 0.625 & 0.614 & 0.609 & 0.640 & 0.684 & 0.582 & 0.619 \\
\quad Spatio-Temporal Reuse & 0.767 & 0.832 & 0.826 & 0.791 & 0.843 & 0.872 & 0.785 & 0.817 \\
\midrule
\multicolumn{9}{l}{\textit{Variance Regularized} ($\lvar = 1$)} \\
\quad No Reuse             & 0.598 & 0.664 & 0.628 & 0.640 & 0.635 & 0.615 & 0.513 & 0.613 \\
\quad Spatial Reuse         & 0.825 & 0.887 & 0.864 & 0.838 & 0.856 & 0.855 & 0.792 & 0.845 \\
\quad Temporal Reuse        & \second{0.849} & \second{0.896} & \second{0.865} & \second{0.880} & \second{0.879} & \second{0.874} & \second{0.799} & \second{0.863} \\
\quad Spatio-Temporal Reuse & \first{0.915} & \first{0.961} & \first{0.948} & \first{0.926} & \first{0.951} & \first{0.952} & \first{0.904} & \first{0.937} \\
\bottomrule
\end{tabular}
\end{table*}

\vspace{1em}

\begin{table*}[h]
\centering
\small
\setlength{\tabcolsep}{6.5pt}
\caption{\textbf{Per-scene LPIPS $\downarrow$.} Evaluated on 7 Mip-NeRF 360 scenes (3 trajectories each).}
\label{tab:per_scene_factorial_lpips}
\begin{tabular}{lcccccccc}
\toprule
\textbf{Configuration} & \textbf{Bicycle} & \textbf{Bonsai} & \textbf{Counter} & \textbf{Garden} & \textbf{Kitchen} & \textbf{Room} & \textbf{Stump} & \textbf{Overall Mean} \\
\midrule
\multicolumn{9}{l}{\textit{Standard 3DGS} ($\lvar = 0$)} \\
\quad No Reuse             & 0.636 & 0.693 & 0.690 & 0.584 & 0.677 & 0.687 & 0.701 & 0.667 \\
\quad Spatial Reuse         & 0.478 & 0.531 & 0.518 & 0.417 & 0.499 & 0.523 & 0.522 & 0.498 \\
\quad Temporal Reuse        & 0.504 & 0.569 & 0.542 & 0.421 & 0.514 & 0.538 & 0.564 & 0.522 \\
\quad Spatio-Temporal Reuse & 0.363 & 0.420 & 0.380 & 0.282 & 0.331 & \second{0.382} & \second{0.398} & 0.365 \\
\midrule
\multicolumn{9}{l}{\textit{Variance Regularized} ($\lvar = 1$)} \\
\quad No Reuse             & 0.498 & 0.552 & 0.534 & 0.424 & 0.513 & 0.562 & 0.591 & 0.525 \\
\quad Spatial Reuse         & \second{0.322} & \second{0.372} & \second{0.355} & 0.252 & 0.320 & 0.398 & 0.400 & \second{0.345} \\
\quad Temporal Reuse        & 0.330 & 0.392 & 0.373 & \second{0.245} & \second{0.319} & 0.403 & 0.433 & 0.357 \\
\quad Spatio-Temporal Reuse & \first{0.204} & \first{0.209} & \first{0.208} & \first{0.147} & \first{0.178} & \first{0.243} & \first{0.270} & \first{0.209} \\
\bottomrule
\end{tabular}
\end{table*}

\end{document}